\documentclass[twoside]{article}

\usepackage[preprint]{aistats2027}
\usepackage[numbers,square,sort&compress]{natbib}

\usepackage{mathtools} 

\usepackage{titletoc}
\usepackage{booktabs} 
\usepackage{tikz} 

\definecolor{piBO}{HTML}{6A5ACD}   
\definecolor{piBOi}{HTML}{20B2AA}  
\usepackage{pifont}
\definecolor{skyblue}{RGB}{86,180,233}
\newcommand{\cmark}{\ding{51}}%
\newcommand{\xmark}{\ding{55}}%

\usepackage[normalem]{ulem}
\newcommand{\add}[1]{\textcolor{red}{#1}}

\usepackage{hyperref}
\usepackage{url}
\usepackage{tabularx}
\usepackage{tcolorbox}
\usepackage{array}  
\usepackage{subcaption}
 \usepackage{amsmath}
 \usepackage{graphicx}
\usepackage{wrapfig}
\usepackage[utf8]{inputenc} 
\usepackage{xcolor}
\usepackage{multirow}
\usepackage{graphicx}
\definecolor{lightblue}{RGB}{220, 235, 255}
\definecolor{lightgreen}{RGB}{220, 255, 220}
\definecolor{lightorange}{RGB}{255, 240, 220}
\definecolor{lightpurple}{RGB}{240, 230, 255}
\definecolor{lightgray}{gray}{0.9}
\definecolor{burntorange}{RGB}{230,159,0}
\definecolor{softorange}{RGB}{251,193,94}
\definecolor{cobaltblue}{RGB}{0,114,178}
\definecolor{charcoal}{RGB}{0,0,0}
\definecolor{darkgreen}{RGB}{0,158,115}
\definecolor{mintgreen}{RGB}{102,221,170}
\definecolor{magenta}{RGB}{204,121,167}
\definecolor{peach}{RGB}{244,165,193}
\definecolor{brickred}{RGB}{213,94,0}
\definecolor{slategray}{RGB}{153,153,153}
\definecolor{3dblue}{RGB}{0,0,139}
\definecolor{3dgreen}{RGB}{0,100,0}
\definecolor{3dgray}{RGB}{128,128,128}
 
\newcommand{\solidline}[2]{\textcolor{#1}{\rule[0.5ex]{#2}{0.8pt}}}
\newcommand{\dashedline}[2]{\textcolor{#1}{
  \rule[0.5ex]{0.3em}{0.8pt}\hspace{0.2em}%
  \rule[0.5ex]{0.3em}{0.8pt}\hspace{0.2em}%
  \rule[0.5ex]{0.3em}{0.8pt}
}}
\usepackage[T1]{fontenc}    
\usepackage{hyperref}       
\usepackage{url}            
\usepackage{booktabs}       
\usepackage{amsfonts}       
\usepackage{nicefrac}    
\usepackage{hyperref}
\usepackage{microtype}      
\usepackage{xcolor}         
\usepackage{enumitem} 
\usepackage{algorithm}
\usepackage{graphicx} 
\usepackage{amssymb}
\usepackage{comment}
\usepackage{algpseudocode}
\usepackage{amsthm}
\newtheorem{theorem}{Theorem}
\newtheorem{lemma}{Lemma}
\newtheorem{assumption}{Assumption} 

\usepackage{lastpage}

\begin{document}

%

%

\twocolumn[

\aistatstitle{LLINBO: Trustworthy LLM-in-the-Loop Bayesian Optimization}

\aistatsauthor{
Chih-Yu Chang \And
Milad Azvar \And
Chinedum Okwudire \And
Raed Al Kontar
}

\aistatsaddress{
University of Michigan, Ann Arbor, MI, USA
}
]

\begin{abstract}%
Bayesian optimization (BO) is a sequential decision-making tool widely used for optimizing expensive black-box functions. Recently, Large Language Models (LLMs) have shown remarkable adaptability in low-data regimes, making them promising tools for black-box optimization by leveraging contextual knowledge to propose high-quality query points. However, LLMs lack explicit surrogate modeling and calibrated uncertainty, and their internal mechanisms remain opaque. This structural opacity can result in unnecessary experiments, which are costly in BO and may prevent the search from approaching the optimum. To address this, we propose \texttt{LLINBO}: LLM-in-the-Loop BO, a general framework that wraps existing LLM-based BO approaches, deciding when and how to act on their suggestions and when to revert to classical BO. We introduce three mechanisms for this interaction, each accompanied by a no-harm theoretical guarantee that ensures sublinear cumulative regret for common kernels while requiring no performance assumptions on the LLM. Beyond these theoretical guarantees, we show through simulation studies that \texttt{LLINBO} makes existing LLM-based BO more reliable while still exploiting the LLM's few-shot learning ability to outperform classical BO. We further demonstrate its practical use on a real-world, expensive 3D printing experiment. Code is available at \url{https://github.com/UMDataScienceLab/LLM-in-the-Loop-BO}.  

\end{abstract}

\section{Introduction}\label{sec:intro}

BO has emerged as a powerful tool for black-box optimization (BBO), providing a principled framework for balancing exploration and exploitation. BO is particularly useful in scenarios where function evaluations are costly \citep{korovina2020chembo, liao2023interaction}. The past decade has witnessed many success stories for BO, and its theoretical guarantees have been established \citep{srinivas2009gaussian,agrawal2012analysis}. These guarantees are typically regret-based, ensuring that, with high probability, one can asymptotically recover an optimal design. 

Recently, the few-shot learning capabilities of LLMs and their ability to generate high-quality outputs from minimal examples have made them attractive tools for optimization tasks \citep{yang2023large}. In particular, LLMs have shown strong empirical performance over random search \citep{liu2024large}, largely due to their ability to leverage problem context to fast-track the exploration of promising designs. Intuitively, LLMs act like \textit{domain experts}, using contextual cues to identify high-quality designs early in the optimization process. At each iteration, different phases of BO are carried out by the LLM through appropriately tailored prompts \citep{liu2024large,yang2023large}. We refer to this class of approaches, where LLMs are solely responsible for proposing design candidates and/or serve as the surrogate model in BO, as \textit{LLM-based BO}.

\textbf{Main considerations and contributions.} While recent work on \textit{LLM-based BO} \citep{liu2024large,guo2023towards,song2024position,yang2023large,ngo2026adaptive} has demonstrated promise in generating reasonable query designs, several limitations hinder its broader applicability. Specifically, LLMs do not provide explicit surrogate modeling or calibrated uncertainty, both of which are essential for principled exploration-exploitation trade-offs. Consequently, although LLMs can accelerate optimization in the early stages, their effectiveness systematically degrades as more data are collected and surrogate models strengthen. This degradation is a central characteristic that we explicitly hedge against in our proposed framework.

Moreover, LLMs remain inherently opaque, making the aforementioned trade-off difficult to interpret or control. This structural opacity, combined with their inability to quantify uncertainty in a principled way, introduces significant risks, particularly in applications where cost or safety is critical. It also undermines the theoretical tractability and reliability of LLM-based BO. For instance, in the case of smooth functions, the predictive capability of $\mathcal{GP}$s, in terms of both the predicted mean and variance as measured by generalization bounds, has a known rate of improvement as more data is gathered (\cite{srinivas2009gaussian}). The same result is hard to characterize for LLMs, whose internal mechanisms for interpolating black-box functions are not fully understood and which currently lack calibrated uncertainty estimates.

With this in mind, we propose \texttt{LLINBO}, a framework that wraps any LLM-based BO method with a statistical surrogate that accepts, rejects, or reshapes its suggestions, with guarantees that hold regardless of the LLM's quality. We propose three such mechanisms, analyze their theoretical properties, and demonstrate their effectiveness and robustness through extensive simulations and a real-world 3D printing case study.

\section{Related Works}
\textbf{LLM-based BO.} Recent work has shown that LLMs generalize effectively from limited in-context information \citep{lampinen2025generalization,brown2020language}, making them promising for BBO, where the objective is unknown and observations are
scarce \citep{liu2024large}. Existing approaches use LLMs to warm-start the search, propose candidates, or act as surrogate models
\citep{liu2024large,yang2023large,requeima2024llm,ngo2026adaptive,gupta2025LLM}. Rather than designing a new LLM-based BO method, \texttt{LLINBO} makes any such method more reliable by pairing it with a statistical surrogate that determines when and how its suggestions are used, reverting to classical BO otherwise.

\textbf{Collaborative BO.} Several research directions use external information to make BO more efficient. In federated BO \citep[F-BO;][]{dai2020federated,yue2025collaborative,chen2025multi,dai2024federated},
multiple clients collaborate so that each uses the others' information to speed up its own decisions. In Human--AI Collaborative BO
\citep[HAIC-BO;][]{hvarfnerpi,xu2024principled,adachi2024looping}, human preferences or beliefs are incorporated into the BO process. The role of LLMs in our framework differs from that of clients in F-BO or humans in HAIC-BO. The few-shot ability of LLMs yields high-quality candidates in low-data regimes \citep{liu2024large,brown2020language}, but their advantage over surrogate models diminishes as data accumulate, as our experiments also show; clients and humans are not typically assumed to degrade in this way. \texttt{LLINBO} explicitly hedges against this degradation, leveraging the LLM's early-stage strengths in tandem with $\mathcal{GP}$s while guarding against its later unreliability. A detailed review of F-BO and HAIC-BO is provided in Appendix~\ref{app:rw}.
\section{LLINBO: LLM-in-the-Loop BO}
\subsection{Preliminaries}
\label{sec:pre}
BO aims to find an optimal design $x^*$ that maximizes a black-box function $f$ over a domain \( \mathcal{X} \) by sequentially selecting query designs. Given a total budget \( T\in\mathbb{N} \), the data at iteration \( t \in [T] \) is denoted as \( \mathcal{D}_{t-1} = \{(x_i, y_i)\}_{i=1}^{t-1} \), 
where $y_i = f(x_i) + \epsilon_i$, and $\epsilon_i$ is assumed to be conditionally sub-Gaussian (see Assumption~\ref{assumption:rkhs_noise} for the formal conditions).

At time \( t \), BO selects the next design, denoted by $x_t$, to observe by maximizing an AF, \( \alpha(x, F_{t-1}) \), where \( F_{t-1} \) is the posterior belief of \( f \) conditioned on \( \mathcal{D}_{t-1} \). After selecting \( x_t \), a noisy observation \( y_t = f(x_t) + \epsilon_t \) is obtained, and the dataset is updated as \( \mathcal{D}_t = \mathcal{D}_{t-1} \cup \{(x_t, y_t)\} \). This process is then repeated until $T$ is exhausted. The posterior belief is typically modeled using a $\mathcal{GP}$ (\cite{kushner1964new, rasmussen2006gaussian}), which requires a prior mean function (often set to zero) and a kernel function \( k(x, x') \) encoding the smoothness of the function. This yields a posterior predictive distribution for $f$ given as
\begin{align*}
f(x) \mid \mathcal{D}_{t-1} \sim \mathcal{GP}(\mu_{t-1}(x), \sigma^2_{t-1}(x)),
\end{align*}
with $\mu_{t-1}(x) = k_{t-1}(x)^\top (K + \lambda^2 I)^{-1}y$ and $
\sigma^2_{t-1}(x) = k(x, x) - k_{t-1}(x)^\top (K + \lambda^2 I)^{-1} k_{t-1}(x),$ where $K$ is the Gram matrix of the training inputs with \( K_{ij} = k(x_i, x_j),\ \forall i,j\in[t-1] \), \( k_{t-1}(x) = [k(x, x_1), \ldots, k(x, x_{t-1})]^\top \) being the covariance vector between the input \( x \) and the training inputs, and \( y = [y_1, \ldots, y_{t-1}]^\top \) is the vector of observed responses. 


The posterior mean \( \mu_{t-1}(x) \) and variance \( \sigma^2_{t-1}(x) \) quantify our posterior belief about the function's value and uncertainty over \( \mathcal{X} \), which we denote compactly as \( F_{t-1} = \mathcal{GP}(\mathcal{D}_{t-1}) \). While many AFs have been proposed, we focus without loss of generality on the Upper Confidence Bound (UCB, \cite{srinivas2009gaussian}), a widely used AF defined as
\begin{align}
\alpha_{\text{UCB}}(x, F_{t-1}) = \mu_{t-1}(x) + \beta_t \sigma_{t-1}(x),
\label{eq:UCB}
\end{align}
where \( \beta_t \) is a parameter that controls the trade-off between exploration and exploitation.

 \subsection{LLM-in-the-Loop BO Framework}
\label{sec:A0}
We consider an LLM-based BO agent $\mathcal{A}$ that, whenever queried at iteration $t$, proposes a single design $x_{\text{LLM},t}$. The agent can be any existing LLM-based BO method (e.g., \cite{liu2024large,requeima2024llm,gupta2025LLM}) or simply an LLM prompted with the problem context and the observed data $\mathcal{D}_{t-1}$. At each iteration, \texttt{LLINBO} combines two sources of guidance: the conventional BO candidate $x_{\mathcal{GP},t}$, obtained by maximizing the AF under the posterior $F_{t-1}$, and the candidate $x_{\text{LLM},t}$ proposed by $\mathcal{A}$, which is assessed against the posterior belief $F_{t-1}$ without querying $f$. While the framework does not prescribe a specific statistical surrogate, for concreteness we take $F_{t-1}$ to be a $\mathcal{GP}$ posterior, whose mean $\mu_{t-1}(x_{\text{LLM},t})$ and variance $\sigma_{t-1}^2(x_{\text{LLM},t})$ characterize the candidate's predicted mean and uncertainty. The mechanisms in Secs.~\ref{sec:A1}--\ref{sec:A3} then determine whether the LLM candidate is retained, rejected, or used to reshape the surrogate. Algorithm~\ref{algo:LLMIBO} summarizes this framework.

\textbf{Admissible agent.} \textcolor{black}{Throughout, we require the LLM-based agent $\mathcal{A}$ to be \emph{admissible}. That is, its suggestion $x_{\text{LLM},t}\in\mathcal{X}$ may depend only on information available before the round-$t$ observation (e.g., $\mathcal{D}_{t-1}$ and the agent's own randomness). The agent's randomness, the Bernoulli draws in Sec.~\ref{sec:A1}, and the fantasy draws in Sec.~\ref{sec:A3} do not use current or future observation noise. No assumption is made on the quality of $x_{\text{LLM},t}$.}  

\begin{algorithm}
\caption{LLM-in-the-Loop BO \texttt{(LLINBO)}}
\textbf{Input:} $\mathcal{D}_0$, $T$, LLM-based BO agent $\mathcal{A}$, kernel function $k$, AF $\alpha$.
\begin{algorithmic}[1]
\For{$t = 1$ to $T$}
    \State Compute $F_{t-1}=\mathcal{GP}(\mathcal{D}_{t-1})$
    \State Compute $x_{\mathcal{GP},t}=\underset{x}{\arg\max}\ \alpha(x,F_{t-1})$
    \State \textcolor{cyan}{Query $\mathcal{A}$ for a suggested design point $x_{\text{LLM},t}$}
    \State \textcolor{cyan}{Assess or incorporate $x_{\text{LLM},t}$ using $F_{t-1}$}
    \State \textcolor{cyan}{Generate $x_t$ by refining, retaining, or rejecting $x_{\text{LLM},t}$ using mechanisms in Secs.~\ref{sec:A1}--\ref{sec:A3}}
    \State Obtain $y_t$ and set $\mathcal{D}_t \gets \mathcal{D}_{t-1} \cup \{(x_t, y_t)\}$
\EndFor \\

\Return $\text{argmax}_{x_i} \{y_i \mid (x_i,y_i)\in\mathcal{D}_T\}$
\end{algorithmic}
\label{algo:LLMIBO}
\end{algorithm}

Without Lines 4--6, Algorithm~\ref{algo:LLMIBO} reduces to conventional BO, with $x_t=x_{\mathcal{GP},t}$. At the other extreme, removing the GP-based recommendation and assessment yields purely LLM-based BO, with $x_t=x_{\text{LLM},t}$. \texttt{LLINBO} sits between these two extremes by introducing a statistical gate that validates $x_{\text{LLM},t}$ before it is acted on, thereby ensuring the quality of the resulting query design.

\subsection{\texttt{LLINBO-Transient}: Exploration by LLMs then Exploitation by $\mathcal{GP}s$}
\label{sec:A1}
Perhaps the most natural form of collaboration between an LLM-based BO and a classical BO method is to leverage the LLM’s contextual reasoning early in the process, initially placing greater attention on \( x_{\text{LLM},t} \), and gradually transitioning to the $\mathcal{GP}$'s suggestion \( x_{\mathcal{GP},t} \) as more data are collected. The $\mathcal{GP}$, with its ability to interpolate observed data and calibrate uncertainty, becomes increasingly reliable for guiding exploitation.

More specifically, we propose that, with $z_t\sim\text{Bernoulli}(p_t)$, the query design \( x_t \) be selected as
\[
x_t=z_t \cdot x_{\mathcal{GP},t}+(1-z_t) \cdot x_{\text{LLM},t},
\]
where \( p_t \) is a monotonically increasing sequence approaching 1 as \( t \) increases. \textcolor{black}{The Bernoulli draws are independent across rounds and independent of the observation noise, with the schedule $p_t$ chosen in advance.} Specifically, with probability \( p_t \), \( x_t \) is set to \( x_{\mathcal{GP},t} \), and with probability \( 1 - p_t \), it is set to \( x_{\text{LLM},t} \). The proposed \texttt{LLINBO-Transient} algorithm distributes exploration and exploitation across different models: LLMs facilitate early-stage exploration, while $\mathcal{GP}$s focus on exploitation as more data become available. 

Since $\mathcal{A}$ is queried only when $z_t=0$, the expected number of LLM queries is $\sum_{t=1}^{T}(1-p_t)$, compared with at least $T$ for LLM-based BO, which queries the LLM at every iteration. \texttt{LLINBO-Transient} is thus the most transparent variant, controlled by a single schedule $p_t$, and suits users unfamiliar with BO or with a limited token budget $Q$, who can set $p_t$ such that $\sum_{t=1}^{T}(1-p_t)\le Q$.

\subsection{\texttt{LLINBO-Justify}: Surrogate-driven Rejection of LLM's Suggestions}
\label{sec:A2}

In contrast to \texttt{LLINBO-Transient}, where $x_{\text{LLM},t}$ is directly incorporated during early exploration, here we exploit the posterior belief $F_{t-1}$ as an evaluator for $x_{\text{LLM},t}$. If the LLM suggestion is found to be substantially worse than the current AF maximizer, it is rejected, and $x_{\mathcal{GP},t}$ is used instead. Fundamentally, our goal is to enable the no-harm integration of LLMs into BO by rejecting suggestions that significantly contradict the client's optimal utility; an approach denoted as \texttt{LLINBO-Justify}.
\begin{figure*}[t]
   \centering

    \begin{subfigure}[b]{0.46\textwidth}
        \centering
        \includegraphics[width = 1\textwidth,height=0.31\textwidth]{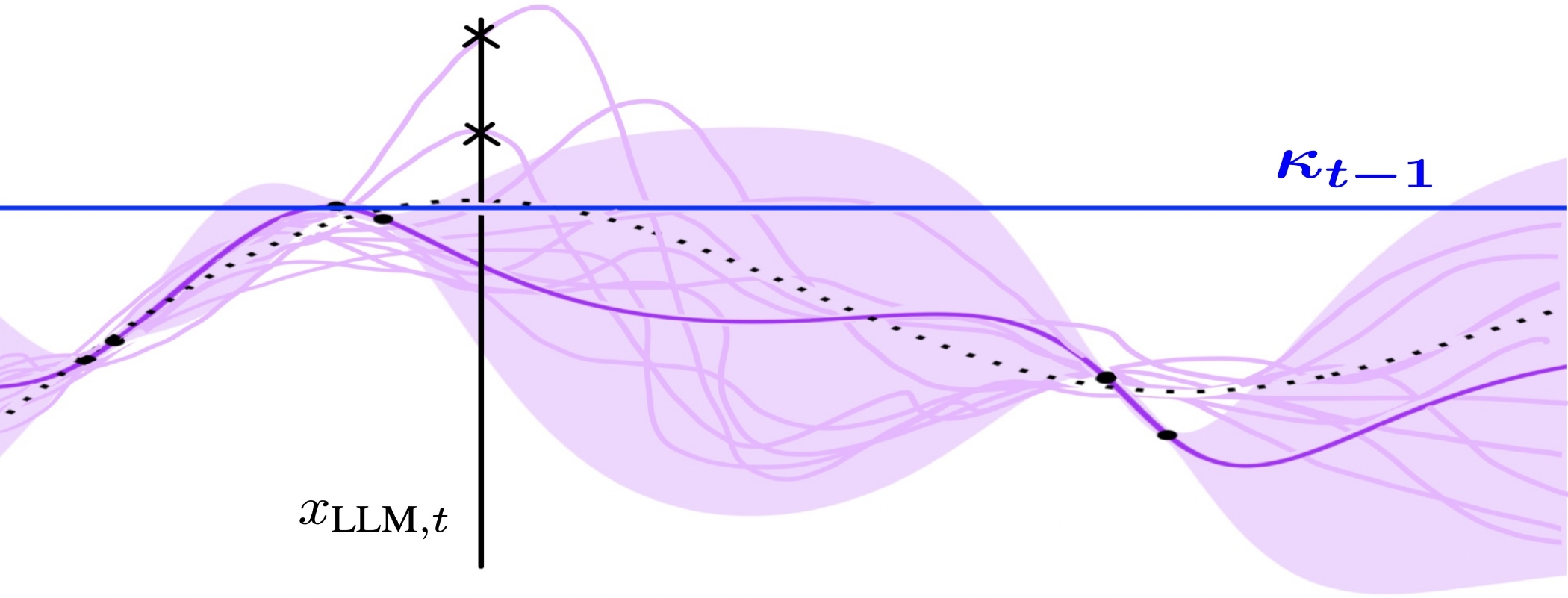}
        \caption{$S_t=10$ realizations $\tilde{f}_{t-1,s}(x_{\text{LLM},t})$ are sampled from $F_{t-1}$ (light purple curves) at the LLM-suggested point $x_{\text{LLM},t}$. Two realizations (crosses) exceed $\kappa_{t-1}$ and are retained, i.e., $I_t = \{1,2\}$.}
    \end{subfigure}\hfill
    \begin{subfigure}[b]{0.46\textwidth}
         \centering
        \includegraphics[width = 1\textwidth,height=0.31\textwidth]{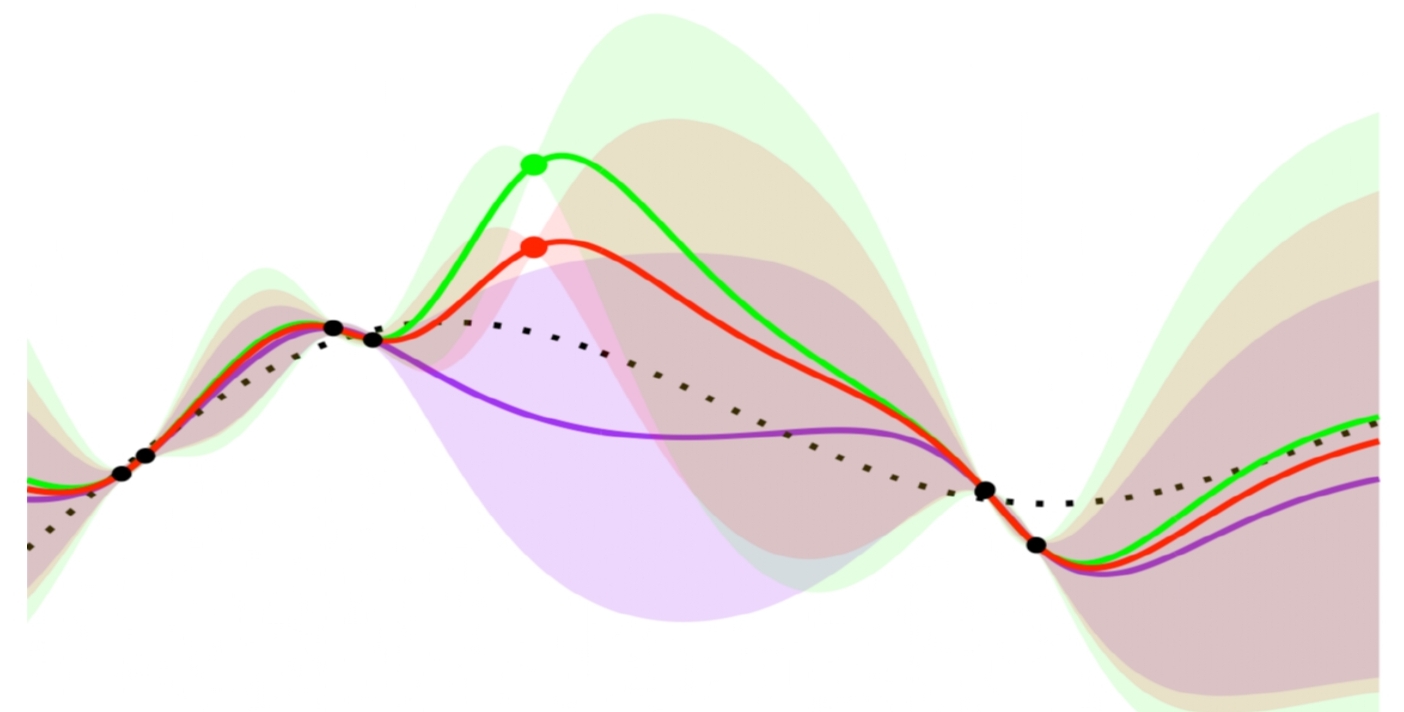}
        \caption{The two $\mathcal{GP}$s (red and green curves) correspond to the two retained samples in (a), each constructed from $\mathcal{D}_{t-1}$ augmented with one retained realization; averaging them gives $\bar{\mu}^+_{t-1}(x)$ used in $\alpha_{\mathcal{CGP}\text{-UCB}}$.}
        \label{fig:qq}
    \end{subfigure}

    \vskip\baselineskip

    \begin{subfigure}[b]{0.46\textwidth}
        \centering
        \includegraphics[width = 1\textwidth,height=0.31\textwidth]{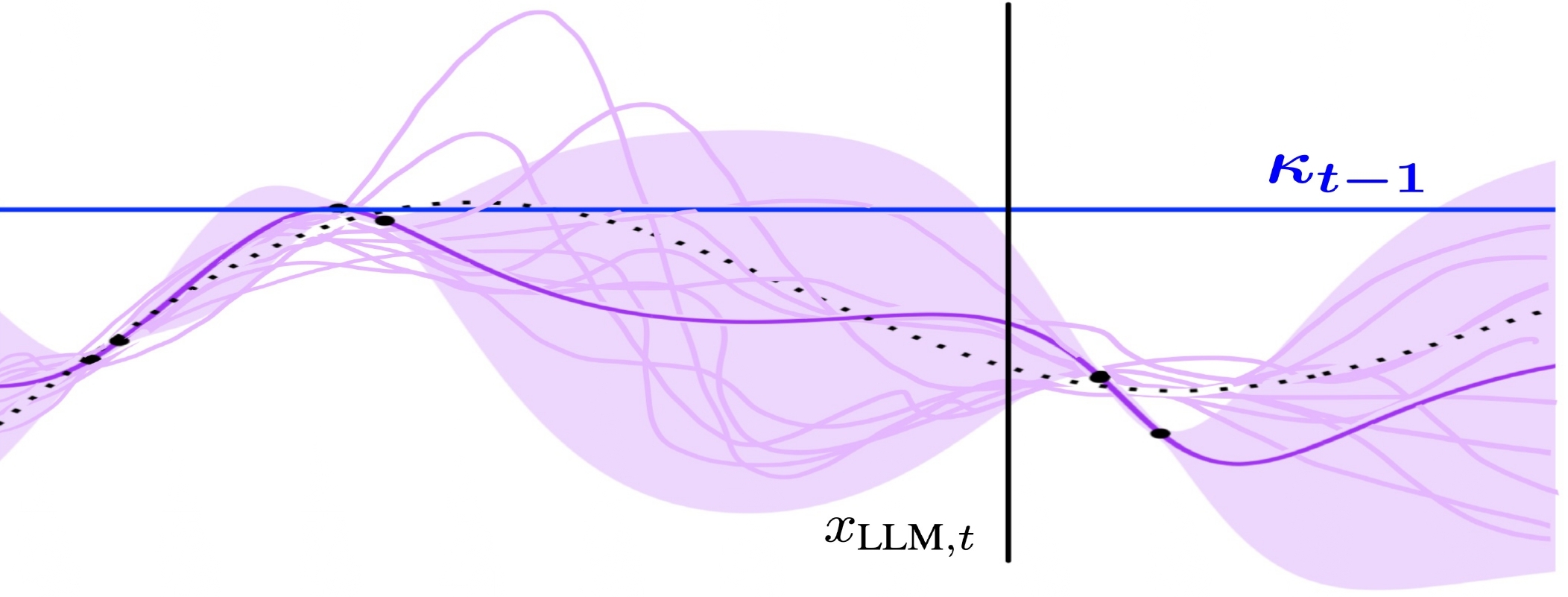}
        \caption{$S_t=10$ realizations are sampled from $F_{t-1}$ at $x_{\text{LLM},t}$, but none exceed $\kappa_{t-1}$, i.e., $I_t = \emptyset$: the LLM's suggestion is not supported by the current posterior.}
    \end{subfigure}\hfill
    \begin{subfigure}[b]{0.46\textwidth}
        \centering
        \includegraphics[width = 1\textwidth,height=0.31\textwidth]{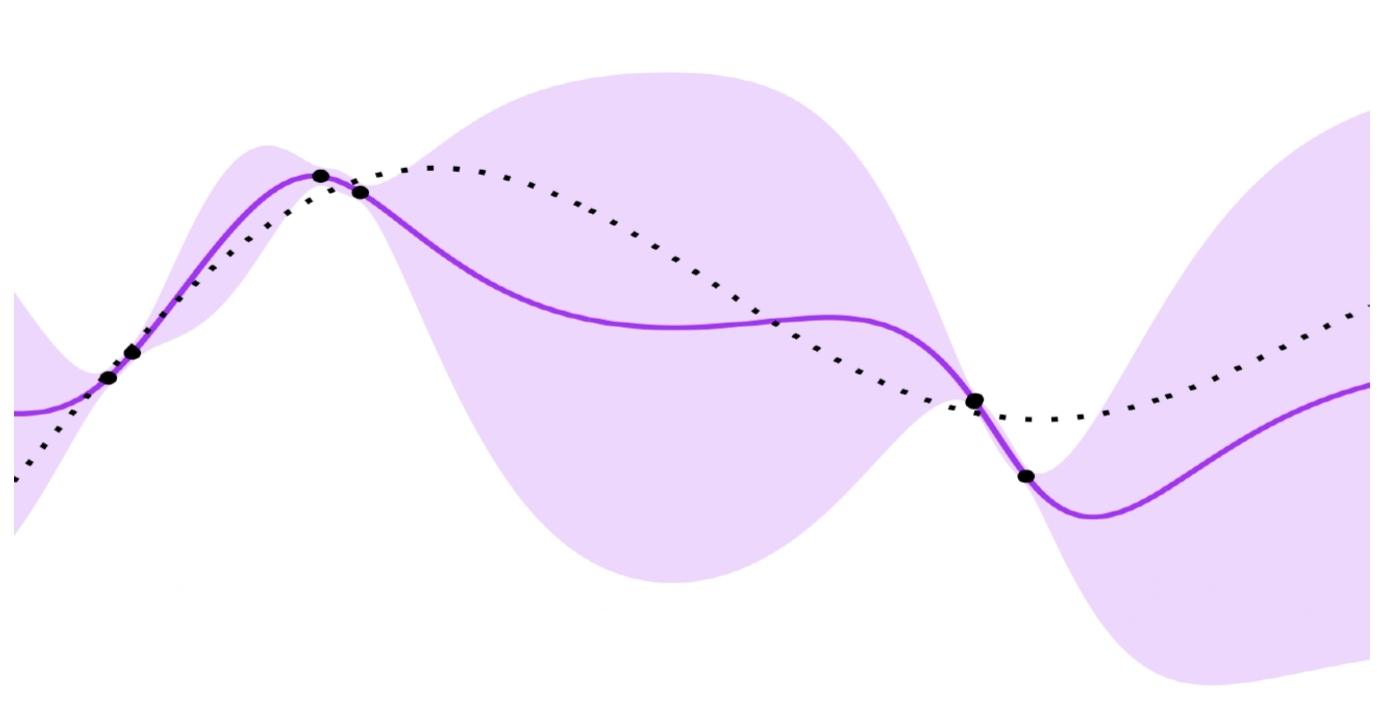}
        \caption{Since $I_t = \emptyset$ in (c), no samples are retained and the procedure falls back to the unconstrained posterior $F_{t-1}$: $x_{\text{LLM},t}$ is discarded in favor of $x_{\mathcal{GP},t}$.}
    \end{subfigure}

    \caption{Illustration of \texttt{LLINBO-Constrained} at a single iteration $t$: the dashed line is the true function $f$, and the horizontal line marks $\kappa_{t-1}$. \textbf{(a)--(b)}: $x_{\text{LLM},t}$ is supported by the posterior and retained, driving the construction of the fantasy models. \textbf{(c)--(d)}: $x_{\text{LLM},t}$ contradicts the posterior and is discarded, leaving the surrogate unchanged.}
    \label{fig:llmCGP}
\end{figure*}

Specifically, given $x_{\text{LLM},t}$ and the AF constructed by $F_{t-1}$, we reject \(x_{\text{LLM},t}\) if
\begin{equation} \label{eq:just}
    \alpha_{\text{UCB}}(x_{\text{LLM},t},F_{t-1})\leq\max_x\alpha_{\text{UCB}}(x,F_{t-1})-\psi_t,
\end{equation}
where $\psi_t$ is the user-selected acceptance tolerance at time $t$. The maximum value of the AF, together with the selected \( \psi_t \), defines the \( \psi_t \)-suboptimal region of the AF. Accordingly, \( x_{\text{LLM},t} \) is accepted and assigned as \( x_t \) if it lies within this region; otherwise, \( x_t = x_{\mathcal{GP},t} \).

In the early stages, when we place greater trust in the LLM’s suggestions, a larger \( \psi_t \) can be chosen to promote broader exploration around \( x_{\text{LLM},t} \), effectively enlarging the \( \psi_t \)-suboptimal region of the AF to investigate a wider area influenced by the LLM. Over time, we recommend gradually decreasing \( \psi_t \) to rely more on the $\mathcal{GP}$, whose uncertainty estimates become increasingly well-calibrated as more data is collected.

Compared with \texttt{LLINBO-Transient}, \texttt{LLINBO-Justify} decides acceptance based on the current posterior rather than a fixed schedule. It thus suits practitioners who prioritize reliability over cost, as every suggestion is validated by the surrogate, at the price of querying the LLM at all $T$ iterations.
\subsection{\texttt{LLINBO-Constrained}: Constraining Surrogates on LLM's Suggestions}
\label{sec:A3}
Apart from the two approaches above that depend on defining \( p_t \) in \texttt{LLINBO-Transient} and \( \psi_t \) in \texttt{LLINBO-Justify}, our third mechanism takes a different approach: it \textbf{directly refines the $\mathcal{GP}$} toward potential regions of improvement using \( x_{\text{LLM},t} \), without requiring such predefined tuning.

Upon receiving $x_{\text{LLM},t}$, we treat this as a potentially good design. Namely, it assumes that $f(x_{\text{LLM},t}) > \kappa_{t-1}$, where $\kappa_{t-1} \triangleq \max_{x}\mu_{t-1}(x)$ is the maximum of the posterior mean. In other words, $x_{\text{LLM},t}$ is treated as a design that can potentially improve upon the current belief of the largest value of $f$. Notice that this constraint may not hold, and we will show shortly how it can be automatically accounted for. With this, the posterior belief we aim to approximate is
\begin{align}
   F^+_{t-1} \triangleq \mathcal{GP}(\mathcal{D}_{t-1}) \mid \left\{ f(x_{\text{LLM},t}) > \kappa_{t-1} \right\}. \label{eq:constr}
\end{align}
This essentially leads to a constrained $\mathcal{GP}$ ($\mathcal{CGP}$). Since \( \mathcal{CGP} \) does not admit a closed-form posterior, we do not sample from it exactly; instead, motivated by the rejection-sampling construction of \cite{chen2025multi}, we use a finite fantasy-model approximation described next.

In practice, we draw \( S_t \) samples, denoted \( \tilde{f}_{t-1,s}(x_{\text{LLM},t}) \) for \( s \in [S_t] \), from \( F_{t-1} \). We retain only those samples satisfying the constraint in~(\ref{eq:constr}), i.e., \( \tilde{f}_{t-1,s}(x_{\text{LLM},t}) > \kappa_{t-1} \). Let \( I_t = \{s \mid \tilde{f}_{t-1,s}(x_{\text{LLM},t}) > \kappa_{t-1} \} \) denote the index set of retained samples. For each \( s \in I_t \), we construct a \( \mathcal{GP} \) (a fantasy model) based on \( \mathcal{D}_{t-1}\cup \{(x_{\text{LLM},t}, \tilde{f}_{t-1,s}(x_{\text{LLM},t}))\} \), treating \( \tilde{f}_{t-1,s}(x_{\text{LLM},t}) \) as a noisy pseudo-observation, with the kernel hyperparameters and noise level of \( F_{t-1} \) kept fixed, and denote its posterior mean and variance by \( \mu_{t-1,s}^+(x) \) and \( \sigma_{t-1,s}^+(x)^2 \). Since all fantasy models share the same inputs, \( \sigma_{t-1,s}^+(x) \) does not depend on \( s \), and we write \( \sigma_{t-1}^+(x) \). As each sampled value is not an exact value of \( f(x_{\text{LLM},t}) \), these fantasy models only approximate \( F^+_{t-1} \).

Fig.~\ref{fig:llmCGP} illustrates the behavior of \texttt{LLINBO-Constrained}. Critically, more samples are retained when the constraint is satisfied, reflecting posterior support for \( x_{\text{LLM},t} \) as a high-quality candidate. In such cases, the posterior means of the fantasy models become elevated near \( x_{\text{LLM},t} \), highlighting promising regions for subsequent exploration (see Fig.~\ref{fig:llmCGP}(a)--(b)). Conversely, when \( x_{\text{LLM},t} \) contradicts the current posterior, no samples are retained (\( |I_t| = 0 \)), and the procedure falls back to the unconstrained surrogate \( F_{t-1} \), effectively discarding $x_{\text{LLM},t}$ in favor of \( x_{\mathcal{GP},t} \) (see Fig.~\ref{fig:llmCGP}(c)--(d)). This selective retention mechanism is key to maintaining the trustworthiness of the BO process and underpins the theoretical guarantees discussed later.

With these \( \mathcal{GP} \)s, each constructed from the union of \( \mathcal{D}_{t-1} \) and a retained sample, we define the AF directly from their empirical moments. Without loss of generality, and focusing on UCB, we define the AF (denoted by $\alpha_{\text{$\mathcal{CGP}$-UCB}}(x,F_{t-1}^+)$, where $F_{t-1}^+$ indicates the target it approximates) as $\bar{\mu}_{t-1}^+(x)+\tilde{\beta}_t\sqrt{\sigma^+_{t-1}(x)^2+s^2_{t-1}(x)},$
where $\bar{\mu}^+_{t-1}(x) = \frac{1}{|I_t|}\sum_{s\in I_t} \mu_{t-1,s}^+(x),\ \ s_{t-1}^2(x) = \frac{1}{|I_t|-1} \sum_{s\in I_t} \left(\mu_{t-1,s}^+(x)-\bar{\mu}_{t-1}^+(x)\right)^2$ and $\tilde{\beta}_t$ is the confidence parameter, which is defined in Theorem~\ref{thm:noharm} in Sec.~\ref{sec:theory}. Here, $\sigma^+_{t-1}(x)^2$ captures the uncertainty within each fantasy model, and $s^2_{t-1}(x)$ captures the spread of the fantasy means across retained samples. The next design is $x_t=\arg\max_x\alpha_{\mathcal{CGP}\text{-UCB}}(x,F^+_{t-1})$,
with $s_{t-1}\equiv 0$ if $|I_t|=1$ and $x_t=x_{\mathcal{GP},t}$ if $I_t=\emptyset$. 


Unlike the other two variants, which either adopt or discard $x_{\text{LLM},t}$, \texttt{LLINBO-Constrained} uses it to reshape the surrogate via finite sampling. In this case, even partially informative suggestions can guide the search. This robustness comes at the cost of
resampling from the $\mathcal{GP}$ and fitting up to $S_t$ parallelizable
$\mathcal{GP}$s per iteration, which is negligible when each evaluation is expensive (e.g., physical experiments taking hours). It is therefore the preferred variant when function evaluations dominate the budget. Moreover, Theorem~\ref{thm:noharm}(iii) only requires $S_t\le S_{\max}$, so the guarantee holds for any bounded schedule, including a constant $S_t$.




\begin{figure*}[t]
    \centering

    \begin{subfigure}[t]{0.32\textwidth}
        \centering
        \includegraphics[width=\linewidth]{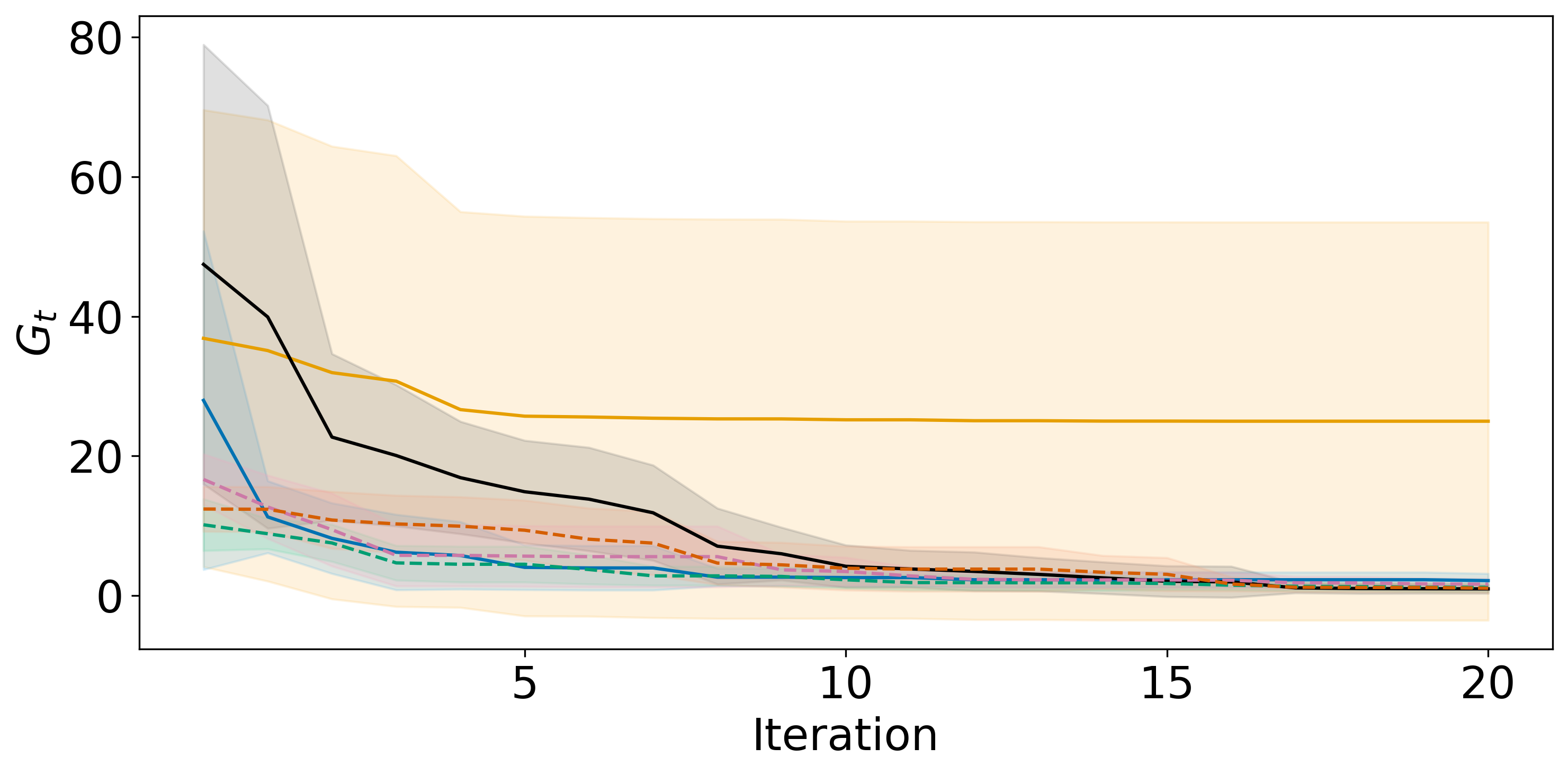}
        \caption{Branin-2$D$}
    \end{subfigure}\hfill
    \begin{subfigure}[t]{0.32\textwidth}
        \centering
        \includegraphics[width=\linewidth]{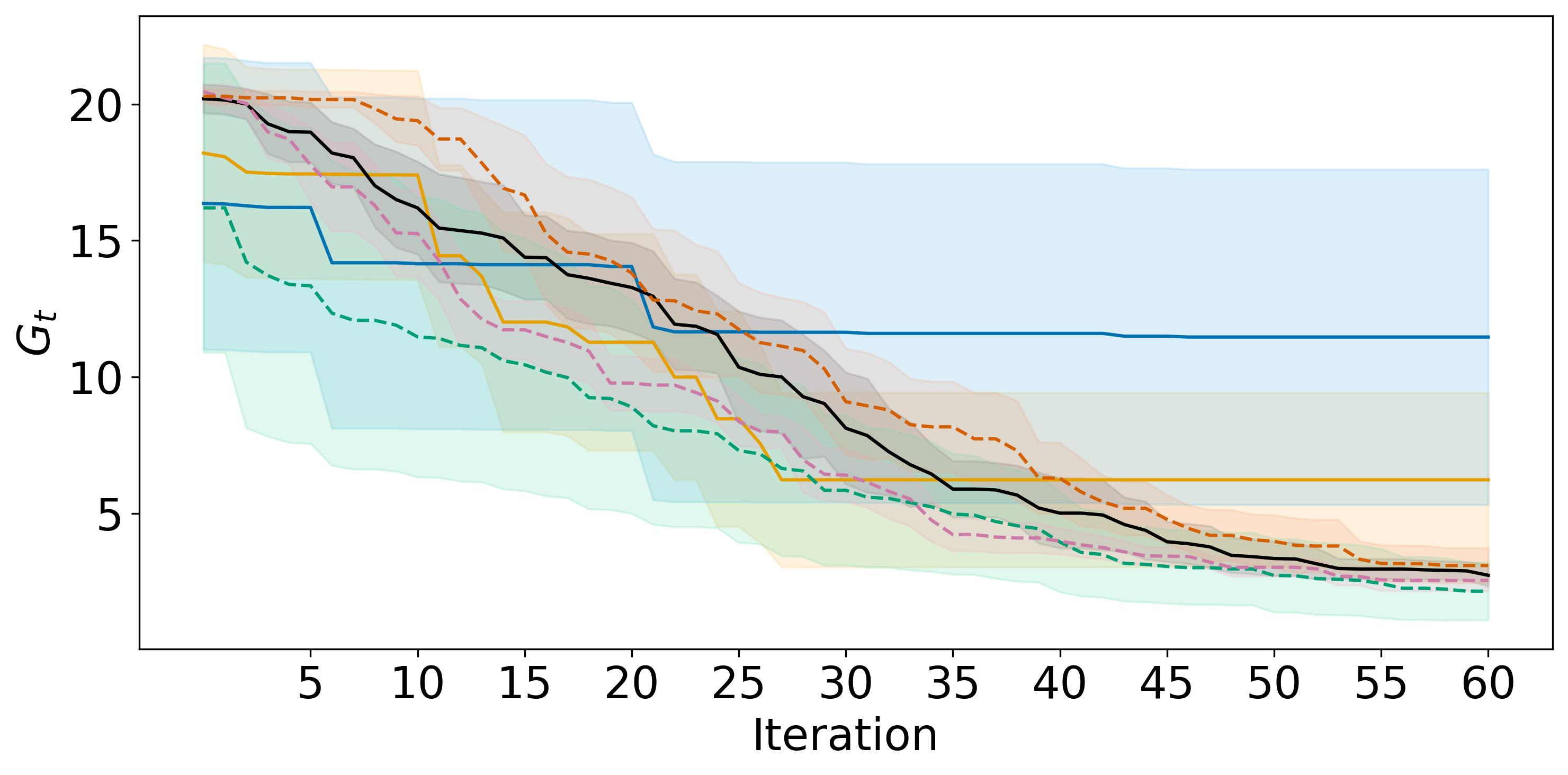}
        \caption{Ackley-6$D$}
    \end{subfigure}\hfill
    \begin{subfigure}[t]{0.32\textwidth}
        \centering
        \includegraphics[width=\linewidth]{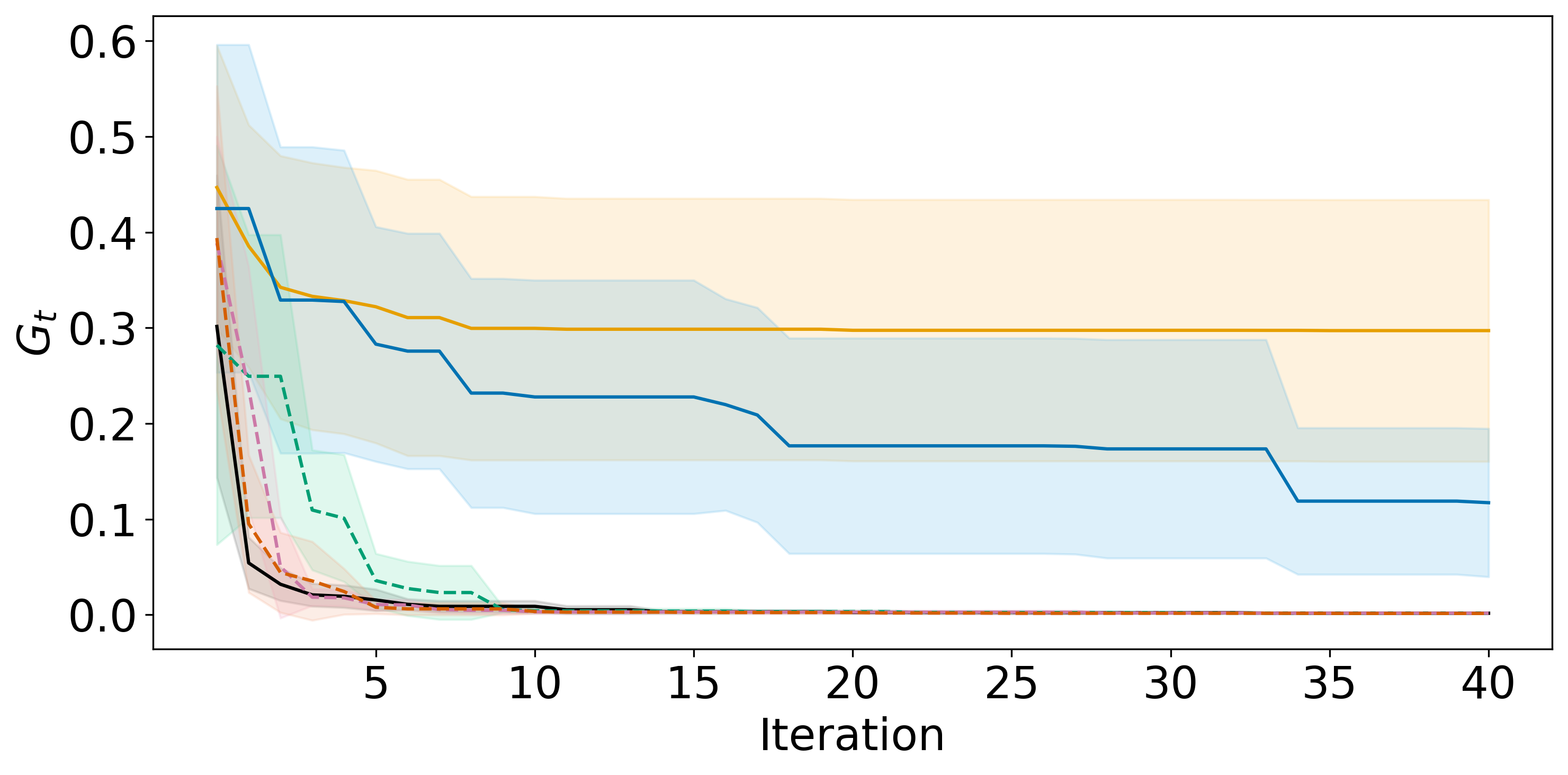}
        \caption{Hartmann-4$D$}
    \end{subfigure}

    \vspace{0.2em}

    \begin{subfigure}[t]{0.32\textwidth}
        \centering
        \includegraphics[width=\linewidth]{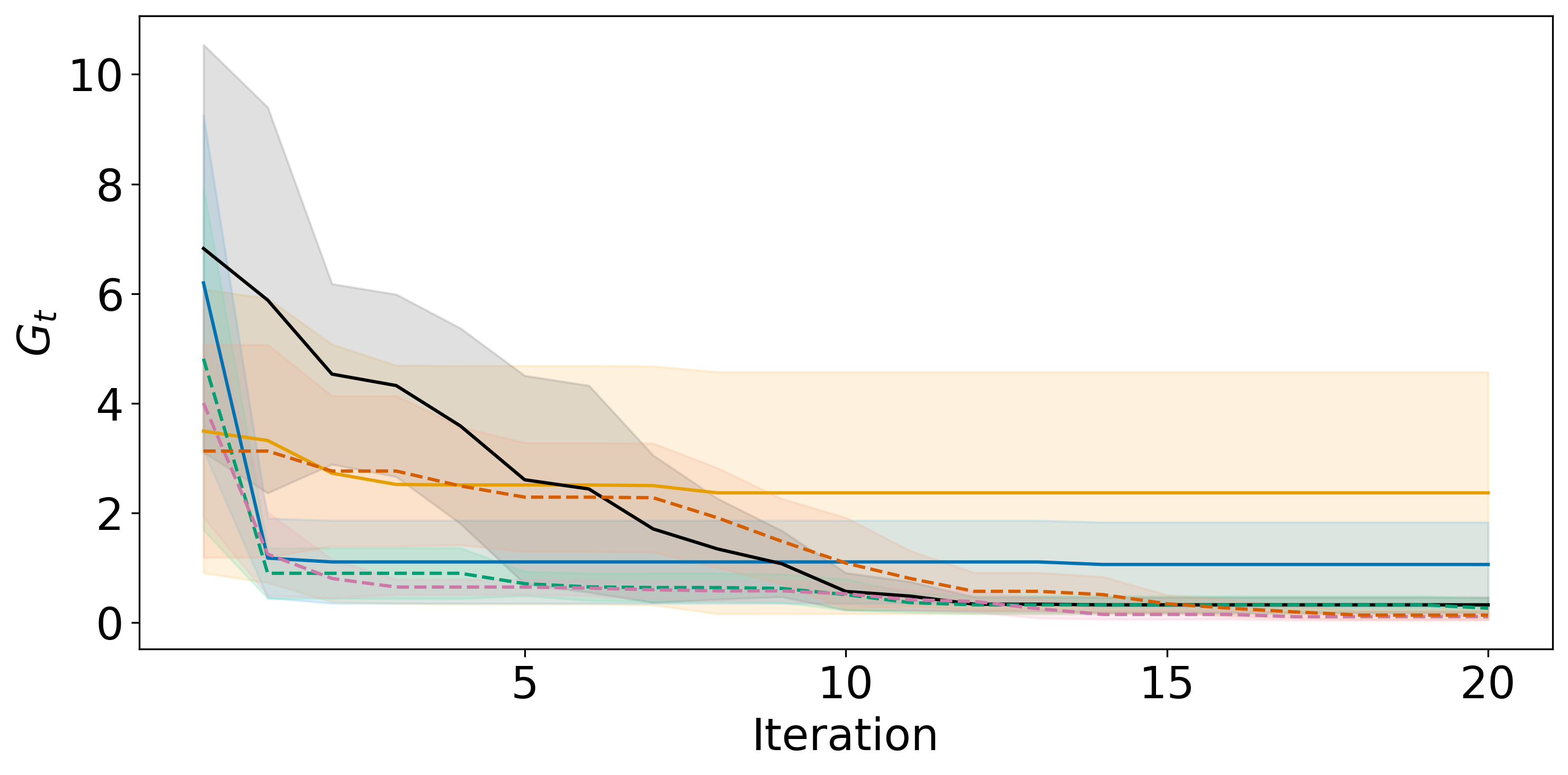}
        \caption{Levy-2$D$}
    \end{subfigure}\hfill
    \begin{subfigure}[t]{0.32\textwidth}
        \centering
        \includegraphics[width=\linewidth]{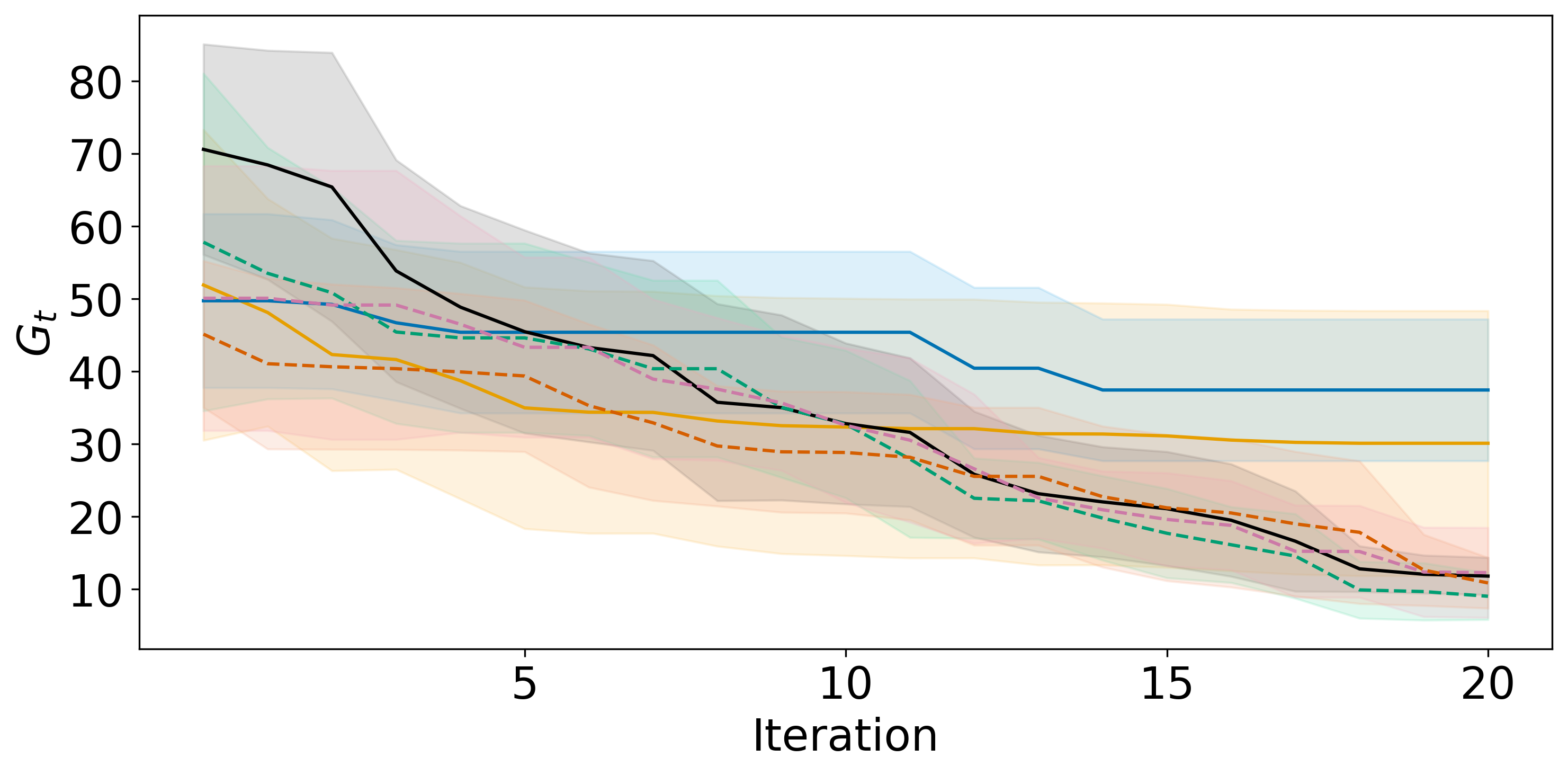}
        \caption{Bukin-2$D$}
    \end{subfigure}\hfill
    \begin{subfigure}[t]{0.32\textwidth}
        \centering
        \includegraphics[width=\linewidth]{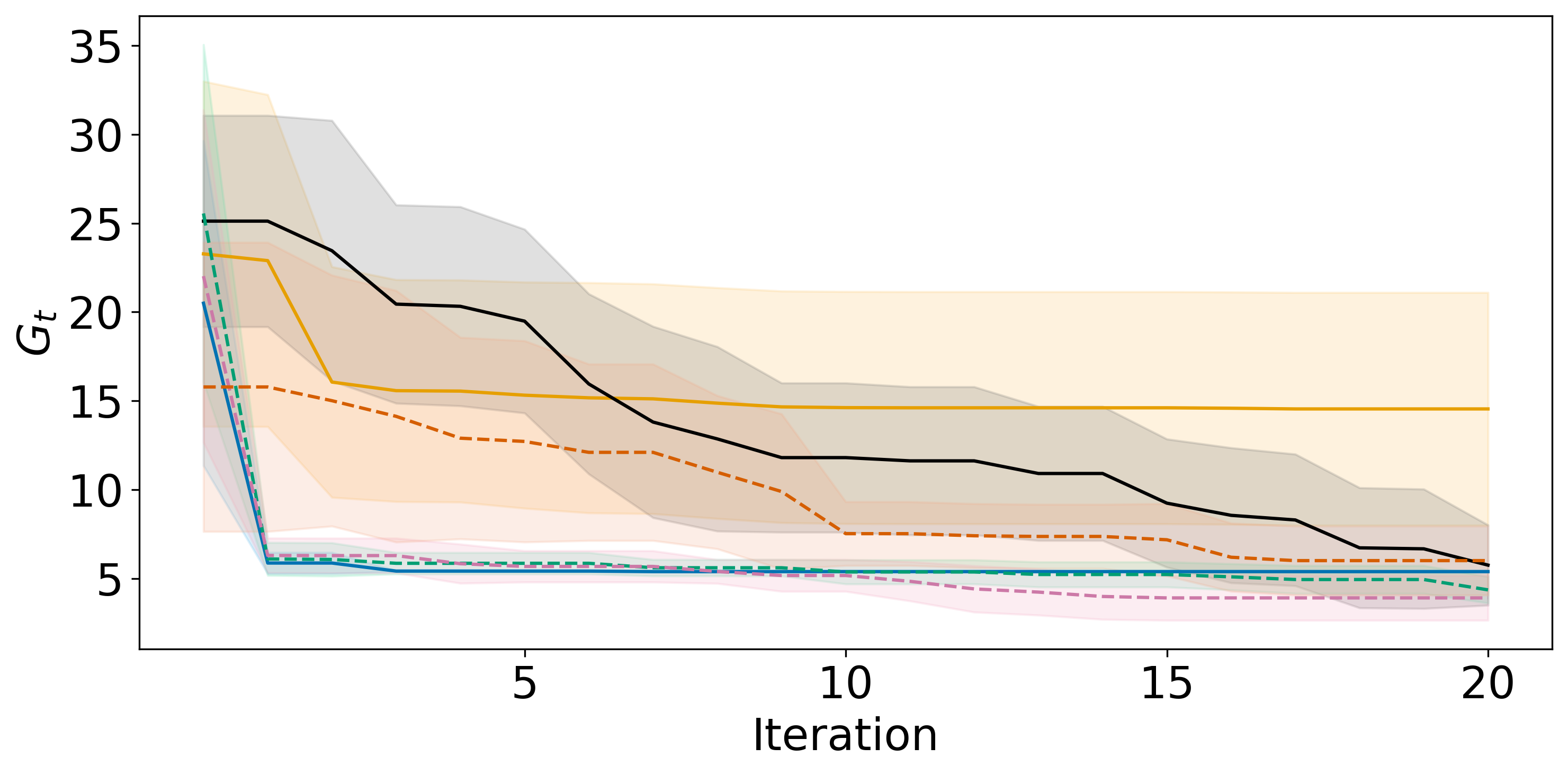}
        \caption{Rastrigin-2$D$}
    \end{subfigure}

    \caption{
    $G_t$ comparison for BBO. Each line shows the mean regret over 10 independent replications; shaded regions denote 95\% confidence intervals of the mean.
    \textbf{Proposed methods}: \dashedline{darkgreen}{1.5em} \texttt{LLINBO-Transient},
    \dashedline{magenta}{1.5em} \texttt{LLINBO-Justify},
    \dashedline{brickred}{1.5em} \texttt{LLINBO-Constrained}. \textbf{Baselines}:
    \solidline{burntorange}{1.5em} \texttt{LLAMBO},
    \solidline{cobaltblue}{1.5em} \texttt{LLAMBO-light},
    \solidline{black}{1.5em} BO.
    }
    \label{fig:bbfores}
\end{figure*}
\section{Theoretical Results}\label{sec:theory}

Our theoretical results serve as a minimal, worst-case guarantee: for any admissible LLM agent (Sec.~\ref{sec:A0}),
including an uninformative or adversarial one, \texttt{LLINBO-Transient} and \texttt{LLINBO-Justify} attain the
cumulative regret order of classical GP-UCB, and \texttt{LLINBO-Constrained} matches it up to logarithmic factors. This is known as a
\emph{no-harm} guarantee \citep{xu2024principled}, a form of justification commonly
adopted in the collaborative and human-in-the-loop BO literature
\citep{hvarfnerpi, adachi2024looping}. We first introduce the following assumptions.

\begin{assumption}\label{assumption:rkhs_noise}
$f$ belongs to a Reproducing Kernel Hilbert Space
(RKHS) $\mathcal{H}_k$ with kernel $k$, such that
$\|f\|_{\mathcal{H}_k}\le B$ for some constant $B\ge 0$,
and the kernel satisfies $k(x,x')\le 1$ for all
$x,x'\in\mathcal{X}$.
The observation noise $\epsilon_t$ is conditionally
$R$-sub-Gaussian for some $R\ge 0$,
\textcolor{black}{given the pre-observation
$\sigma$-algebra $\mathcal{H}_{t-1}$ generated by
$\mathcal{D}_{t-1}$ and all randomness used to
select $x_t$.}
\end{assumption}

\begin{assumption}\label{def:UCB}
Let $\gamma_{t}$ denote the maximum information gain after $t$ observations, as
defined in \citet{chowdhury2017kernelized}. \textcolor{black}{For any confidence parameter $\rho\in(0,1)$, the GP-UCB AF in \eqref{eq:UCB} uses $\beta_t(\rho)=B+R\sqrt{2\big(\gamma_{t-1}+1+\ln(1/\rho)\big)}$. The parameter $\rho$ is specified for each variant in Theorem~\ref{thm:noharm}.}
\end{assumption}

Theorem~\ref{thm:noharm} presents the regret bounds for all three variants of \texttt{LLINBO}, proofs can be found in Appendix~\ref{app:tech}.

\textcolor{black}{Let $x^*\in\arg\max_{x\in\mathcal X}f(x)$ and define the instantaneous regret $r_t=f(x^*)-f(x_t)$ and cumulative regret $R_T=\sum_{t=1}^{T}r_t$, where $x_t$ is the design actually evaluated at iteration $t$.}

\begin{theorem}[No-harm guarantee]\label{thm:noharm}

Suppose Assumptions~\ref{assumption:rkhs_noise}-\ref{def:UCB}
hold and all $\mathcal{GP}$ quantities,
including the fantasy models,
use $\hat{\lambda}=1+2/T$ following \citet{chowdhury2017kernelized}.
Fix $\delta\in(0,1)$; then for any admissible LLM
agent $\mathcal{A}$, with probability at least $1-\delta$:

\begin{enumerate}
  \item[(i)] \textbf{\texttt{LLINBO-Transient}.} Using GP-UCB with $\beta_t(\delta/2)$ and independent Bernoulli draws, if $\sum_{t=1}^{T}(1-p_t)=\mathcal{O}(\sqrt{T})$,
        then $R_T=\mathcal{O}\big(B\sqrt{T}+\beta_T(\delta/2)\sqrt{T\gamma_T}\big)$.
  \item[(ii)] \textbf{\texttt{LLINBO-Justify}.} \textcolor{black}{Using GP-UCB with $\beta_t(\delta)$, if} $\psi_t\ge 0$ and $\sum_{t=1}^{T}\psi_t=\mathcal{O}(\sqrt{T})$,
        then $R_T=\mathcal{O}\big(\sqrt{T}+\beta_T(\delta)\sqrt{T\gamma_T}\big)$.
  \item[(iii)] \textbf{\texttt{LLINBO-Constrained}.} When $I_t=\emptyset$, the fallback maximizes GP-UCB with $\beta_t(\delta/2)$. If $1\le S_t\le S_{\max}$ and
        $\tilde{\beta}_t = 2B + 2R\sqrt{2(\gamma_t + 1 + \ln(2/\delta))} + \sqrt{2\ln(2S_{\max}T/\delta)}$, then
        $R_T=\mathcal{O}\big(\sqrt{T\gamma_T\ln(S_{\max}T)}\,\big(\sqrt{\gamma_T}+\sqrt{\ln(S_{\max}T)}\big)\big)$.
\end{enumerate}
\end{theorem}

Since $\beta_T=\mathcal{O}(\sqrt{\gamma_T})$, all three bounds are $\mathcal{O}(\gamma_T\sqrt{T})$ up to logarithmic factors and are therefore sublinear whenever $\gamma_T=o(\sqrt{T/\ln T})$, which holds for common kernels, e.g., the squared-exponential or sufficiently smooth Mat\'ern kernels \citep{srinivas2009gaussian,vakili2021information}.
 
\paragraph{Choosing the parameters.} Theorem~\ref{thm:noharm} specifies the admissible range of each variant's hyperparameter, which we use to guide its choice. \emph{(i) $p_t$ in \texttt{LLINBO-Transient}:} any schedule with $\sum_{t=1}^{T}(1-p_t)=\mathcal{O}(\sqrt{T})$ is admissible. For example, $p_t=\min(t^2/T,1)$ gives $\sum_t(1-p_t)\approx\tfrac{2}{3}\sqrt{T}$. \emph{(ii) $\psi_t$ in \texttt{LLINBO-Justify}:} any $\psi_t\ge 0$ with
$\sum_t\psi_t=\mathcal{O}(\sqrt{T})$ is admissible, e.g., $\psi_0/\sqrt{t}$; we use the more conservative $\psi_t=\psi_0/t$. \textcolor{black}{We use $\psi_0=\max_{x\in\mathcal{X}}\sigma_0^2(x)$, the largest posterior variance of the $\mathcal{GP}$ fitted to the initial data $\mathcal{D}_0$ (computed on a grid of $10^D$ points, or $10^5$ Sobol points when $D>6$), as a heuristic uncertainty-based scale.} \emph{(iii) $S_t$ in \texttt{LLINBO-Constrained}:} any bounded schedule is admissible. $S_t$ controls how readily the LLM's suggestion is acted on: with $q_t=\Phi\big((\mu_{t-1}(x_{\text{LLM},t})-\kappa_{t-1})/\sigma_{t-1}(x_{\text{LLM},t})\big)$,
at least one sample is retained with probability $1-(1-q_t)^{S_t}$. We therefore recommend a large $S_1$ to exploit the LLM early and a decreasing schedule $S_t=\lceil S_1/t^2\rceil$ to rely increasingly on the $\mathcal{GP}$, which Appendix~\ref{app:constrained} confirms empirically.

\section{Numerical Studies} 
\label{sec:num}

We test the effectiveness of the three \texttt{LLINBO} variants in (i) BBO and (ii) HPT tasks by integrating \texttt{LLAMBO-light} with GPT-3.5-Turbo, a light version of LLAMBO \citep{liu2024large} as $\mathcal{A}$ in Algorithm~\ref{algo:LLMIBO}, (iii) a stress test for \texttt{LLINBO}, where $x_{\text{LLM},t},\ \forall t\in[T]$ in Algorithm~\ref{algo:LLMIBO} is sampled randomly from $\mathcal{X}$, and (iv) computational cost. Note that the proposed framework can make any LLM-based BO no-harm in cumulative regret (Theorem~\ref{thm:noharm}), and hence we also consider integrating $\mathcal{A}$=LLMP \citep{requeima2024llm} and LLAMBO-light with newer GPT model (GPT-5 mini) in Appendix~\ref{app:LLMP} and \ref{app:newmodel} respectively.

\textbf{Baselines.} We compare against classical BO and \texttt{LLAMBO} \citep{liu2024large}. Since \texttt{LLAMBO} requires extensive prompting for candidate generation and surrogate evaluation, we also construct \texttt{LLAMBO-light}, which directly prompts the LLM with the problem context and history to produce the next design. \texttt{LLAMBO-light} serves both as the $\mathcal{A}$ embedded in \texttt{LLINBO} and as a standalone baseline, ensuring a controlled comparison against the same agent. 
\begin{figure*}[h]

    \centering

    \begin{subfigure}[t]{0.31\textwidth}
        \includegraphics[width=\linewidth]{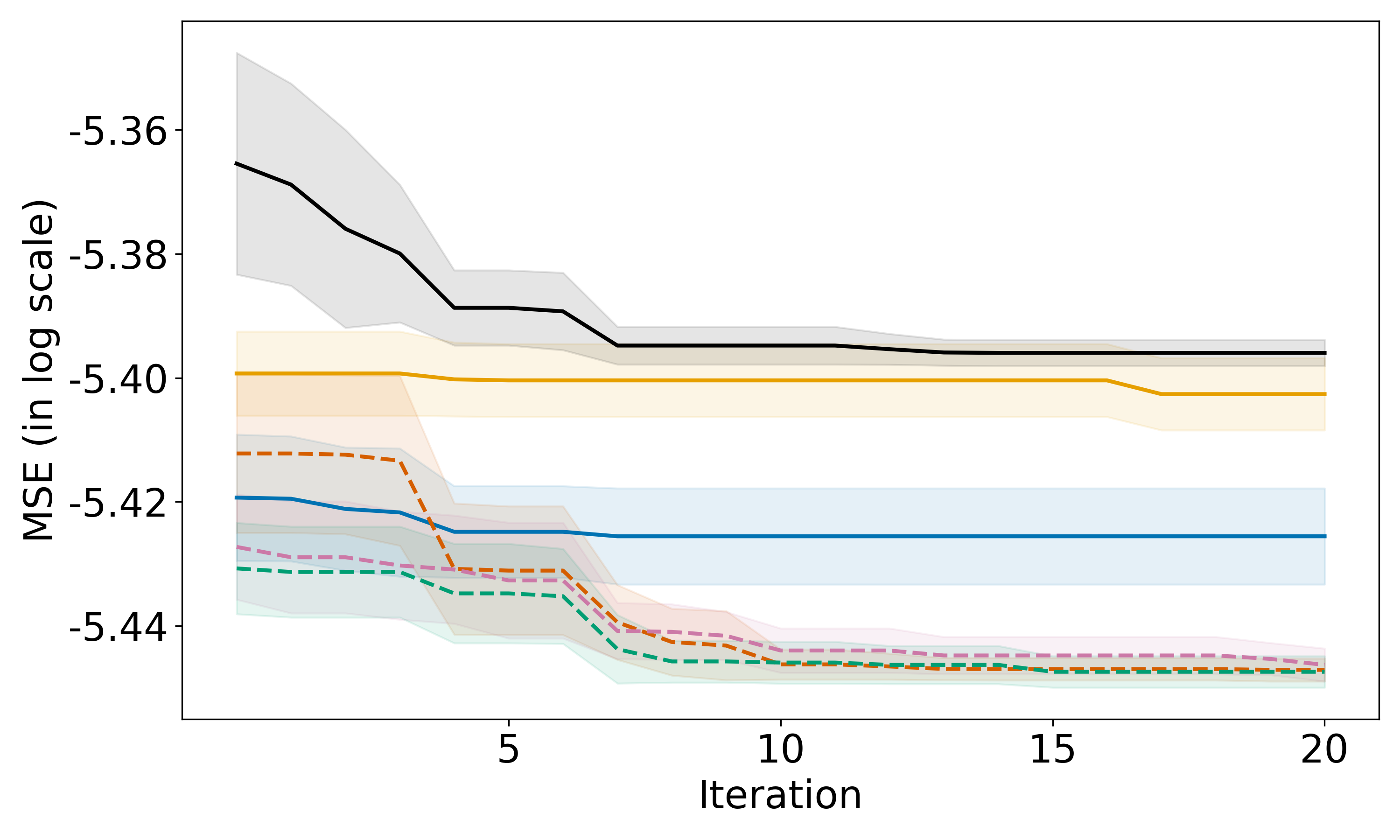}
        \caption{Piston with RF-4$D$}
    \end{subfigure}
    \hfill
    \begin{subfigure}[t]{0.31\textwidth}
        \includegraphics[width=\linewidth]{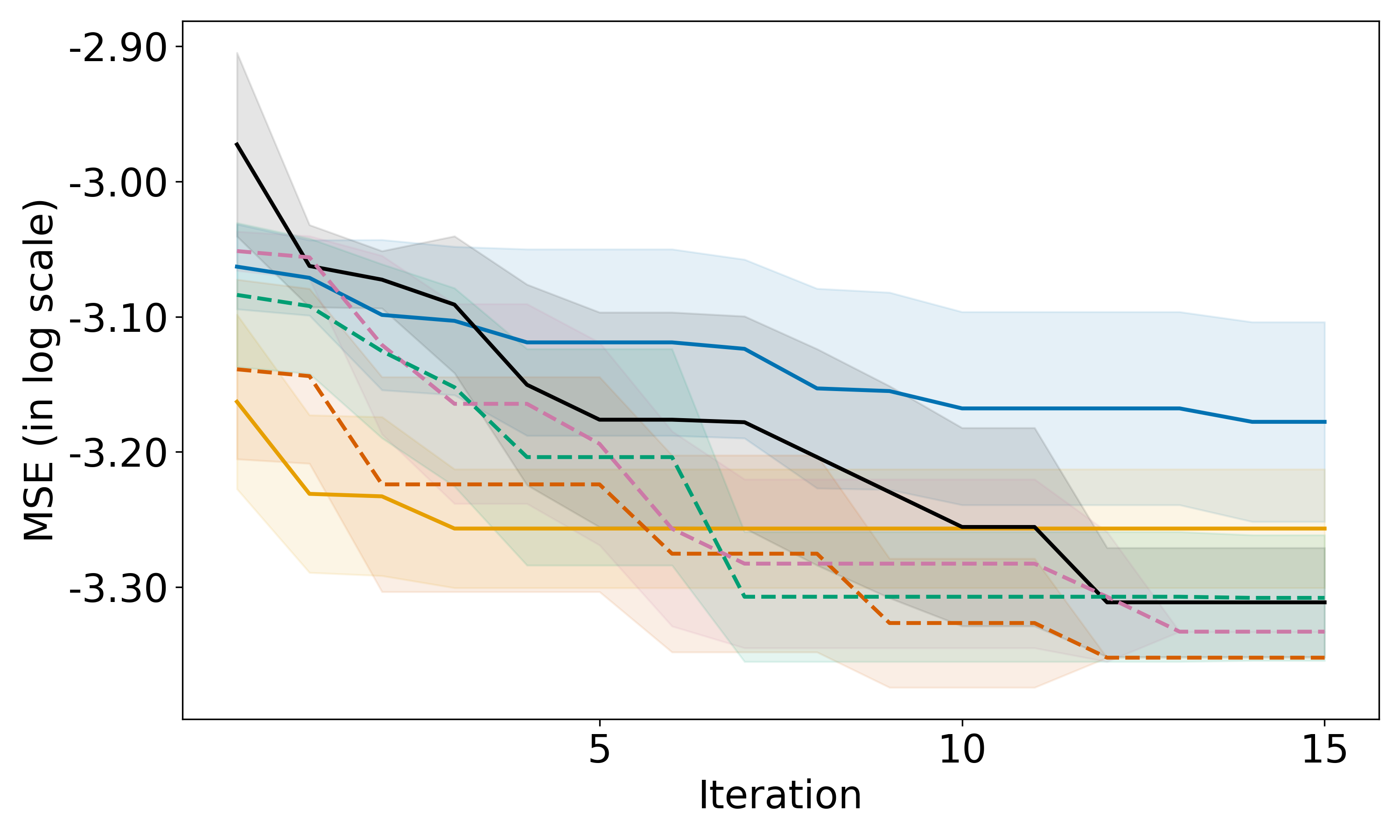}
        \caption{Piston with SVR-3$D$}
    \end{subfigure}
    \hfill
    \begin{subfigure}[t]{0.31\textwidth}
        \includegraphics[width=\linewidth]{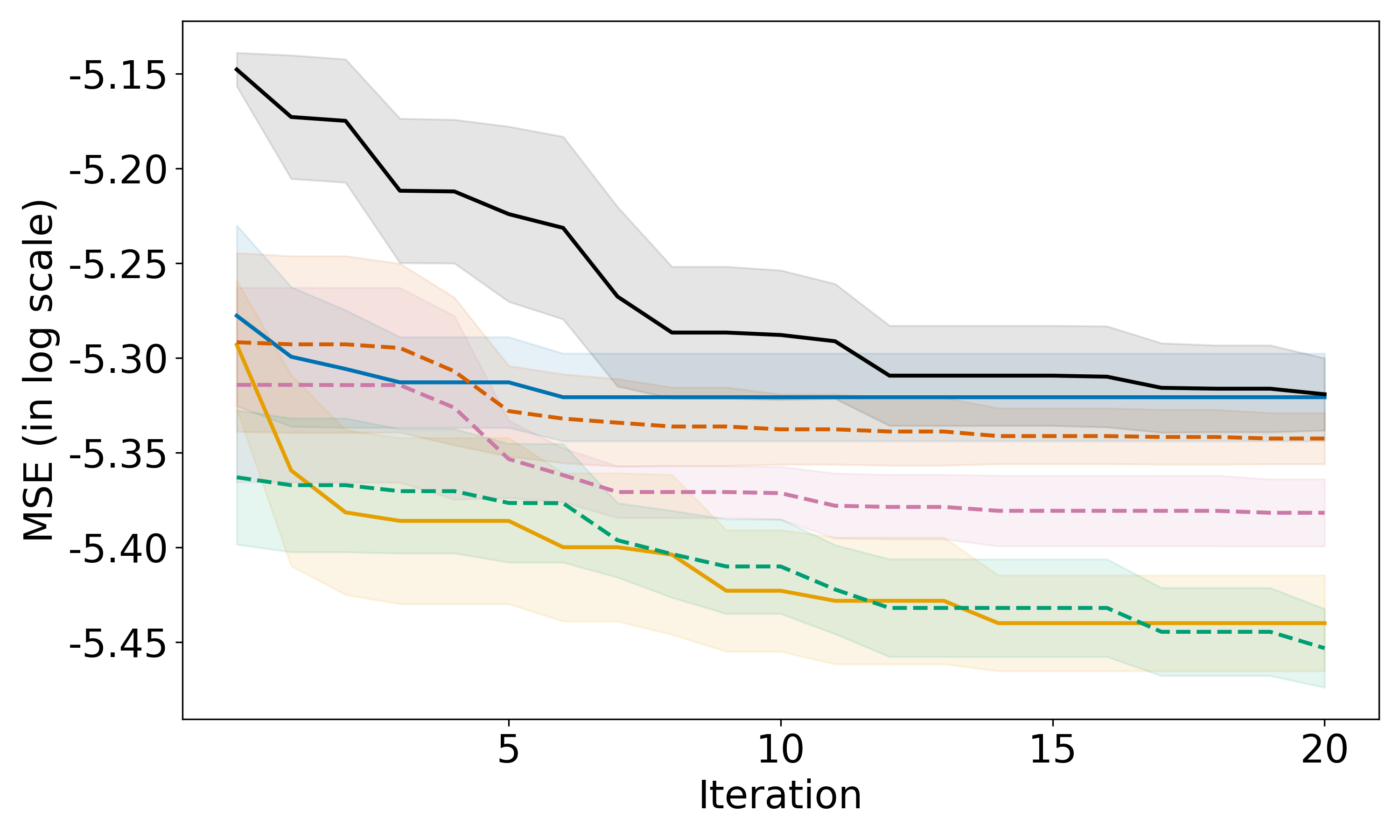}
        \caption{Piston with XGB-4$D$}
    \end{subfigure}

    \vspace{0.1em}

    \begin{subfigure}[t]{0.31\textwidth}
        \includegraphics[width=\linewidth]{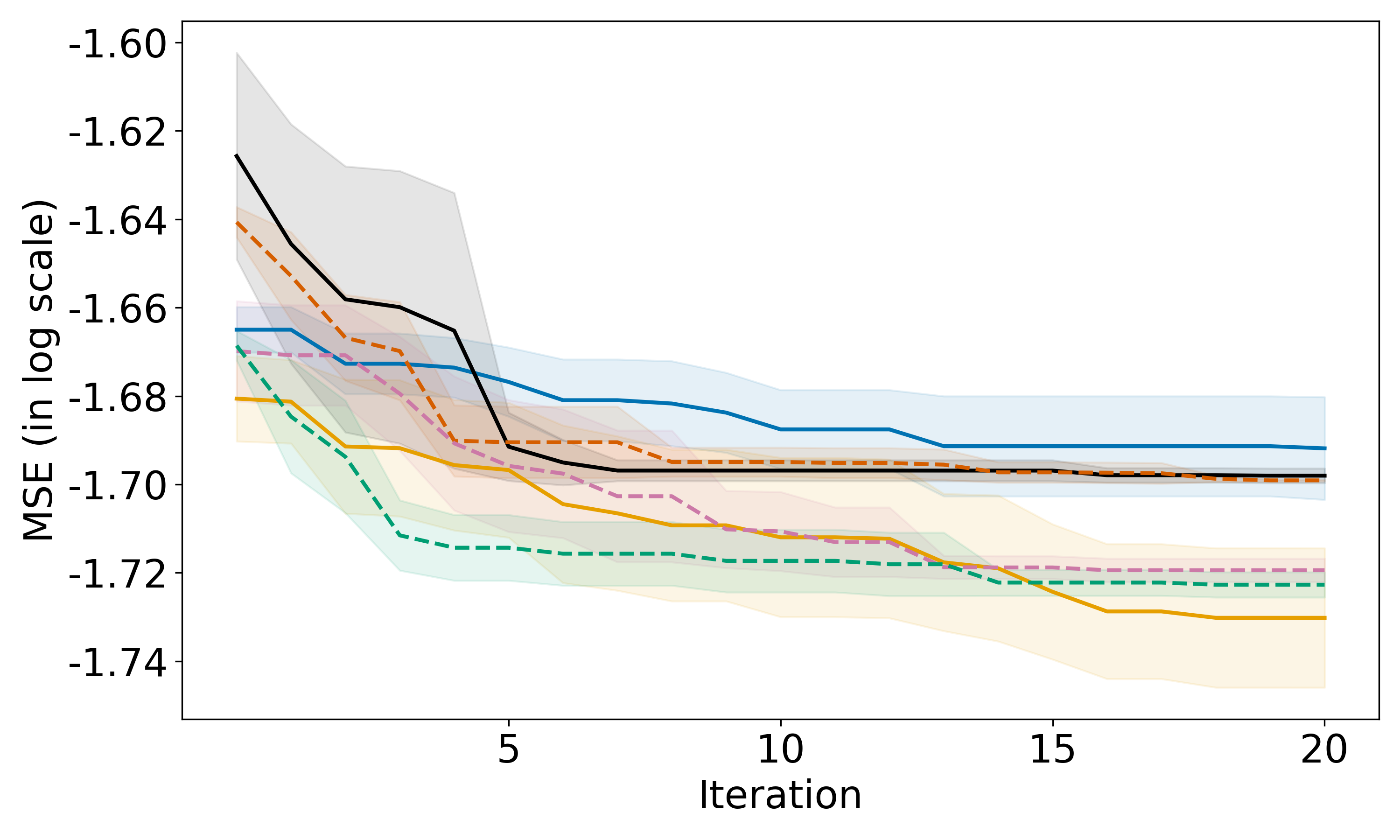}
        \caption{Robot with RF-4$D$}
    \end{subfigure}
    \hfill
    \begin{subfigure}[t]{0.31\textwidth}
        \includegraphics[width=\linewidth]{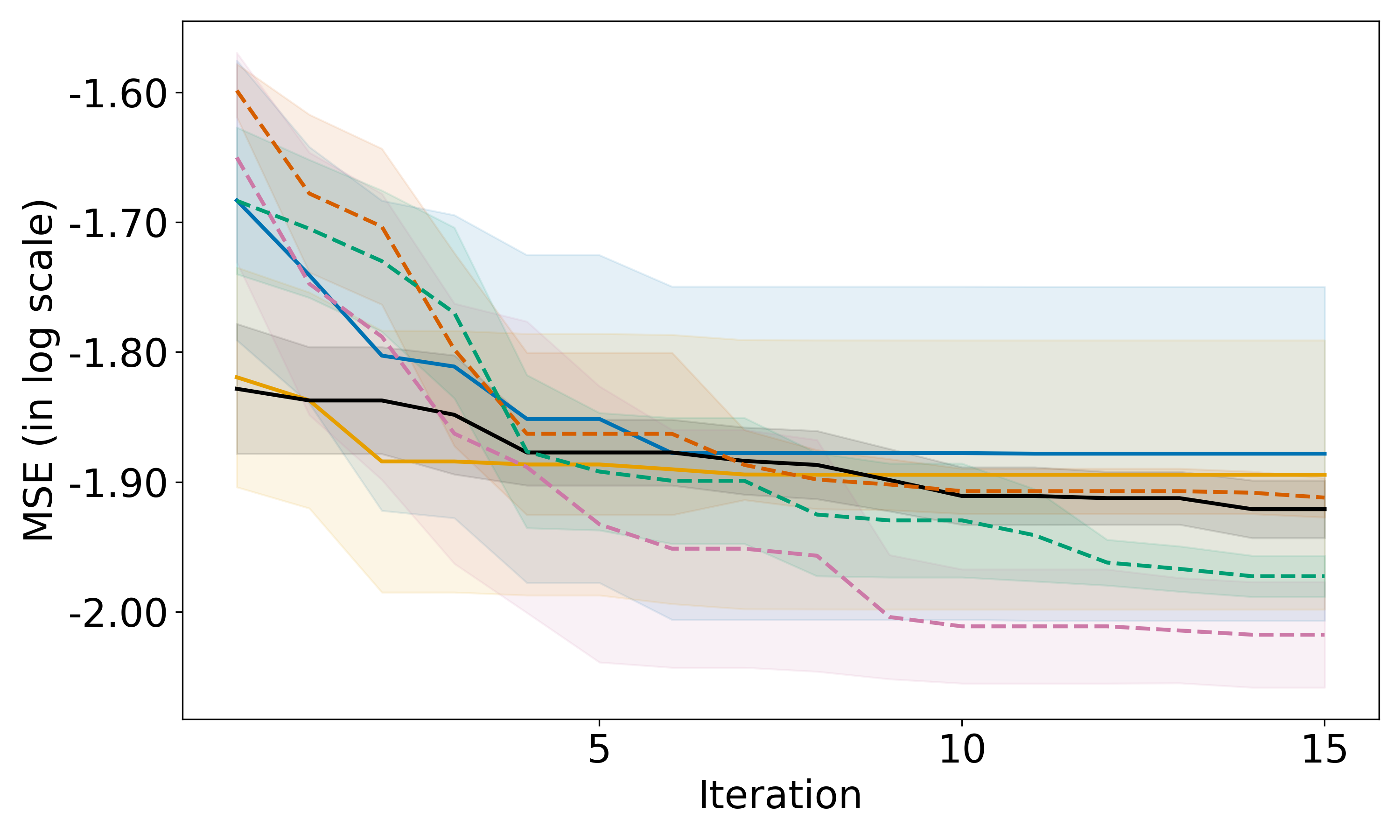}
        \caption{Robot with SVR-3$D$}
    \end{subfigure}
    \hfill
    \begin{subfigure}[t]{0.31\textwidth}
        \includegraphics[width=\linewidth]{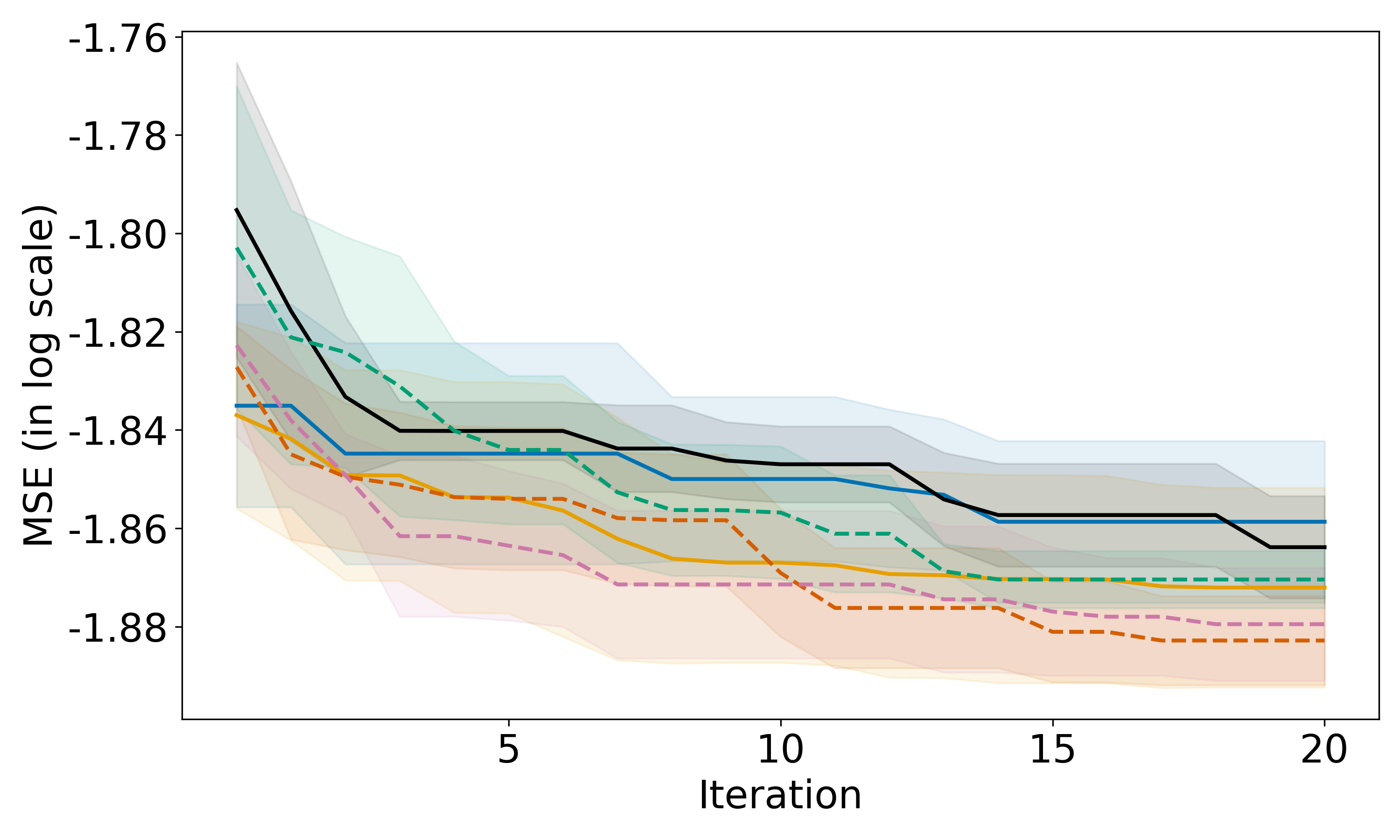}
        \caption{Robot with XGB-4$D$}
    \end{subfigure}

\caption{
MSE comparison for HPT. Each line shows the mean regret over 10 independent replications; shaded regions denote 95\% confidence intervals of the mean. \textbf{Proposed methods}: \dashedline{darkgreen}{1.5em} \texttt{LLINBO-Transient}, 
\dashedline{magenta}{1.5em} \texttt{LLINBO-Justify}, 
\dashedline{brickred}{1.5em} \texttt{LLINBO-Constrained}. \textbf{Baselines}:
\solidline{burntorange}{1.5em} \texttt{LLAMBO}, 
\solidline{cobaltblue}{1.5em} \texttt{LLAMBO-light}, 
\solidline{black}{1.5em} BO.}
\label{fig:hpores}
\end{figure*}

\begin{figure*}[ht]
\centering
    \begin{subfigure}[t]{0.32\textwidth}
        \includegraphics[width=\linewidth]{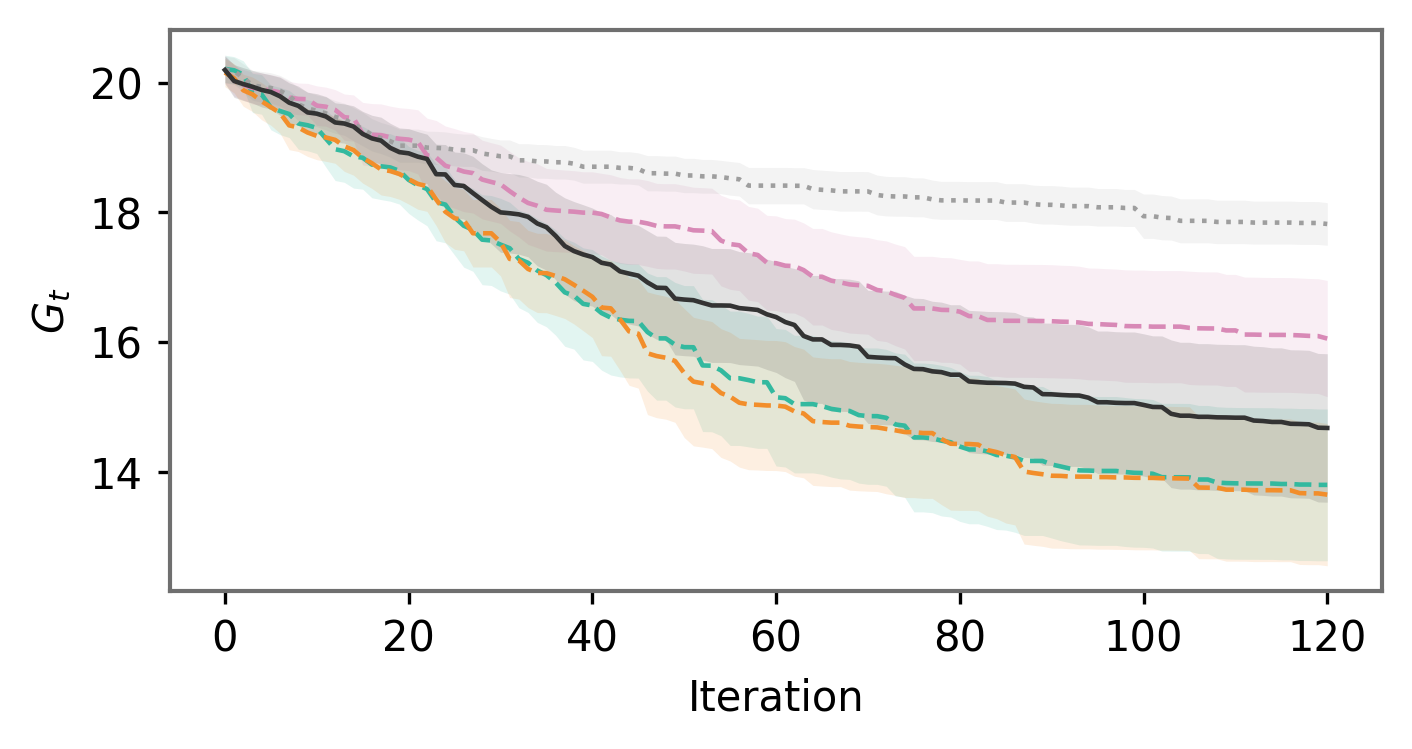}
        \caption{Ackley-6$D$}
    \end{subfigure}
    \hfill
    \begin{subfigure}[t]{0.32\textwidth}
        \includegraphics[width=\linewidth]{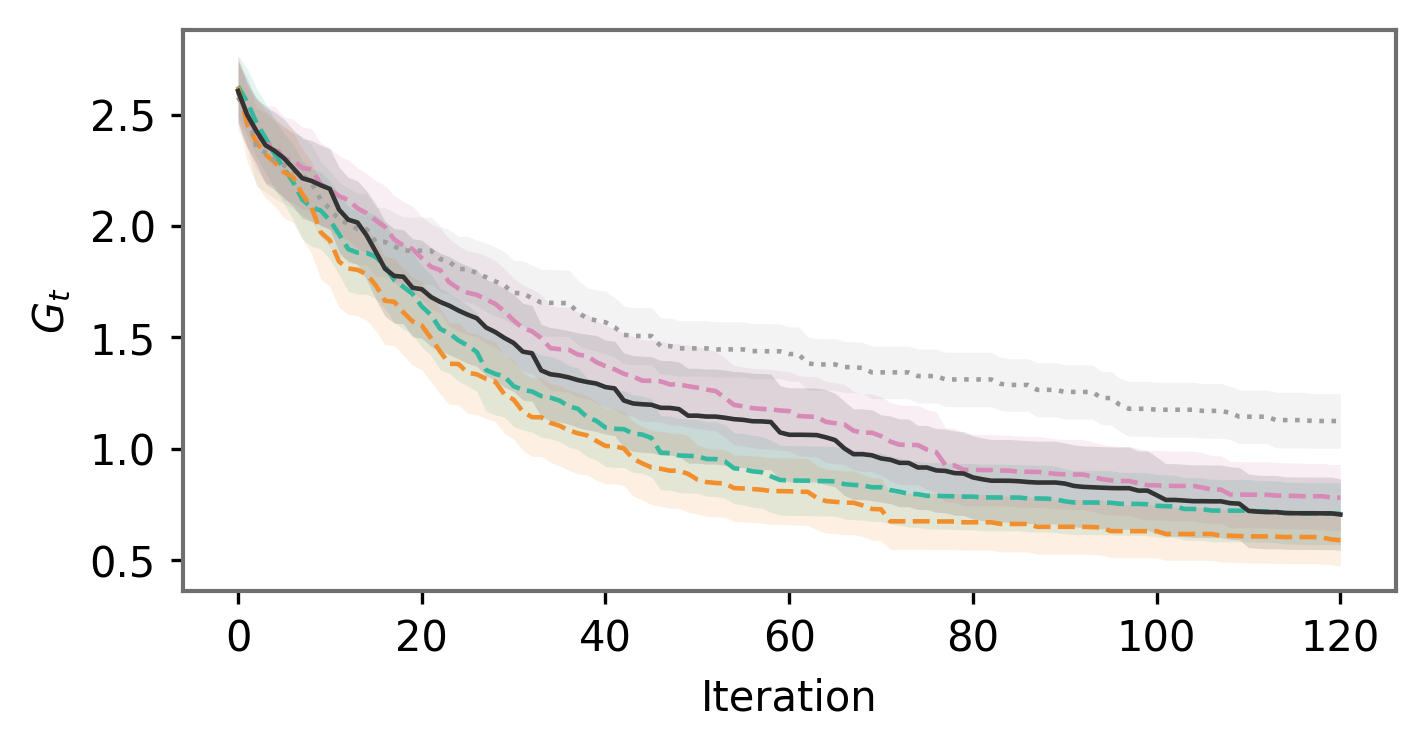}
        \caption{Hartmann-6$D$}
    \end{subfigure}
    \hfill
    \begin{subfigure}[t]{0.32\textwidth}
        \includegraphics[width=\linewidth]{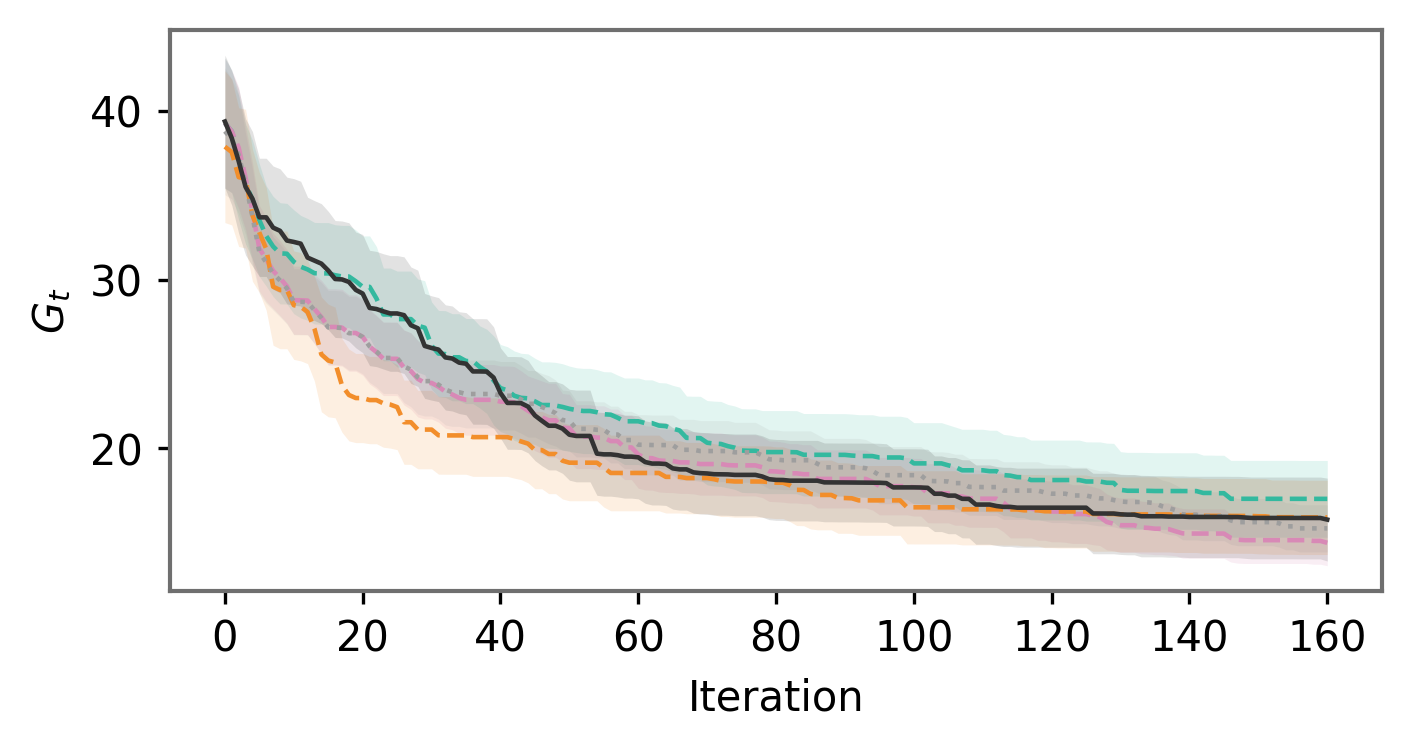}
        \caption{Levy-8$D$}
    \end{subfigure}

\caption{
$G_t$ comparison for BBO under the stress-test setting. Each line shows the mean regret over 10 independent replications; shaded regions denote 95\% confidence intervals of the mean. \textbf{Proposed methods}: \dashedline{darkgreen}{1.5em} \texttt{LLINBO-Transient}, 
\dashedline{magenta}{1.5em} \texttt{LLINBO-Justify}, 
\dashedline{brickred}{1.5em} \texttt{LLINBO-Constrained}. 
\textbf{Baselines}:
\dashedline{3dgray}{1.5em} Random Search,
\solidline{black}{1.5em} BO.
}
\label{fig:stress}
\end{figure*}

\begin{figure*}[h]
    \centering
    \begin{subfigure}[t]{0.18\textwidth}
        \centering
        \includegraphics[width=\textwidth,height=1.27\textwidth]{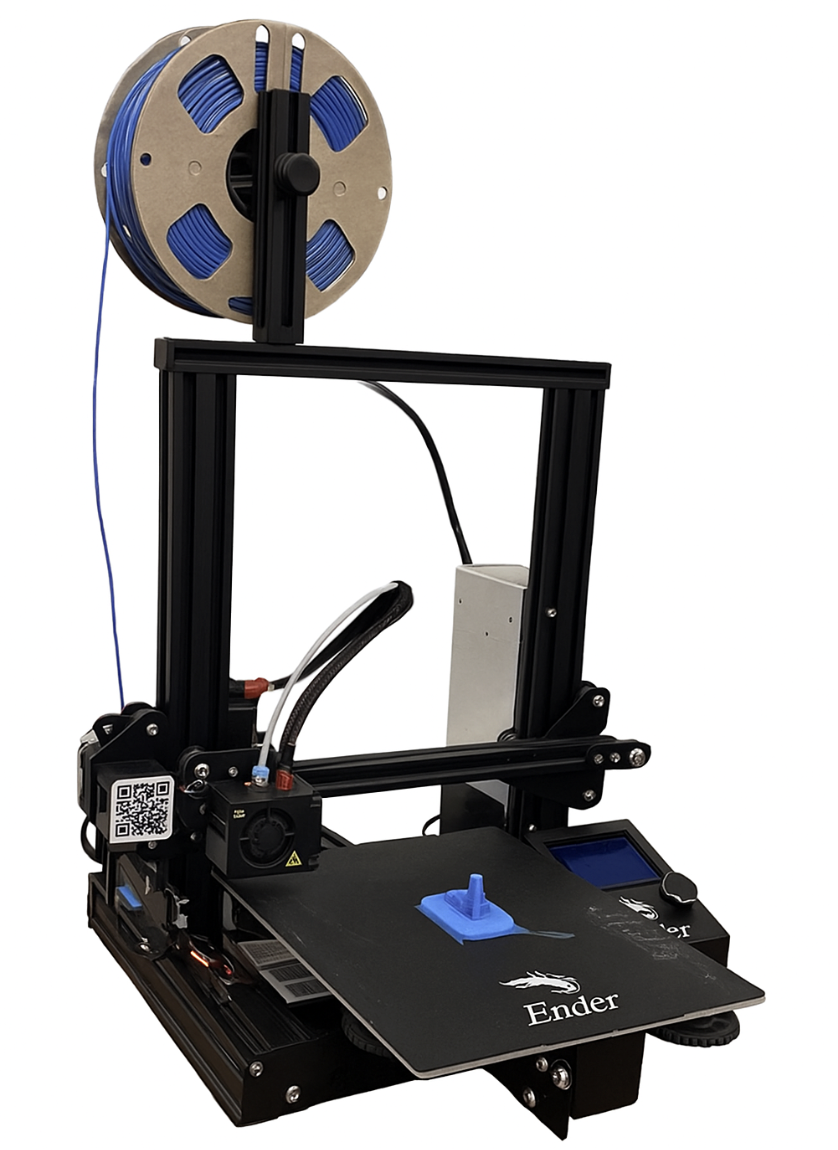}
        \caption{}
    \end{subfigure}
    \hfill
    \begin{subfigure}[t]{0.26\textwidth}
        \centering
\includegraphics[width=\textwidth]{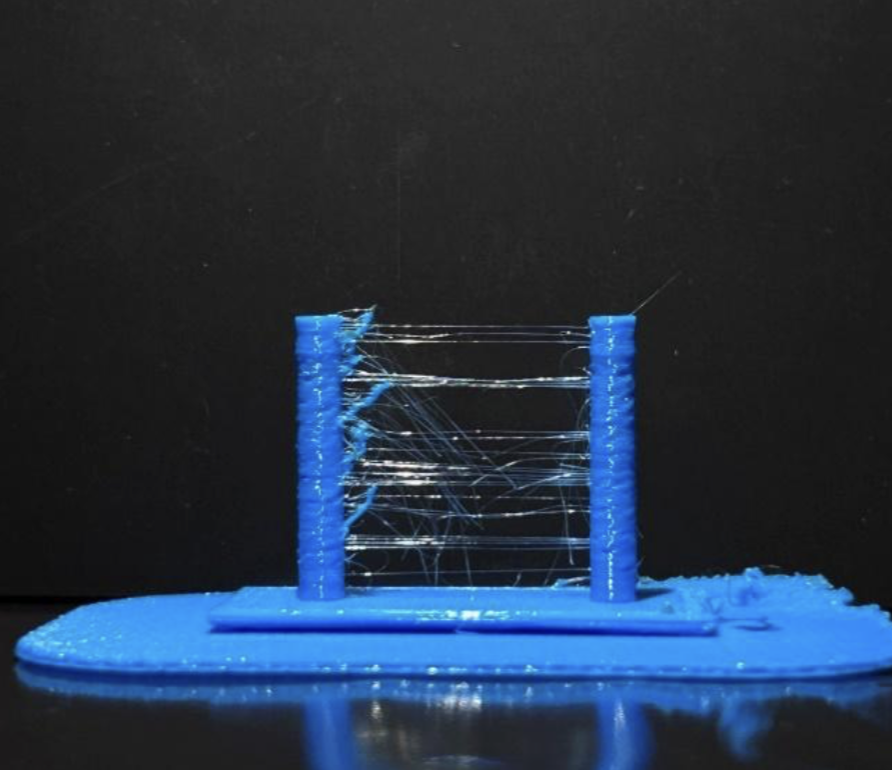}
        \caption{}
    \end{subfigure}
    \hfill
    \begin{subfigure}[t]{0.5\textwidth}
        \centering
        \includegraphics[width=\textwidth,height=0.46\textwidth]{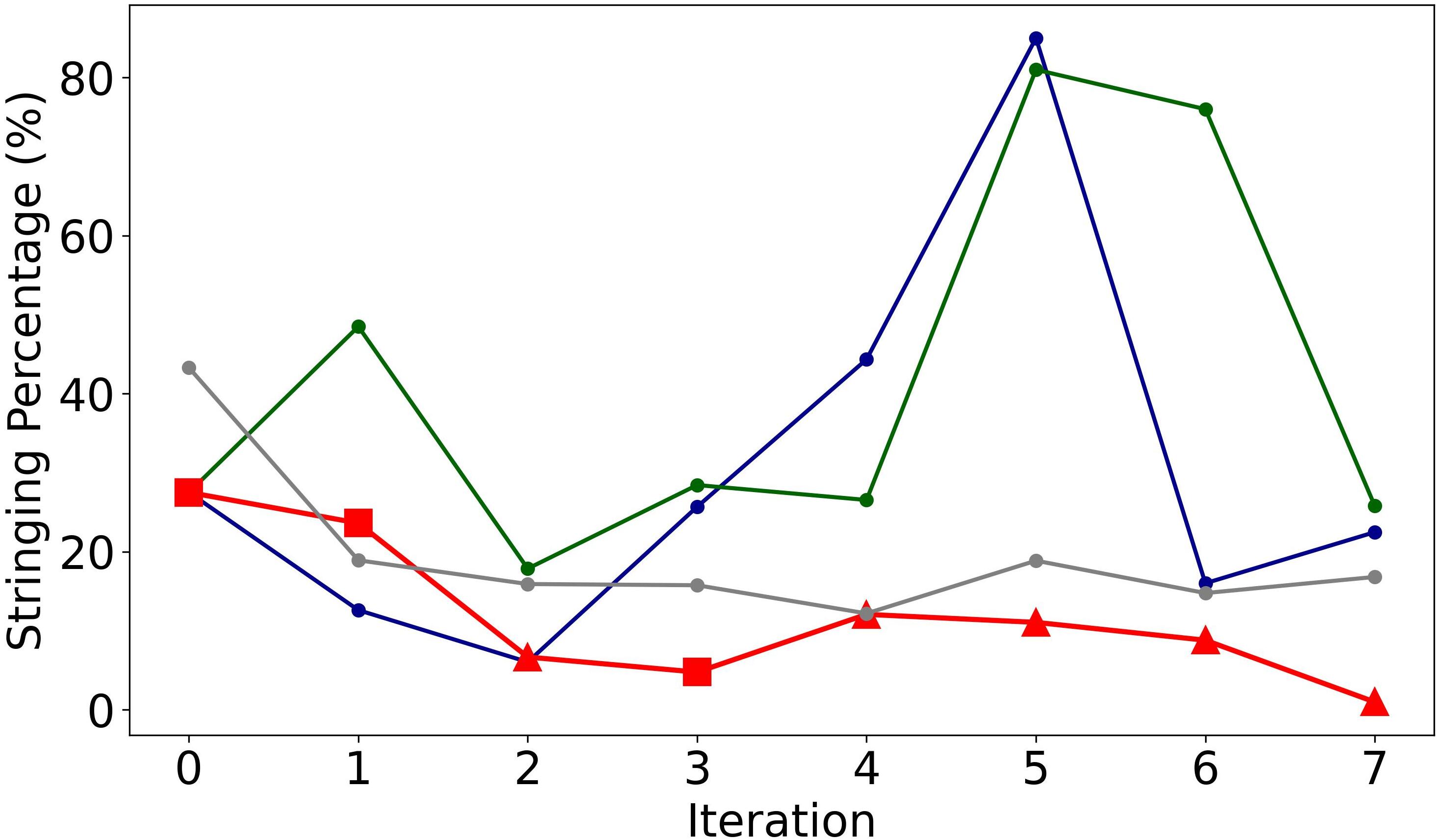}
        \caption{}

    \end{subfigure}
    \caption{3D printing experiments and results. \textbf{(a):} 3D printer, \textbf{(b):} stringing between two columns, \textbf{(c):} benchmark results.
Methods: 
\textcolor{3dblue}{\textbf{---} \texttt{LLAMBO-light}},
\textcolor{3dgreen}{\textbf{---} \texttt{LLAMBO}}, \textcolor{3dgray}{\textbf{---} BO} and 
\textcolor{black}{\textbf{---} \texttt{LLINBO-Transient}} (we use square and triangle markers to indicate updates chosen based on an LLM or $\mathcal{GP}$, respectively).}
    \label{fig:stringingvisual}
\end{figure*}

For each task with a $D$-dimensional design space, we generate an initial dataset $\mathcal{D}_0$ with $D$ observations. This is done via prompting within the problem context, also known as warmstarting, for methods that utilize LLMs, and via random sampling for BO. To capture the uncertainty in each method’s performance, we perform a total of 10 replications. All results are averaged over 10 independent replications with different random seeds and initial datasets $\mathcal{D}_0$, with shaded regions denoting 95\% confidence intervals of the mean. The surrogate model is a \(\mathcal{GP}\) with a constant prior mean and a Mat\'ern-5/2 kernel (details in Appendix~\ref{sec:appb}). Detailed implementation of the LLM agent, including the structured templates and context for each problem, can be found in Appendix~\ref{sec:appb}. We use UCB as the AF, and set the relevant parameters as follows: $p_t = \min(t^2/T, 1)$, $S_t=\lceil 10^4/t^2\rceil$, $\psi_t=\psi_0/t$ with $\psi_0=\max_{x}\sigma_0^2(x)$, and $\beta_t=2\log\frac{tD\pi^2}{6\cdot 0.1}$ \citep{srinivas2009gaussian}. For the experiments, we adopt
an exploration schedule motivated by
\citet{srinivas2009gaussian}, rather than the
RKHS-based confidence coefficient specified
in Assumption~\ref{def:UCB}. Following the BoTorch convention, $\beta_t$ enters the UCB as $\mu_{t-1}(x)+\beta_t^{1/2}\sigma_{t-1}(x)$, and we use $\tilde{\beta}_t=\beta_t$ in the $\mathcal{CGP}$-UCB of \texttt{LLINBO-Constrained}.

 \textbf{BBO task (Figure~\ref{fig:bbfores}) and HPT task (Figure~\ref{fig:hpores}).}
For BBO, we utilize six commonly used simulation functions: Levy-2$D$, Rastrigin-2$D$, Branin-2$D$, Bukin-2$D$, Hartmann-4$D$, and Ackley-6$D$ (see Appendix~\ref{sec:appb1} for details of each function). For each function, its characteristic patterns and the objective of the problem are incorporated into the prompts as part of the problem context (see Appendix~\ref{sec:appb1} for prompt details). Performance is reported in terms of the best observed regret, defined as \( G_t = f(x^*) - \max_{i\leq t} y_i \), where \( f(x^*) \) denotes the true global maximum. The total budget is set to $T = 10D$.

In the HPT task, we consider two physical simulation functions: the Piston (\cite{kenett1998modern}) and robot arm (Robot, \cite{an2001quasi}), along with three regression models: Random Forest (RF-4$D$), Support Vector Regression (SVR-3$D$), and XGBoost (XGB-4$D$). The total budget is set to $T=5D.$ For each dataset-regression-model pair, we generate $1{,}000$ data points and define the regret at each iteration as the best (lowest) observed mean squared error (MSE) so far, where the MSE is computed by fitting the corresponding regression model and evaluating it using $10$-fold cross-validation. Each data-model pair and the problem formulation are described in the prompt (Appendix~\ref{sec:appb2}).

Across both the BBO (Fig.~\ref{fig:bbfores}) and HPT (Fig.~\ref{fig:hpores}), a consistent pattern emerges. \textit{LLM-based BO} methods, including \texttt{LLAMBO-light} and \texttt{LLAMBO}, significantly underperform compared to other benchmarks, with regret curves that often remain flat. This drawback, especially in higher dimensions (e.g. Hartmann-4$D$ and Ackley-6$D$), highlights that LLMs can assist with optimization but are not yet reliable as standalone agents. At the same time, methods involving LLMs, including the three variants \texttt{LLINBO}, achieve a strong early lead by effectively leveraging problem context to identify promising regions or generate high-quality designs in the early iterations. Our hybrid mechanisms consistently outperform the benchmarks across almost all functions and tasks, with this advantage being especially pronounced early on and gradually diminishing as more data is collected. This is further reinforced by Fig.~\ref{fig:hpores}, where our proposed LLM-GP collaborative mechanisms yield significantly lower MSE than all benchmarks, including BO, across the tasks.

\textbf{Stress test (Fig.~\ref{fig:stress})} To further evaluate the robustness of \texttt{LLINBO}, we conduct a stress test in which the LLM provides completely uninformative suggestions. We consider three higher-dimensional benchmark functions: Ackley-6$D$, Hartmann-6$D$, and Levy-8$D$.

Figure~\ref{fig:stress} shows that the proposed variants remain competitive under this uninformative weak-agent setting. In particular, \texttt{LLINBO-Transient} and \texttt{LLINBO-Constrained} continue to reduce regret steadily across the optimization process and generally outperform the random-search baseline. \texttt{LLINBO-Justify} is more conservative, but still avoids the failure mode of blindly following uninformative LLM recommendations.

\begin{table}[h]
\centering
\small
\setlength{\tabcolsep}{4pt}
\begin{tabular}{lcc}
\hline
\textbf{Method} & \textbf{Rastrigin-2$D$} & \textbf{Robot RF-4$D$} \\
\hline
BO                          & $12 \pm 1$    & $49 \pm 4$ \\
\texttt{LLAMBO}             & $907 \pm 12$  & $3628 \pm 48$ \\
\texttt{LLAMBO-light}       & $144 \pm 5$   & $432 \pm 11$ \\
\texttt{LLINBO-Transient}   & $93 \pm 11$   & $278 \pm 18$ \\
\texttt{LLINBO-Justify}     & $167 \pm 4$   & $668 \pm 8$ \\
\texttt{LLINBO-Constrained} & $197 \pm 15$  & $719 \pm 23$ \\
\hline
\end{tabular}
\caption{Wall-clock runtime (seconds).}
\label{tab:comp_time}
\end{table}

\textbf{Computational cost (Tab.~\ref{tab:comp_time}).}
Table~\ref{tab:comp_time} reports wall-clock runtime on Rastrigin-2$D$ (BBO) and Robot with RF-4$D$ (HPT) over 10 runs, measured on five nodes, each with dual Intel Xeon Platinum CPUs and 512\,GB RAM. BO is fastest as it never queries an LLM, while \texttt{LLAMBO} is the most expensive due to its multi-stage prompting. \texttt{LLINBO-Transient} is faster than \texttt{LLAMBO-light}, as it queries the LLM only when $z_t=0$; \texttt{LLINBO-Justify} and \texttt{LLINBO-Constrained} add modest $\mathcal{GP}$ computation to \texttt{LLAMBO-light}, yet remain about $5\times$ faster than \texttt{LLAMBO}. In practice, BO is typically applied to expensive experiments (e.g., the 3D printing study in Sec.~\ref{sec:real}). The overhead of \texttt{LLINBO}, at most a few minutes per run, is therefore negligible compared with the cost of a single wasted experiment, which LLM-based BO can incur.

\textbf{Additional experiments.} We repeat the experiments with a newer LLM (GPT-5 mini; Appendix~\ref{app:newmodel}). While BO starts from a randomly sampled $\mathcal{D}_0$ and the LLM-based methods start from an LLM-generated one, in Appendix~\ref{app:llmibo} we consider a stronger baseline, LLM-initialized BO (\texttt{LLMiBO}), which uses the LLM for warmstarting and then runs BO for the remaining iterations. We further consider a different agent, $\mathcal{A}=$ \texttt{LLMP} (Appendix~\ref{app:LLMP}), a comparison against $\pi$BO (an HAIC-BO, Appendix~\ref{app:pibo}), higher-dimensional tasks (Levy-10$D$, Ackley-15$D$; Appendix~\ref{app:highdim}), and a prompt ablation on Robot with RF-4$D$ (Appendix~\ref{sec:ablation}) that varies the context given to $\mathcal{A}$ from fully to minimally informed.

\section{Application to 3D Printing}
\label{sec:real}
In addition to the numerical evaluation, we further assess the performance of our method through a case study in 3D printing, aimed at reducing stringing in a printed product. Stringing (Fig.~\ref{fig:stringingvisual}(b)) is a prevalent defect in fused filament fabrication (FFF) 3D printing. FFF is commonly used for rapid prototyping and low-cost part production. However, stringing degrades surface quality and often requires additional post-processing (\cite{paraskevoudis2020real}). This study aims to optimize the design parameters of a Creality Ender 3 Pro desktop FFF printer (Fig.~\ref{fig:stringingvisual}(a)), including nozzle temperature, Z hop height, retraction distance, outer wall wipe distance, and coasting volume, to minimize stringing percentage. Details about the parameters can be found in Appendix~\ref{sec:appc}.

\textbf{Experiment setup.} All experiments were conducted on a single printer using PETG filament (\cite{holcomb2022optimized}), selected for its high tendency to produce stringing (see Fig. \ref{fig:stringingvisual}(b)). We adopted a standard two-column geometry with a horizontal gap, commonly used in stringing evaluations (\cite{haque2020minimizing}). At each iteration, after printing the object with the proposed parameters, the stringing percentage (ranging from 0 to 100\%) was quantified (details in Appendix~\ref{sec:appc1}).

Due to the cost associated with this experiment (each run takes several hours), we limit our comparison to \texttt{LLINBO-Transient} with \( p_t = 1 - \frac{1}{t} \), evaluated against \texttt{LLAMBO}, \texttt{LLAMBO-light}, and BO, with a budget of $T=10$; the LLM-based methods are warmstarted with two designs. All other settings follow Sec.~\ref{sec:num}. Unlike the simulation studies, for \texttt{LLAMBO} we sample 20 candidates per iteration and estimate each candidate's mean and uncertainty from 20 surrogate queries with permuted histories. For BO, the acquisition function is maximized over the discrete grid of admissible parameter values (the ranges and step sizes in Appendix~\ref{sec:appc2}). The prompts specifying the problem context and controllable parameters are provided in Appendix~\ref{sec:appc2}. Unlike Sec.~\ref{sec:num}, the objective here is not a statistically replicated comparison: since each evaluation requires a physical print, we report a single trajectory per method. Instead, this study serves as a physical proof of concept, showing that \texttt{LLINBO} can be deployed in a real experimental process where each evaluation is costly.

Several observations can be drawn from the trajectories shown in Fig.~\ref{fig:stringingvisual}(c): (i) Our method attains lower stringing values in the reported physical trajectory, reaching near-zero stringing by the final iteration. (ii) Methods utilizing LLMs start from lower stringing values than BO, suggesting that LLMs can provide useful initial designs. (iii) \texttt{LLAMBO} and \texttt{LLAMBO-light} do not exhibit a decreasing trend in stringing, which is in line with our simulation results. (iv) BO improves over time but remains above our method in this trajectory. These observations are consistent with the collaboration benefits observed in Sec.~\ref{sec:num}.

\section{Conclusion}
\texttt{LLINBO} combines the contextual reasoning of LLMs, which yields strong early designs, with surrogate-model guidance that refines the search as data accumulate. Across simulations and a real-world case study, the proposed mechanisms are robust and competitive: they make \textit{LLM-based BO} safer while still exploiting the LLM's contextual strengths. We believe this offers a foundation for trustworthy LLM-assisted optimization beyond standard black-box settings.

\textbf{Limitations.} Our framework refines candidate designs proposed by an LLM agent, but LLMs can enter BO in other ways, e.g., by generating code that selects the next design. Extending our guarantees to such paradigms, and linking \texttt{LLINBO}'s hyperparameters to a quantitative measure of LLM reliability, are important directions for future work.

\bibliography{sample}

\clearpage
\appendix
\thispagestyle{empty}

\onecolumn
\aistatstitle{LLINBO: Trustworthy LLM-in-the-Loop Bayesian Optimization \\
Supplementary Materials}

\startcontents[appendix]
\printcontents[appendix]{l}{0}{\section*{Contents of the Appendix}\setcounter{tocdepth}{2}}

\section{More Related Works} \label{app:rw}

\paragraph{F-BO} To enable collaboration between LLMs and statistical surrogates in enhancing BO, while leveraging the distinctive ability of LLMs to provide a set of designs, we draw inspiration from the literature on F-BO. Federated learning (FL) aims to establish a collaborative framework that allows clients to work together while keeping their own data private. This setting has directly influenced prior work in F-BO, where a single design is often shared across clients. 

For example, \citet{yue2025collaborative} developed a consensus framework for collaborative BO, where the next design to query is selected as a weighted combination, dictated by a dynamically coupled stochastic consensus matrix, of the AF maximizers from all clients in the system, including each client’s own. Other works like \cite{chen2025multi} and \cite{dai2020federated} require only the design from other clients, while the former requires a design point from other clients directly, and the latter requires Random Fourier Features (\cite{rahimi2007random}). A recent review on federated and collaborative BO can be found in \cite{al2024collaborative}.

\begin{table}[h]
\centering
\caption{Comparison of \texttt{LLINBO} with related BO frameworks incorporating external information.}
\label{tab:comparison}
\begin{tabularx}{\linewidth}{lXXXX}
\hline
\textbf{Feature} & \texttt{LLINBO} & \textbf{HAIC-BO} & \textbf{F-BO} & \textit{LLM-based BO} \\
\hline
Minimal assumption on external info & \cmark & \xmark\ & \cmark & \cmark/\xmark \\
Theoretical guarantees & \cmark & \cmark/\xmark & \cmark & \xmark \\
Handles early LLM strength, later weakness & \cmark & \xmark & \xmark & \cmark/\xmark \\
Preserves BO structure & \cmark & \cmark & \cmark & \xmark \\
Dynamic reliance adjustment & \cmark & \xmark & \xmark & \xmark \\
\hline
\end{tabularx}
\end{table}

\paragraph{HAIC-BO} In contrast, HAIC-BO requires richer information compared with F-BO, as privacy concerns are not considered in this setting. For instance, COBOL (\cite{xu2024principled}) requires explicit beliefs about the function from the user, while CoExBO (\cite{adachi2024looping}) relies on preference pairs provided by a human. Similarly, $\pi$BO (\cite{hvarfnerpi}) assumes access to a formal prior distribution specified by a human expert, and the method in (\cite{av2022human}) requires information about good and bad regions of the design space. Another key difference between our framework and HAIC-BO lies in the assumptions placed on LLMs or humans. Notably, our method does not impose any assumptions on the capacity of LLMs; instead, it hedges against poor suggestions through three distinct hedging processes. By contrast, HAIC-BO often introduces behavioral assumptions about humans; for example, \cite{av2022human} assumes that a human expert follows a BO-like strategy. 

We acknowledge all works that tried to cooperate with outside information to make decision-making more efficient, and we use the table below to compare and highlight the key differences between the proposed \texttt{LLINBO} framework and the rich existing works, including HAIC-BO, F-BO, and \textit{LLM-based BO}.

\section{Technical Results}
\label{app:tech}
We first introduce two Lemmas that are quite common in BO analysis. Lemma~\ref{lemma:UCBconcentrate} derives the concentration between the posterior mean and the ground truth.

\begin{lemma}(Theorem 2 of \cite{chowdhury2017kernelized})
\label{lemma:UCBconcentrate}
Under Assumption~\ref{assumption:rkhs_noise}, let $\hat{\lambda}=1+2/T$. For arbitrary $\delta\in(0,1)$, with probability at least $1-\delta$, for all $x\in\mathcal{X}$ and all $t\geq 1$, we have:
\begin{align}
| \mu_{t-1}(x) - f(x)| 
&\leq | k_{t-1}(x)^\top (K_{t-1} + \hat{\lambda} I)^{-1} [\delta_1,...,\delta_{t-1}]^\top|\nonumber\\
&\ \ \ \ + | f(x) - k_{t-1}(x)^\top (K_{t-1} + \hat{\lambda} I)^{-1} [f(x_1),...,f(x_{t-1})]^\top | \label{eq:lemma1-1}\\
&\leq (B + R\sqrt{2\left( \gamma_{t-1} + 1 + \ln\left({1}/{\delta}\right) \right)})\sigma_{t-1}(x)\nonumber\\
&=\textcolor{black}{\beta_t(\delta)}\sigma_{t-1}(x)\label{eq:lemma1-2},
\end{align} where $\delta_i=f(x_i)-y_i\ \forall i\in[t-1]$.
\end{lemma}
With this Lemma, we can bound the regret raised at every iteration, which is stated in Lemma~\ref{lemma:UCBregretbd}.
\begin{lemma}\textcolor{black}{(See the proof of Theorem 3 in \citet{chowdhury2017kernelized}.)}
\label{lemma:UCBregretbd}
\textcolor{black}{Assume that Assumptions~\ref{assumption:rkhs_noise} and \ref{def:UCB} hold. Fix $\rho\in(0,1)$. On the confidence event of Lemma~\ref{lemma:UCBconcentrate} at level $\rho$, if $x_t$ maximizes $\alpha_{\mathrm{UCB}}(\cdot,F_{t-1})$ with coefficient $\beta_t(\rho)$, then the regret at time $t$ is bounded by}
\[
\textcolor{black}{r_t=f(x^*)-f(x_t)\leq\beta_t(\rho)\sigma_{t-1}(x_t)+\mu_{t-1}(x_t)-f(x_t)\leq 2\beta_t(\rho)\sigma_{t-1}(x_t).}
\]
\end{lemma}

Next, when using the UCB as the AF, we present a commonly used lemma that bounds the cumulative posterior variance at the selected design points in terms of the information gain.
\begin{lemma}(Lemma 4 in Appendix of \cite{chowdhury2017kernelized})
\label{lemma:sumvarbd}
Let $x_1,\dots,x_T$ be the designs selected by the algorithm. Then, the sum of the predictive standard deviations at these points can be bounded by
\[
\sum_{t=1}^T\sigma_{t-1}(x_t)\leq\sqrt{4(T+2)\gamma_T}=\mathcal{O}(\sqrt{T\gamma_T}).
\]
\end{lemma}

\subsection{Proof of regret bound for \texttt{LLINBO-Transient} in Theorem~\ref{thm:noharm}}
\label{sec:theorem1}
\textcolor{black}{Let $z_t\sim\mathrm{Bernoulli}(p_t)$ denote the independent selection indicator from Sec.~\ref{sec:A1}, with $z_t=1$ selecting the GP-UCB design and $z_t=0$ selecting the LLM design. The schedule $p_t$ is fixed in advance. In this variant the GP-UCB design maximizes the acquisition function with coefficient $\beta_t(\delta/2)$.}

\textcolor{black}{Let $\mathcal E_{\mathrm{GP}}$ denote the simultaneous GP confidence event of Lemma~\ref{lemma:UCBconcentrate} with confidence level $\delta/2$. This event has probability at least $1-\delta/2$. On $\mathcal E_{\mathrm{GP}}$, if $z_t=1$, Lemma~\ref{lemma:UCBregretbd} gives}
\[
\textcolor{black}{r_t\leq 2\beta_t(\delta/2)\sigma_{t-1}(x_t).}
\]
\textcolor{black}{If $z_t=0$, the reproducing property and Assumption~\ref{assumption:rkhs_noise} yield $|f(x)|\leq B\sqrt{k(x,x)}\leq B$, so $r_t\leq 2B$. Since $\beta_t(\delta/2)$ is nondecreasing, these two cases imply, on $\mathcal E_{\mathrm{GP}}$,}
\[
\textcolor{black}{R_T\leq 2\beta_T(\delta/2)\sum_{t=1}^{T}\sigma_{t-1}(x_t)+2B\sum_{t=1}^{T}(1-z_t).}
\]
\textcolor{black}{Because the $z_t$ are independent Bernoulli variables with deterministic success probabilities $p_t$, Hoeffding's inequality implies, with probability at least $1-\delta/2$,}
\[
\textcolor{black}{\sum_{t=1}^{T}(1-z_t)\leq\sum_{t=1}^{T}(1-p_t)+\sqrt{\frac{T}{2}\ln\frac{2}{\delta}}.}
\]
\textcolor{black}{A union bound ensures that this event and $\mathcal E_{\mathrm{GP}}$ hold simultaneously with probability at least $1-\delta$. On their intersection, applying Lemma~\ref{lemma:sumvarbd} yields}
\[
\textcolor{black}{\begin{aligned}
R_T\leq{}&2\beta_T(\delta/2)\sqrt{4(T+2)\gamma_T}\\
&+2B\left(\sum_{t=1}^{T}(1-p_t)+\sqrt{\frac{T}{2}\ln\frac{2}{\delta}}\right).
\end{aligned}}
\]
\textcolor{black}{Therefore, if $\sum_{t=1}^{T}(1-p_t)=\mathcal O(\sqrt T)$, then $R_T=\mathcal O\big(B\sqrt T+\beta_T(\delta/2)\sqrt{T\gamma_T}\big)$, establishing Theorem~\ref{thm:noharm}(i).}

\subsection{Proof of regret bound for \texttt{LLINBO-Justify} in Theorem~\ref{thm:noharm}}

\label{sec:theorem2}
The process of selecting $x_t$ via \texttt{LLINBO-Justify} can be written as \[
x_t=\begin{cases}
    x_{\mathcal{GP},t}\ \text{if}\ \alpha_{\text{UCB}}(x_{\text{LLM},t},F_{t-1})\le\alpha_{\text{UCB}}(x_{\mathcal{GP},t},F_{t-1})-\psi_t\\
    x_{\text{LLM},t}\ \text{else}
\end{cases}.
\]

Note that whichever case holds, $x_t$ is a $\psi_t$-suboptimal maximizer of $\alpha_{\mathrm{UCB}}(\cdot,F_{t-1})$, i.e., $\alpha_{\mathrm{UCB}}(x_t,F_{t-1})\geq\alpha_{\mathrm{UCB}}(x_{\mathcal{GP},t},F_{t-1})-\psi_t$, since $\psi_t\geq 0$. \textcolor{black}{Also, for $\delta\in(0,1)$ and with $\beta_t=\beta_t(\delta)$ in both the confidence event and the acquisition function,} we can upper bound $r_t$ by
\begin{align*}
r_t&=f(x^*)-f(x_t)\\
&\leq \mu_{t-1}(x^*)+\beta_t\sigma_{t-1}(x^*)-f(x_t) \hfill\tag{by Lemma~\ref{lemma:UCBconcentrate}}\\
&\leq \mu_{t-1}(x_{\mathcal{GP},t})+\beta_t\sigma_{t-1}(x_{\mathcal{GP},t})-\big(\mu_{t-1}(x_t)-\beta_t\sigma_{t-1}(x_t)\big) \hfill\tag{by definition of $x_{\mathcal{GP},t}$ and Lemma~\ref{lemma:UCBconcentrate}}\\
&\leq \mu_{t-1}(x_t)+\beta_t\sigma_{t-1}(x_t)+\psi_t-\big(\mu_{t-1}(x_t)-\beta_t\sigma_{t-1}(x_t)\big) \hfill\tag{by the acceptance rule}\\
&= \psi_t+2\beta_t\sigma_{t-1}(x_t).
\end{align*}
Since $\beta_t$ is nondecreasing in $t$ and $\sum_{t=1}^T\psi_t=\mathcal{O}(\sqrt{T})$, by Lemma~\ref{lemma:sumvarbd}, with probability at least $1-\delta$, we have
\[
R_T=\sum_{t=1}^T r_t\leq\sum_{t=1}^T\psi_t+2\beta_T\sum_{t=1}^T\sigma_{t-1}(x_t)=\mathcal{O}(\sqrt{T})+\beta_T\mathcal{O}(\sqrt{T\gamma_T}).
\]
\subsection{Proof of regret bound for \texttt{LLINBO-Constrained} in Theorem~\ref{thm:noharm}}
\label{sec:theorem34}
We first introduce a lemma that includes some algebraic derivations, which will be useful for proving the subsequent results.

\begin{lemma}[Appendix C in \cite{chowdhury2017kernelized}]
\label{lemma:algebra}
For any vector $\epsilon$ and let $\hat{\lambda}=1+2/T$, the following holds algebraically
\begin{align*}
    \left|k_{t-1}(x)^\top(K_{t-1} + \hat{\lambda} I)^{-1} \epsilon \right|
    &\leq \hat{\lambda}^{-1/2} \sigma_{t-1}(x) \sqrt{\epsilon^\top K_{t-1}(K_{t-1} + \hat{\lambda} I)^{-1} \epsilon},\\
    \epsilon^\top K_{t-1}(K_{t-1} + \hat{\lambda} I)^{-1} \epsilon
    &\leq \epsilon^\top \left( (K_{t-1} + (\hat{\lambda}-1)I)^{-1} + I \right)^{-1} \epsilon,
\end{align*}
\end{lemma}
where $K_{t-1}$ denotes the Gram matrix at time $t$, defined identically as in the main paper but indexed with a subscript to emphasize its dependence on the dataset available up to time $t-1$.
The same inequalities hold for any finite input set; below we also apply them to the augmented input set that includes $x_{\text{LLM},t}$.
Next, for each retained index $s\in I_t$, we denote by $\mathcal{M}_{t,s}$ the fantasy model, i.e., the $\mathcal{GP}$ constructed from $\mathcal{D}_{t-1}\cup \{(x_{\text{LLM},t},\tilde{f}_{t-1,s}(x_{\text{LLM},t}))\}$, in which $\tilde{f}_{t-1,s}(x_{\text{LLM},t})$ is treated as a noisy pseudo-observation with the same regularization $\hat{\lambda}$ as the observed data (see Sec.~\ref{sec:A3}). These fantasy models approximate, but do not exactly reproduce, the constrained posterior $F^+_{t-1}$ in~(\ref{eq:constr}), and the analysis below does not rely on any distributional identity between the two; instead, we bound each fantasy model directly. Recall from Lemma~\ref{lemma:UCBconcentrate} that for the $\mathcal{GP}$ constructed using \( \mathcal{D}_{t-1} \), previously denoted by \( F_{t-1} \), the difference between the posterior mean \( \mu_{t-1}(x) \) and the ground truth function \( f(x) \) can be bounded with a suitable \( \beta_t \). However, this bound does not directly apply to the fantasy models $\mathcal{M}_{t,s}$, as each is constructed using both historical data and imagined data \( (x_{\text{LLM},t}, \tilde{f}_{t-1,s}(x_{\text{LLM},t})) \).

Throughout this subsection, we work on the intersection of two simultaneous events. Let $\mathcal{E}_{\mathrm{GP}}$ denote the event on which the self-normalized bound on the observation noise in Theorem~1 of \cite{chowdhury2017kernelized} holds for all $t\geq 1$ at confidence level $\delta/2$. This event has probability at least $1-\delta/2$, and on it the conclusion of Lemma~\ref{lemma:UCBconcentrate}, with $\delta$ replaced by $\delta/2$, holds for all $t\geq 1$ and all $x\in\mathcal{X}$. Let $\mathcal{E}_{\mathrm{FS}}$ denote the event that
\[
|\tilde{f}_{t-1,s}(x_{\text{LLM},t})-\mu_{t-1}(x_{\text{LLM},t})|\leq\sqrt{2\ln(2S_{\max}T/\delta)}\,\sigma_{t-1}(x_{\text{LLM},t}),\quad \forall t\in[T],\ \forall s\in[S_t].
\]
Since $x_{\text{LLM},t}\in\mathcal{X}$ is chosen before the round-$t$ fantasy draws are generated, conditionally on the history and on $x_{\text{LLM},t}$, each $\tilde{f}_{t-1,s}(x_{\text{LLM},t})$ is normally distributed with mean $\mu_{t-1}(x_{\text{LLM},t})$ and variance $\sigma_{t-1}^2(x_{\text{LLM},t})$. \textcolor{black}{By the standard two-sided Gaussian tail inequality (see Lemma~5.1 of \citet{srinivas2009gaussian}),} $\mathbb{P}(|\tilde{f}_{t-1,s}(x_{\text{LLM},t})-\mu_{t-1}(x_{\text{LLM},t})|>a\,\sigma_{t-1}(x_{\text{LLM},t}))\leq e^{-a^2/2}$, applied with $a=\sqrt{2\ln(2S_{\max}T/\delta)}$, together with a union bound over the at most $S_{\max}T$ fantasy draws, $\mathcal{E}_{\mathrm{FS}}$ has probability at least $1-\delta/2$. Hence $\mathcal{E}\triangleq\mathcal{E}_{\mathrm{GP}}\cap\mathcal{E}_{\mathrm{FS}}$ has probability at least $1-\delta$. The following Theorem provides a bound on the difference between each fantasy mean and $f$ using a newly constructed \( \tilde{\beta}_t \).
\begin{theorem}
\label{thm:theorem3mp}
Under Assumption~\ref{assumption:rkhs_noise}, fix $\delta \in (0, 1)$ and $T \in \mathbb{N}$. Then, on the event $\mathcal{E}$, which has probability at least $1-\delta$, for all $t\in[T]$, all sample indices \( s\in I_t \) and all $x\in\mathcal{X}$, we have:
\[
| \mu^+_{t-1,s}(x) - f(x) |\leq\tilde{\beta}_t\sigma_{t-1}^+(x),
\]
where $\tilde{\beta}_t = 2B + 2R\sqrt{2( \gamma_t + 1 + \ln(2/\delta))}+\sqrt{2\ln(2S_{\max}T/\delta)}.$
\end{theorem}
\begin{proof}
Fix $t\in[T]$ and $s\in I_t$, and work on $\mathcal{E}$. As $s$ is fixed and we focus on deriving the difference between $\mu_{t-1,s}^+(x)$ and $f(x)$, we drop the subscript $s$ for simplicity. Let \( k_{t-1}^+ \) and \( K_{t-1}^+ \) denote the kernel vector and Gram matrix, respectively, defined as in Section~\ref{sec:pre}, except with the input set augmented to include \( x_{\text{LLM},t} \); that is, the input consists of the union of the previously observed designs \( x_1, \dots, x_{t-1} \) and the LLM-suggested point \( x_{\text{LLM},t} \). Let $\tilde{\delta}=f(x_{\text{LLM},t})-\tilde{f}_{t-1}(x_{\text{LLM},t})$, one can express the term $|\mu_{t-1}^+(x)-f(x)|$ as
\begin{align*}
| \mu^+_{t-1}(x) - f(x) | 
&\leq | f(x) - k^+_{t-1}(x)^\top \left( K^+_{t-1} + \hat{\lambda} I \right)^{-1} [f(x_1),...,f(x_{t-1}),f(x_{\text{LLM},t})]^\top|\\
&\ \ \ \ \ + | k^+_{t-1}(x)^\top ( K^+_{t-1} + \hat{\lambda} I )^{-1} [\delta_1,...,\delta_{t-1},\tilde{\delta}]^\top|\hfill\tag{by (\ref{eq:lemma1-1})}\\
&\underbrace{\leq | f(x) - k^+_{t-1}(x)^\top \left( K^+_{t-1} + \hat{\lambda} I \right)^{-1} [f(x_1),...,f(x_{t-1}),f(x_{\text{LLM},t})]^\top|}_{A}\\
&\ \ \ \ \ + \underbrace{| k^+_{t-1}(x)^\top ( K^+_{t-1} + \hat{\lambda} I )^{-1} [\delta_1,...,\delta_{t-1},0]^\top|}_{B}\\
&\ \ \ \ \ + \underbrace{| k^+_{t-1}(x)^\top ( K^+_{t-1} + \hat{\lambda} I )^{-1} [0,...,0,\tilde{\delta}]^\top|}_{C}\hfill\tag{by triangle inequality}.
\end{align*}
Note that, on $\mathcal{E}_{\mathrm{GP}}$, the sum $A+B$ is bounded by $(B+R\sqrt{2( \gamma_t + 1 + \ln(2/\delta))})\sigma_{t-1}^+(x)$ according to (\ref{eq:lemma1-2}). \textcolor{black}{For term $A$, the deterministic RKHS bias bound $|f(x)-k_X(x)^\top(K_X+\hat\lambda I)^{-1}f_X|\leq B\sigma_X(x)$ applies to any finite input set $X$, including the augmented inputs $\{x_1,\ldots,x_{t-1},x_{\text{LLM},t}\}$.} Here (\ref{eq:lemma1-2}) still applies to term B with the augmented input set, since the last entry of $\epsilon=[\delta_1,...,\delta_{t-1},0]^\top$ is zero: writing $\epsilon=[\epsilon'^\top,0]^\top$, the Schur complement gives $\epsilon^\top K^+_{t-1}(K^+_{t-1}+\hat{\lambda} I)^{-1}\epsilon\leq\epsilon'^\top K_{t-1}(K_{t-1}+\hat{\lambda} I)^{-1}\epsilon'$. \textcolor{black}{Indeed, write $K^+_{t-1}+\hat\lambda I=\begin{psmallmatrix}A&b\\b^\top&d\end{psmallmatrix}$ with $A=K_{t-1}+\hat\lambda I$ and $d>0$. The top-left block of $(K^+_{t-1}+\hat\lambda I)^{-1}$ is $(A-bd^{-1}b^\top)^{-1}\succeq A^{-1}$. Since $K(K+\hat\lambda I)^{-1}=I-\hat\lambda(K+\hat\lambda I)^{-1}$, the displayed inequality follows for $\epsilon=(\epsilon'^\top,0)^\top$.} Based on Lemma~\ref{lemma:algebra}, we can further bound the term $C$ as \begin{align*}
\left| {k^{+}_{t-1}(x)}^\top (K^{+}_{t-1} + \hat{\lambda} I)^{-1} [0,...,0,\tilde{\delta}]^\top \right| &\leq \hat{\lambda}^{-1/2} \, \sigma^{+}_{t-1}(x) \, \sqrt{ \begin{bmatrix} 0 & \tilde{\delta} \end{bmatrix} K^{+}_{t-1} (K^{+}_{t-1} + \hat{\lambda} I)^{-1} \begin{bmatrix} 0 & \tilde{\delta} \end{bmatrix}^\top }.
\end{align*}
On $\mathcal{E}$ and by Lemma~\ref{lemma:algebra}, the square root part of the above equation can be further simplified as\begin{align*}    
&\sqrt{ \begin{bmatrix} 0 & \tilde{\delta} \end{bmatrix} K^{+}_{t-1} (K^{+}_{t-1} + \hat{\lambda} I)^{-1} \begin{bmatrix} 0 & \tilde{\delta} \end{bmatrix}^\top}\\
&\leq \sqrt{ \begin{bmatrix} 0 & \tilde{\delta} \end{bmatrix} \left((K^{+}_{t-1} + (\hat{\lambda}-1) I)^{-1}+I\right)^{-1} \begin{bmatrix} 0 & \tilde{\delta} \end{bmatrix}^\top}\\
&\leq \left\|\begin{bmatrix} 0 & \tilde{\delta} \end{bmatrix}\right\|_2\\
&\leq |f(x_{\text{LLM},t})-\tilde{f}_{t-1}(x_{\text{LLM},t})|\\
&\leq |f(x_{\text{LLM},t})-\mu_{t-1}(x_{\text{LLM},t})|+|\mu_{t-1}(x_{\text{LLM},t})-\tilde{f}_{t-1}(x_{\text{LLM},t})|\\
&\leq (B + R\sqrt{2( \gamma_t + 1 + \ln(2/\delta))})\sigma_{t-1}(x_{\text{LLM},t})\\
&\ \ \ +\sqrt{2\ln(2S_{\max}T/\delta)}\sigma_{t-1}(x_{\text{LLM},t})\hfill\tag{on $\mathcal{E}$}.
\end{align*}
Here, the first term follows from Lemma~\ref{lemma:UCBconcentrate} at $x_{\text{LLM},t}\in\mathcal{X}$ on $\mathcal{E}_{\mathrm{GP}}$, and the second follows from the definition of $\mathcal{E}_{\mathrm{FS}}$, which controls all $S_t$ fantasy draws and hence all retained samples $s\in I_t$. As a result, since $\hat{\lambda}\geq 1$ and $\sigma_{t-1}(x_{\text{LLM},t})\leq 1$, on $\mathcal{E}$ term $C$ is bounded by $(B + R\sqrt{2( \gamma_t + 1 + \ln(2/\delta))}+\sqrt{2\ln(2S_{\max}T/\delta)})\sigma_{t-1}^+(x)$. Finally, by combining with the bound on $A+B$, on $\mathcal{E}$ we have\begin{align*}
| \mu^+_{t-1}(x) - f(x) |&\leq(2B + 2R\sqrt{2( \gamma_t + 1 + \ln(2/\delta))}+\sqrt{2\ln(2S_{\max}T/\delta)})\sigma_{t-1}^+(x)\\
&=\tilde{\beta}_t\sigma_{t-1}^+(x),
\end{align*}
where $\tilde{\beta}_t=2B + 2R\sqrt{2( \gamma_t + 1 + \ln(2/\delta))}+\sqrt{2\ln(2S_{\max}T/\delta)}$.
\end{proof}

\begin{lemma}
\label{lemma:varbd}
For a set of $S \geq 2$ samples $X_1, \dots, X_S$, if $|X_s| \leq c$, $\forall s \in [S]$, then the sample variance satisfies:
\[
\varsigma = \frac{1}{S-1} \sum_{s=1}^S (X_s - \bar{X})^2\leq 2c^2.
\]  

\end{lemma}
\begin{proof}
Let $\bar{X}$ be the sample mean as  
$\bar{X} = \frac{1}{S} \sum_{s=1}^S X_s$.  
This proof follows the definition of sample variance  
\[
\varsigma = \frac{1}{S-1} \sum_{s=1}^S (X_s - \bar{X})^2 = \frac{1}{S-1} \left(\sum_{s=1}^S X_s^2 - S\bar{X}^2\right) \leq \frac{1}{S-1} \sum_{s=1}^S X_s^2 \leq \frac{S}{S-1}c^2\leq 2c^2.
\]  
\end{proof}

Now we are ready to derive the upper bound for the cumulative regret. Note that $x_t$ is selected as the maximizer of the $\mathcal{CGP}$-UCB, which means \[
\bar{\mu}^+_{t-1}(x_t)+\tilde{\beta}_t\sqrt{\sigma_{t-1}^+(x_t)^2+s_{t-1}^2(x_t)}\geq \bar{\mu}^+_{t-1}(x)+\tilde{\beta}_t\sqrt{\sigma_{t-1}^+(x)^2+s_{t-1}^2(x)}\ \forall x\in\mathcal{X}.
\]
We first deal with the error caused by $s^2_{t-1}(x)$, which is the sample variance of the predicted mean at $x$, or namely, $k_{t-1}^+(x)^\top(K_{t-1}^++\hat{\lambda}I)^{-1}(y_1,...,y_{t-1},\tilde{f}_{t-1,s}(x_{\text{LLM},t}))^\top\ \forall s\in I_t$. Note that there is no uncertainty in $k_{t-1}^+(x)^\top(K_{t-1}^++\hat{\lambda}I)^{-1}$ and also $(y_1,...,y_{t-1})$, hence we can subtract it and simply consider the variance of \[k_{t-1}^+(x)^\top(K_{t-1}^++\hat{\lambda}I)^{-1}\begin{bmatrix}
    0 & \tilde{f}_{t-1,s}(x_{\text{LLM},t})
\end{bmatrix}^\top\ \forall s\in I_t.\]
Since the sample variance does not change when the same constant is subtracted from all samples, this is also the sample variance of $k_{t-1}^+(x)^\top(K_{t-1}^++\hat{\lambda}I)^{-1}\begin{bmatrix}
    0 & \tilde{f}_{t-1,s}(x_{\text{LLM},t})-\mu_{t-1}(x_{\text{LLM},t})
\end{bmatrix}^\top\ \forall s\in I_t$. In order to apply Lemma~\ref{lemma:varbd}, we first derive the upper bound for $k_{t-1}^+(x)^\top(K_{t-1}^++\hat{\lambda}I)^{-1}\begin{bmatrix}
    0 & \tilde{f}_{t-1,s}(x_{\text{LLM},t})-\mu_{t-1}(x_{\text{LLM},t})
\end{bmatrix}^\top\ \forall s\in I_t$. On $\mathcal{E}_{\mathrm{FS}}$ and by Lemma~\ref{lemma:algebra}, we have\begin{align*}
&\left|k_{t-1}^+(x)^\top(K_{t-1}^++\hat{\lambda}I)^{-1}\begin{bmatrix}
    0 & \tilde{f}_{t-1,s}(x_{\text{LLM},t})-\mu_{t-1}(x_{\text{LLM},t})
\end{bmatrix}^\top\right|\\
&\leq \hat{\lambda}^{-1/2} \sigma_{t-1}^{+}(x) \sqrt{ [0, \tilde{f}_{t-1,s}(x_{\text{LLM},t}) - \mu_{t-1}(x_{\text{LLM},t})] K_{t-1}^{+}(K_{t-1}^{+} + \hat{\lambda} I)^{-1} [0, \tilde{f}_{t-1,s}(x_{\text{LLM},t}) - \mu_{t-1}(x_{\text{LLM},t})]^\top }\\
&\leq \hat{\lambda}^{-1/2} \sigma_{t-1}^{+}(x) \sqrt{ (\tilde{f}_{t-1,s}(x_{\text{LLM},t})-\mu_{t-1}(x_{\text{LLM},t}))^2}\\
&= \hat{\lambda}^{-1/2} \sigma_{t-1}^{+}(x) | \tilde{f}_{t-1,s}(x_{\text{LLM},t})-\mu_{t-1}(x_{\text{LLM},t})|\\
&\leq \sigma_{t-1}^{+}(x) \sqrt{2\ln(2S_{\max}T/\delta)},
\end{align*}
where the last inequality uses the fact that $\hat{\lambda}\geq1$, $\sigma_{t-1}(x_{\text{LLM},t})\leq 1$, and the definition of $\mathcal{E}_{\mathrm{FS}}$. In this case, by Lemma~\ref{lemma:varbd} with $c=\sigma_{t-1}^{+}(x)\sqrt{2\ln(2S_{\max}T/\delta)}$, the variance of $k_{t-1}^+(x)^\top(K_{t-1}^++\hat{\lambda}I)^{-1}\begin{bmatrix}
    0 & \tilde{f}_{t-1,s}(x_{\text{LLM},t})
\end{bmatrix}^\top\ \forall s\in I_t$ can be bounded as follows when $|I_t|\geq 2$ (when $|I_t|=1$, $s_{t-1}\equiv 0$ by definition)\begin{equation}
    s^2_{t-1}(x)\leq 4\sigma_{t-1}^{+}(x)^2\ln(2S_{\max}T/\delta).
\label{eq:samplevarbd}
\end{equation}
Note that by Theorem~\ref{thm:theorem3mp}, on $\mathcal{E}$ the ground truth $f(x)$ can be bounded by ${\mu}^+_{t-1,s}(x)\pm \tilde{\beta}_t\sigma_{t-1}^+(x)$ for all $t\in[T]$, all indices $s$ in $I_t$ and all $x\in\mathcal{X}$. Since all models $\mathcal{M}_{t,s}$ share the same inputs $x_1,...,x_{t-1},x_{\text{LLM},t}$, $\sigma_{t-1}^+(x)$ does not depend on $s$, and this also holds for the mean over all $s\in I_t$, that is, \[
\bar{\mu}^+_{t-1}(x)- \tilde{\beta}_t\sigma_{t-1}^+(x)\leq f(x)\leq \bar{\mu}^+_{t-1}(x)+ \tilde{\beta}_t\sigma_{t-1}^+(x).\] \textcolor{black}{When $I_t=\emptyset$, the procedure falls back to $F_{t-1}$ and selects $x_t=x_{\mathcal{GP},t}$ by maximizing GP-UCB with coefficient $\beta_t(\delta/2)$, as specified in Theorem~\ref{thm:noharm}(iii). Therefore, on $\mathcal E_{\mathrm{GP}}$, Lemma~\ref{lemma:UCBregretbd} with confidence level $\delta/2$ gives $r_t\leq 2\beta_t(\delta/2)\sigma_{t-1}(x_t)\leq 2\tilde\beta_t\sigma_{t-1}(x_t)$, which has the same form as the bound below upon taking $\sigma^+_{t-1}=\sigma_{t-1}$ and $s_{t-1}\equiv 0$.} On $\mathcal{E}$, we can derive the upper bound for $r_t=f(x^*)-f(x_t)$ for all $t$ with $I_t\neq\emptyset$ as
\begin{align*}
r_{t} &= f(x^*) - f(x_t) \\
&\leq \bar{\mu}_{t-1}^+({x}^*) + \tilde{\beta}_t \sigma_{t-1}^+(x^*) - \left(\bar{\mu}_{t-1}^+(x_{t}) - \tilde{\beta}_t\sigma_{t-1}^+(x_{t})\right)\\
&=\left(\bar{\mu}_{t-1}^+({x}^*)-\bar{\mu}_{t-1}^+(x_{t})\right) + \tilde{\beta}_t \sigma_{t-1}^+(x^*) + \tilde{\beta}_t\sigma_{t-1}^+(x_{t})\\
&\leq \tilde{\beta}_{t} \sqrt{\sigma_{t-1}^+(x_{t})^2 + s^2_{t-1}(x_{t})} 
- \tilde{\beta}_{t} \sqrt{\sigma_{t-1}^+(x^*)^2 + s^2_{t-1}(x^*)}+\tilde{\beta}_t \sigma_{t-1}^+(x^*) +\tilde{\beta}_t\sigma_{t-1}^+(x_t) \\
&\leq \tilde{\beta}_{t}\sigma_{t-1}^+(x_t) + \tilde{\beta}_{t}s_{t-1}(x_t)-\tilde{\beta}_{t}\sigma_{t-1}^+(x^*)+\tilde{\beta}_t \sigma_{t-1}^+(x^*) +\tilde{\beta}_t\sigma_{t-1}^+(x_t) \\ 
&=2\tilde{\beta}_{t}\sigma_{t-1}^+(x_t) + \tilde{\beta}_{t}s_{t-1}(x_t) \\
&\leq 2\tilde{\beta}_{t}\sigma_{t-1}^+(x_t) + 2\tilde{\beta}_{t}\sqrt{\ln(2S_{\max}T/\delta)}\,\sigma_{t-1}^+(x_t)\hfill\tag{by (\ref{eq:samplevarbd})} \\
&\leq \mathcal{O}\left(\sqrt{\ln(S_{\max}T)}\left(\sqrt{\gamma_T}+\sqrt{\ln(S_{\max}T)}\right)\right) \sigma_{t-1}^+(x_{t}),
\end{align*}
where the last inequality uses $\gamma_t\leq\gamma_T$, and thus $\tilde{\beta}_t=\mathcal{O}(\sqrt{\gamma_T}+\sqrt{\ln(S_{\max}T)})$.

The cumulative regret can be bounded as \begin{align*}
    R_T=\sum_{t=1}^T r_t&\leq\sum_{t=1}^T \mathcal{O}\left(\sqrt{\ln(S_{\max}T)}\left(\sqrt{\gamma_T}+\sqrt{\ln(S_{\max}T)}\right)\right) \sigma_{t-1}^+(x_{t})\\
    &\leq \mathcal{O}\left(\sqrt{\ln(S_{\max}T)}\left(\sqrt{\gamma_T}+\sqrt{\ln(S_{\max}T)}\right)\right)\sum_{t=1}^T\sigma_{t-1}(x_{t})\\
    &\leq \mathcal{O}\left(\sqrt{\ln(S_{\max}T)}\left(\sqrt{\gamma_T}+\sqrt{\ln(S_{\max}T)}\right)\right) \mathcal{O}(\sqrt{T\gamma_T})\hfill\tag{by Lemma~\ref{lemma:sumvarbd}}\\
    &=\mathcal{O}\left(\sqrt{T\gamma_T\ln(S_{\max}T)}\left(\sqrt{\gamma_T}+\sqrt{\ln(S_{\max}T)}\right)\right),
\end{align*}
where the second inequality uses $\sigma^+_{t-1}(x)\leq\sigma_{t-1}(x)$, since adding the input $x_{\text{LLM},t}$ can only reduce the posterior variance. Since $\mathcal{E}$ has probability at least $1-\delta$, this yields the result in Theorem~\ref{thm:noharm}(iii).

\section{Additional Simulation Results}

\subsection{Trade-off between complexity and performance in \texttt{LLINBO-Constrained}}
\label{app:constrained}
An experiment on BBO task using Rastrigin-2$D$ is performed to assess the trade-off between computational complexity and the regret in \texttt{LLINBO-Constrained}. More specifically, we aim to link the settings of $S_1$ (initial sample size) and the decreasing rate of $S_t$ to the performance. We consider $S_1=100,\ 1000, 5000$ and set $S_t=S_1, S_t=S_1/t,\ S_t=S_1/t^2$. All other settings are the same as in the main paper. Fig.~\ref{fig:bbfrc} shows the regret curves for each setting compared with the two baselines: BO and \texttt{LLAMBO-light}.

\begin{figure*}[ht]
\centering
    \begin{subfigure}[t]{0.32\textwidth}
        \includegraphics[width=\linewidth]{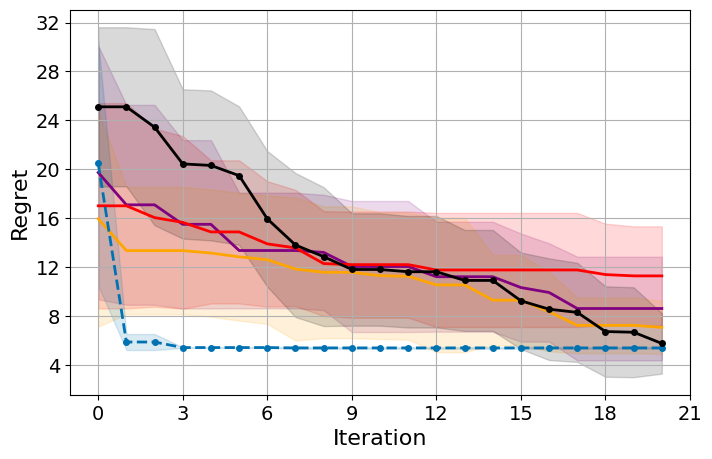}
         \caption{$S_t=S_1$}
    \end{subfigure}\hfill\begin{subfigure}[t]{0.32\textwidth}
        \includegraphics[width=\linewidth]{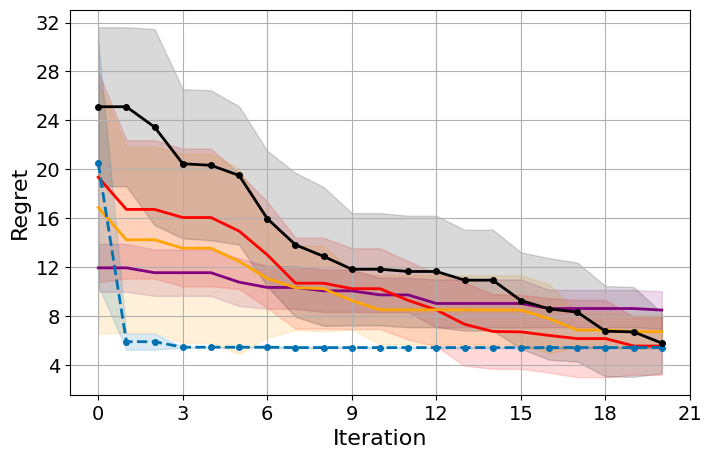}
        \caption{$S_t=S_1/t$}
    \end{subfigure}\hfill\begin{subfigure}[t]{0.32\textwidth}
\includegraphics[width=\linewidth]{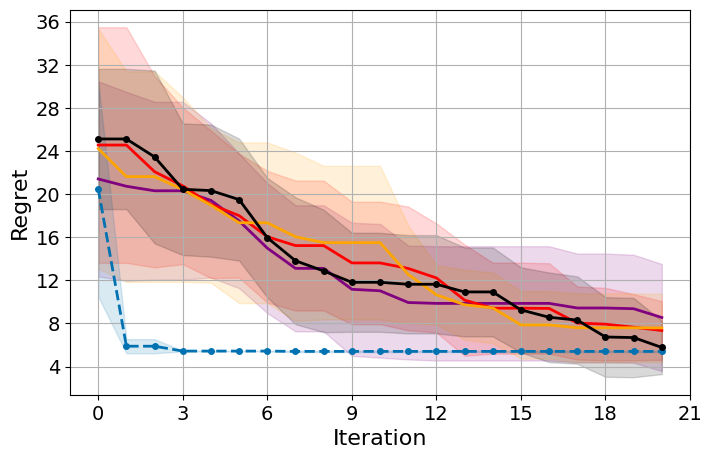}
        \caption{$S_t=S_1/t^2$}
    \end{subfigure}

\caption{$G_t$ comparison on Rastrigin-2$D$ using different settings of $S_1$ and $S_t$. Each line shows the mean regret, shaded with 95\% confidence intervals. \textbf{Proposed methods:}
\solidline{red}{1.5em} \texttt{LLINBO-Constrained} with $S_1=100$,
\solidline{purple}{1.5em} \texttt{LLINBO-Constrained} with $S_1=1000$,
\solidline{yellow}{1.5em} \texttt{LLINBO-Constrained} with $S_1=5000$.
\textbf{Baselines:}
\solidline{black}{1.5em} {BO},
\solidline{cobaltblue}{1.5em} \texttt{LLAMBO-light}.
}
\label{fig:bbfrc}
\end{figure*}

We can derive several important insights from Fig.~\ref{fig:bbfrc}. First, BO consistently reduces regret across all iterations, while the LLM-based agent (\texttt{LLAMBO-light}) is highly effective in the early phase but struggles to provide meaningful improvements thereafter. This highlights a fundamental limitation of LLM-guided exploration—strong initial performance followed by diminishing returns.

Second, the behavior of \texttt{LLINBO-Constrained} becomes increasingly similar to standard BO as the decay rate of $S_t$ increases. When $S_t$ is constant, the algorithm remains partially influenced by the LLM's suggestions, which can be suboptimal in later stages. However, setting $S_t = S_1 / t^2$ yields regret curves that closely align with BO, indicating that a faster decay reduces reliance on the LLM at later iterations, when its suggestions become less reliable. This finding supports the use of an aggressive decay schedule for $S_t$.

Third, varying $S_1$ reveals a useful trade-off between early-stage performance and long-term robustness. A larger initial $S_1$ enables the algorithm to leverage the LLM's few-shot learning strengths, as seen in the yellow and purple curves in Fig.~\ref{fig:bbfrc}(a) and Fig.~\ref{fig:bbfrc}(b). These observations are consistent with our design choices in Sec.~\ref{sec:num}, where we adopt $S_t = 10000 / t^2$. In practice, we recommend setting $S_1$ based on available computational resources and adopting a decay rate on the order of $\mathcal{O}(1/t^2)$.

Finally, while the computational–performance trade-off is difficult to quantify precisely due to the black-box nature of LLMs, the trend in the regret curves offers intuitive guidance. Since LLMs perform well in the early iterations but fail to exploit in the later phase, setting a large $S_1$ (to fully utilize initial LLM strength) and decreasing it over time (to prioritize exploitation and reduce computational cost) provides a balanced and practical strategy for \texttt{LLINBO-Constrained}.

\subsection{Experiments on \texttt{LLINBO} using \texttt{LLMP} Agent}
\label{app:LLMP}
In this section, we demonstrate the general usage of \texttt{LLINBO}. As an additional \textit{LLM-based BO} (the agent $\mathcal{A}$ in Algorithm~\ref{algo:LLMIBO}), we consider \texttt{LLMP} proposed by \cite{requeima2024llm}. \texttt{LLMP} was originally introduced as a procedure to construct an LLM-based surrogate model. Here, we extend \texttt{LLMP} into a full BO pipeline by coupling its predictive mean and uncertainty estimates with a standard AF to sequentially select new evaluation points.

At iteration $t$, we prompt the LLM with (i) the set of previously evaluated input--output pairs and (ii) a short textual description of the objective-function pattern (Table~\ref{tab:llmbbfofp}). For a set of candidate locations $\{x_j\}_{j=1}^{J}$, we query the LLM to produce predicted objective values. We repeat the query $M$ times and treat the empirical mean and variance of the $M$ responses at each $x_j$ as the predictive mean and uncertainty of the surrogate model. These estimates are then used within the AF to select the next evaluation point.

Unlike the experiments in other parts of the paper (where \texttt{LLAMBO-light} is the LLM agent and warmstarting is used to generate an initial dataset), we use random sampling to generate the initial dataset across all methods (Ours, \texttt{LLMP} and BO). Specifically, at each iteration we sample $J=50$ candidate points uniformly at random from the search domain, and obtain $M=10$ LLM predictions per candidate point. All other settings (hyperparameters, number of iterations, size of the initial data, LLM and AF) are the same as those detailed in Sec.~\ref{sec:num}.

From the results in Figure~\ref{fig:bbfollmpres}, several insights can be drawn. Because all methods use random sampling to generate the initial dataset, the LLM-based approaches (ours and \texttt{LLMP}) do not benefit from LLM-generated initialization as in Fig.~\ref{fig:bbfores}. Nevertheless, this setting does not diminish the purpose of our method, which is to make \texttt{LLMP} in the no-harm sense of Theorem~\ref{thm:noharm}. This is particularly evident on Ackley-6$D$: \texttt{LLMP} fails to converge to a near-optimal design, whereas all three variants of \texttt{LLINBO} eventually approach the optimum and achieve performance comparable to, or better than, BO. Another notable observation is that \texttt{LLINBO} not only safeguards against \texttt{LLMP}'s failures, but also leverages \texttt{LLMP}'s strengths when it is informative. In such cases, \texttt{LLINBO} can match or outperform BO, as illustrated on Levy-2$D$ and Branin-2$D$.

\begin{figure*}[h]
\centering
    \begin{subfigure}[t]{0.31\textwidth}
        \includegraphics[width=\linewidth]{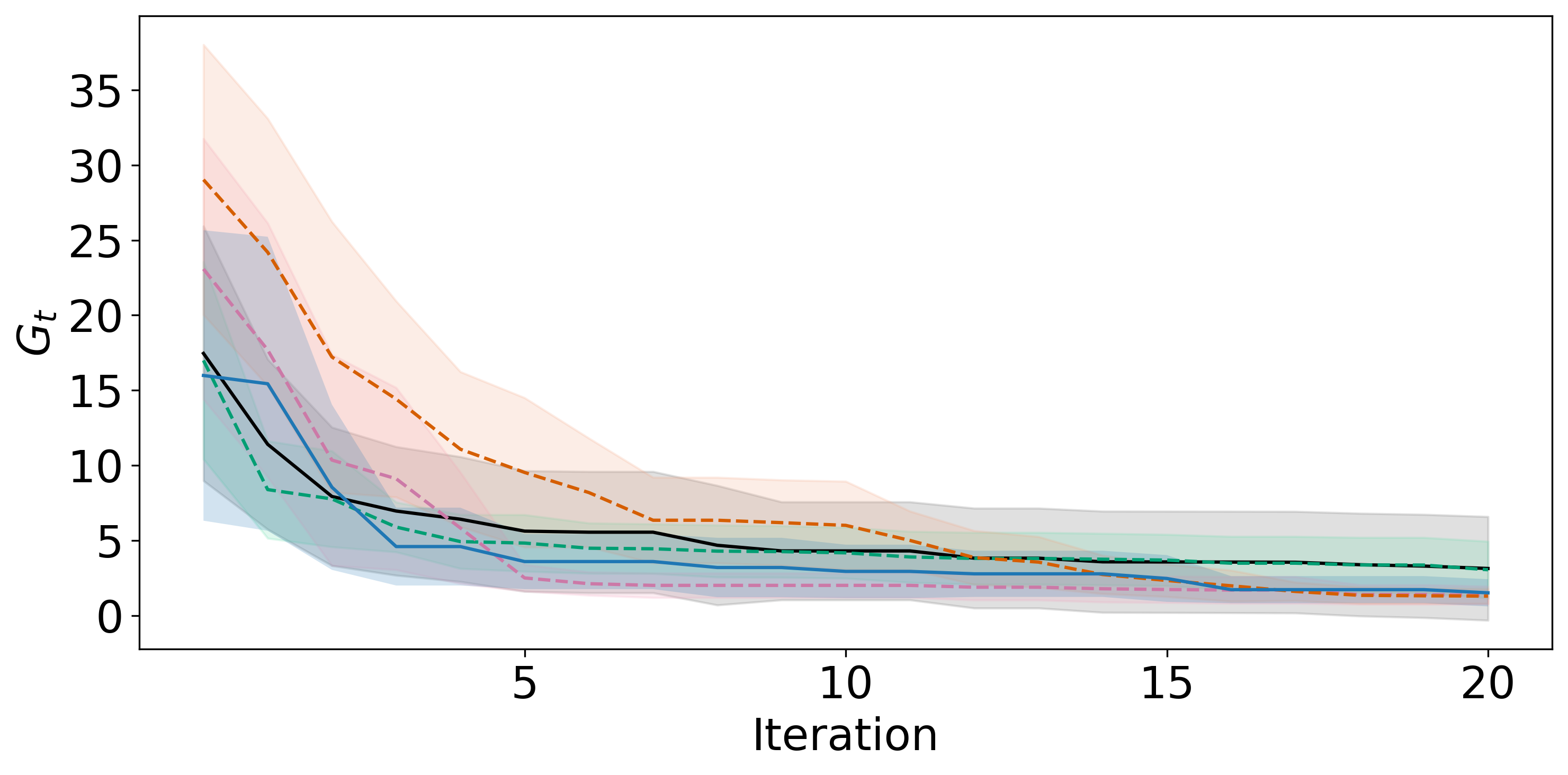}
        \caption{Branin-2$D$}
    \end{subfigure}
    \hfill
    \begin{subfigure}[t]{0.31\textwidth}
        \includegraphics[width=\linewidth]{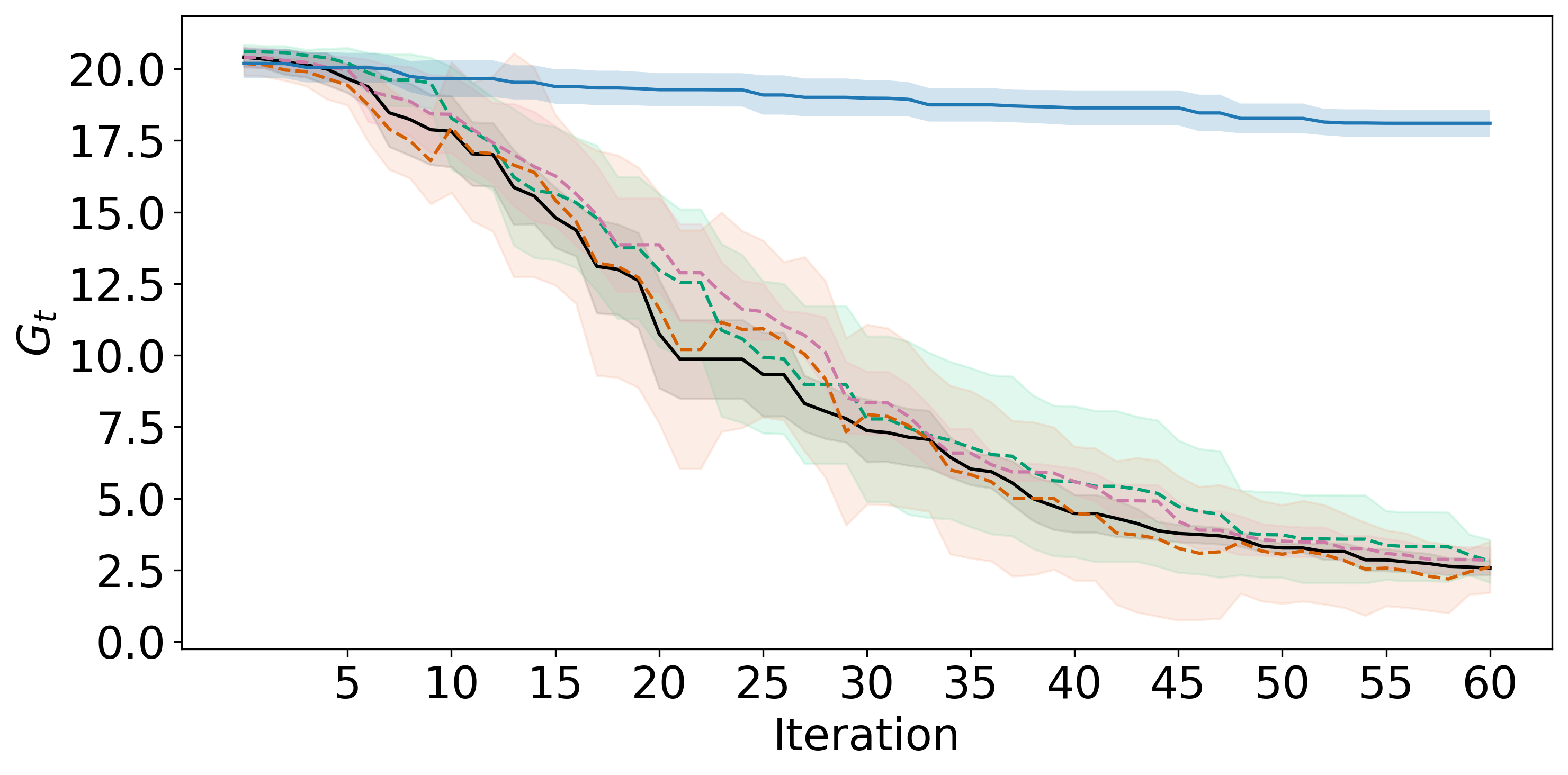}
        \caption{Ackley-6$D$}
    \end{subfigure}
    \hfill
    \begin{subfigure}[t]{0.31\textwidth}
        \includegraphics[width=\linewidth]{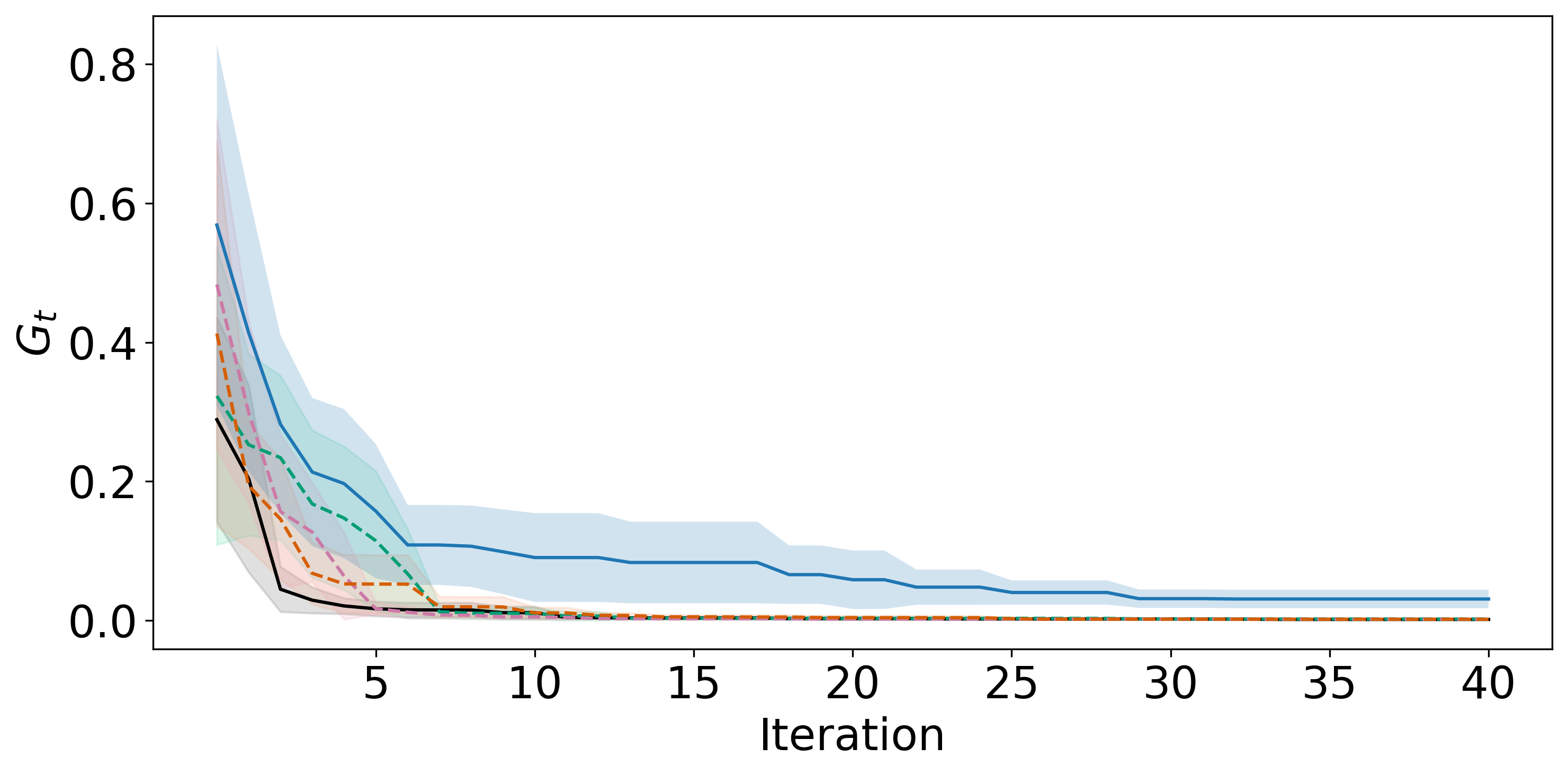}
        \caption{Hartmann-4$D$}
    \end{subfigure}

    \vspace{0.1em}

    \begin{subfigure}[t]{0.31\textwidth}
        \includegraphics[width=\linewidth]{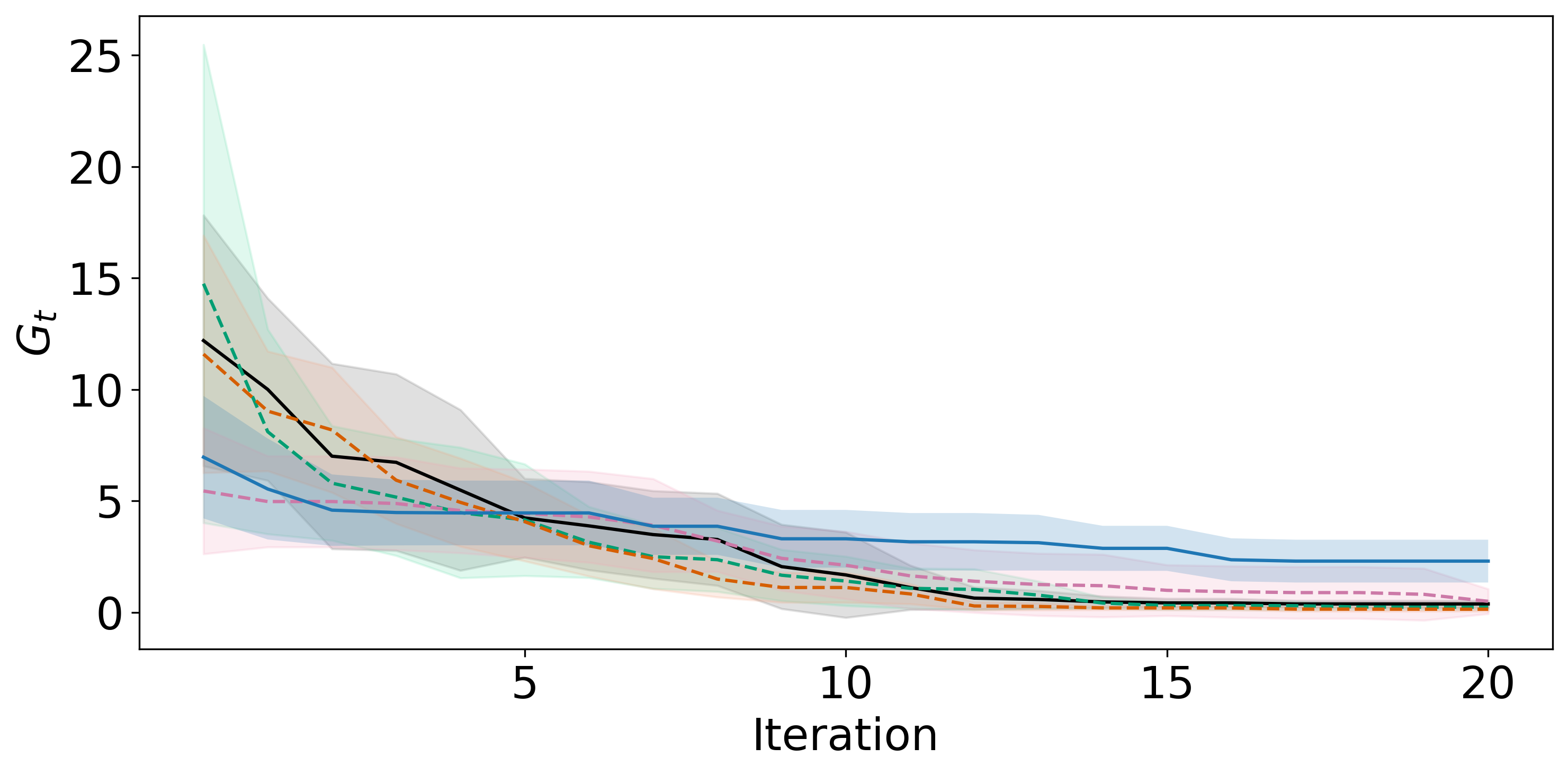}
        \caption{Levy-2$D$}
    \end{subfigure}
    \hfill
    \begin{subfigure}[t]{0.31\textwidth}
        \includegraphics[width=\linewidth]{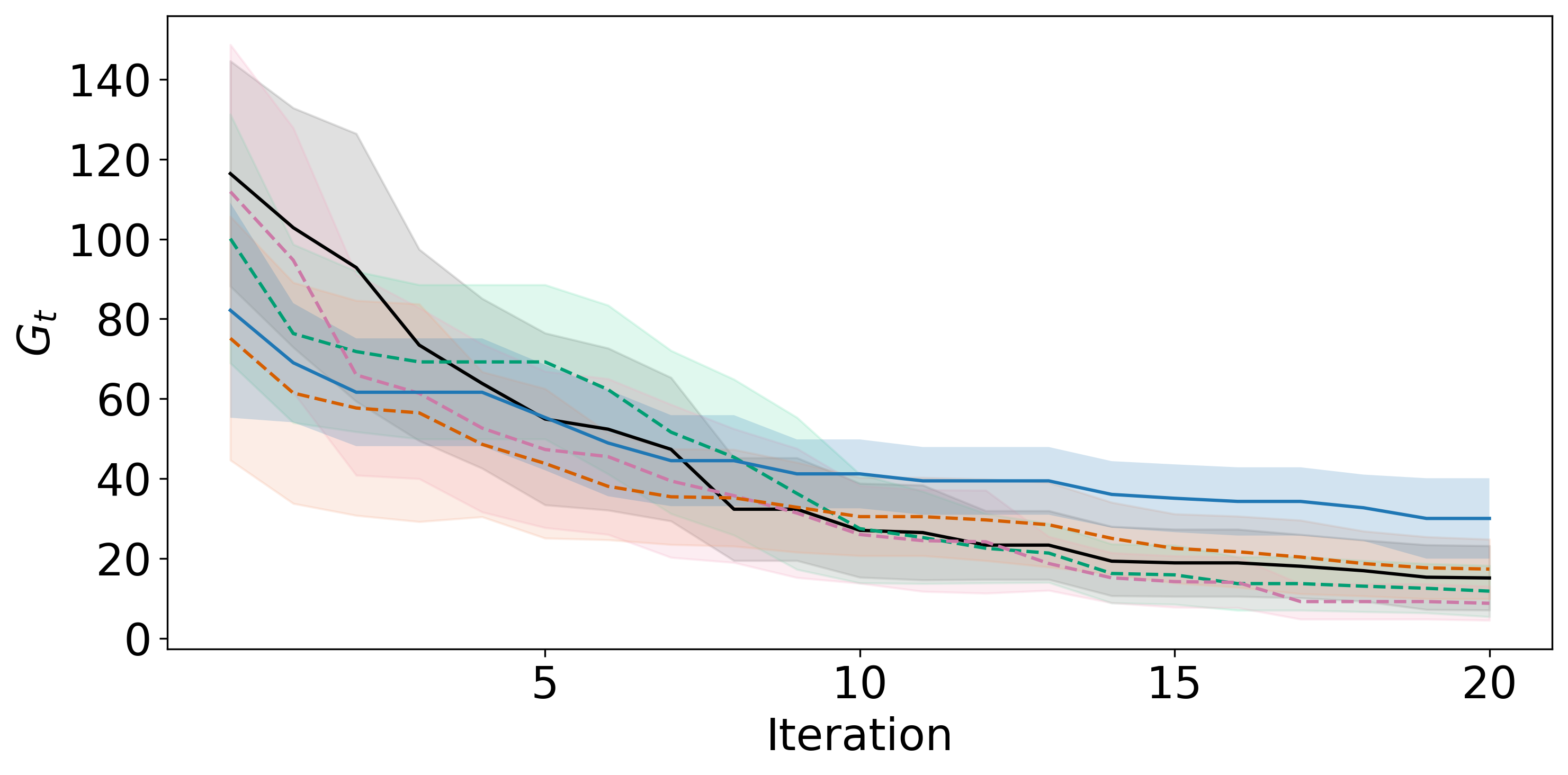}
        \caption{Bukin-2$D$}
    \end{subfigure}
    \hfill
    \begin{subfigure}[t]{0.31\textwidth}
        \includegraphics[width=\linewidth]{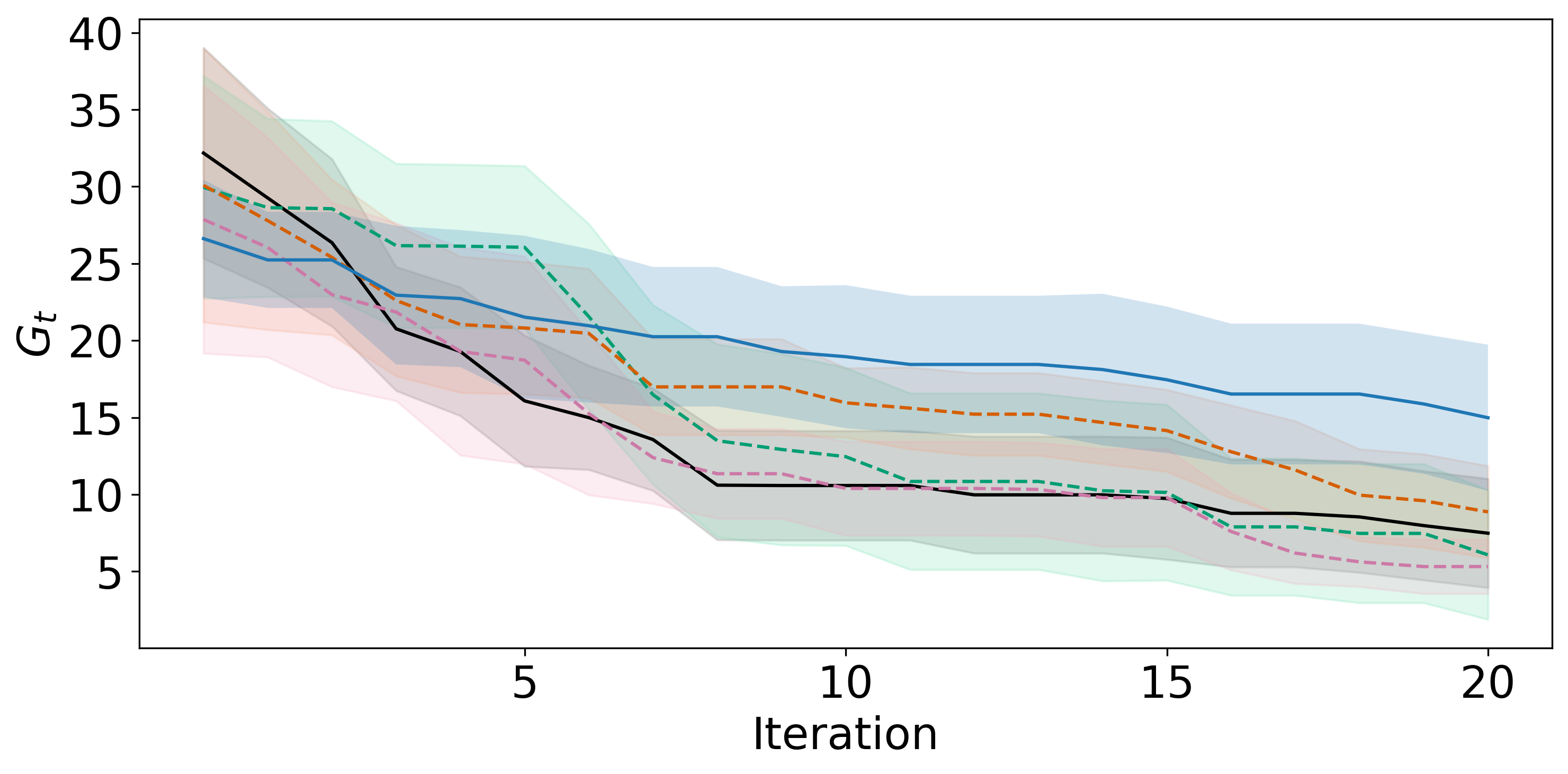}
        \caption{Rastrigin-2$D$}
    \end{subfigure}

\caption{
Regret comparison for BBO using \texttt{LLMP} as the LLM agent. Each line shows the mean regret, shaded with 95\% confidence intervals. \textbf{Proposed methods}:\dashedline{darkgreen}{1.5em} \texttt{LLINBO-Transient}, 
\dashedline{magenta}{1.5em} \texttt{LLINBO-Justify}, 
\dashedline{brickred}{1.5em} \texttt{LLINBO-Constrained}. \textbf{Baselines}:
\solidline{cobaltblue}{1.5em} \texttt{LLMP}, 
\solidline{black}{1.5em} BO.}
\label{fig:bbfollmpres}
\end{figure*}

\subsection{BBO experiments with a newer LLM model}\label{app:newmodel}
To further assess the effectiveness and generality of the proposed methods, we repeat the BBO experiments using a newer LLM, GPT-5 mini, as the LLM agent. These experiments also examine whether advances in the underlying LLM are sufficient to make standalone \textit{LLM-based BO} reliable. We consider three benchmark functions: Ackley-6$D$, Hartmann-6$D$, and Levy-8$D$. Due to the substantial computational and token costs of \texttt{LLAMBO}, we omit it from this experiment and retain \texttt{LLAMBO-light} and BO as baselines.

Figure~\ref{fig:newmodel} yields three main observations. First, replacing GPT-3.5-Turbo with GPT-5 mini substantially improves \texttt{LLAMBO-light}, transforming it from a frequently plateauing method into a competitive baseline comparable to BO. Nevertheless, its relative performance still varies across tasks, indicating that a stronger LLM alone does not fully resolve the reliability concerns of standalone \textit{LLM-based BO}. Second, \texttt{LLINBO-Justify} acts as an effective gatekeeper: it relies more heavily on the LLM when its recommendations are useful, as in Levy-8$D$, while reverting to GP-based decisions when the LLM guidance is less reliable, as in Ackley-6$D$. Third, \texttt{LLINBO-Constrained} consistently achieves the lowest regret across all three functions. Overall, these results demonstrate that the benefits of \texttt{LLINBO} persist with a newer and stronger LLM backbone and arise from the principled collaboration between the LLM and the statistical surrogate, rather than from the capability of a particular LLM alone.

\begin{figure*}[ht]
\centering
    \begin{subfigure}[t]{0.32\textwidth}
        \includegraphics[width=\linewidth]{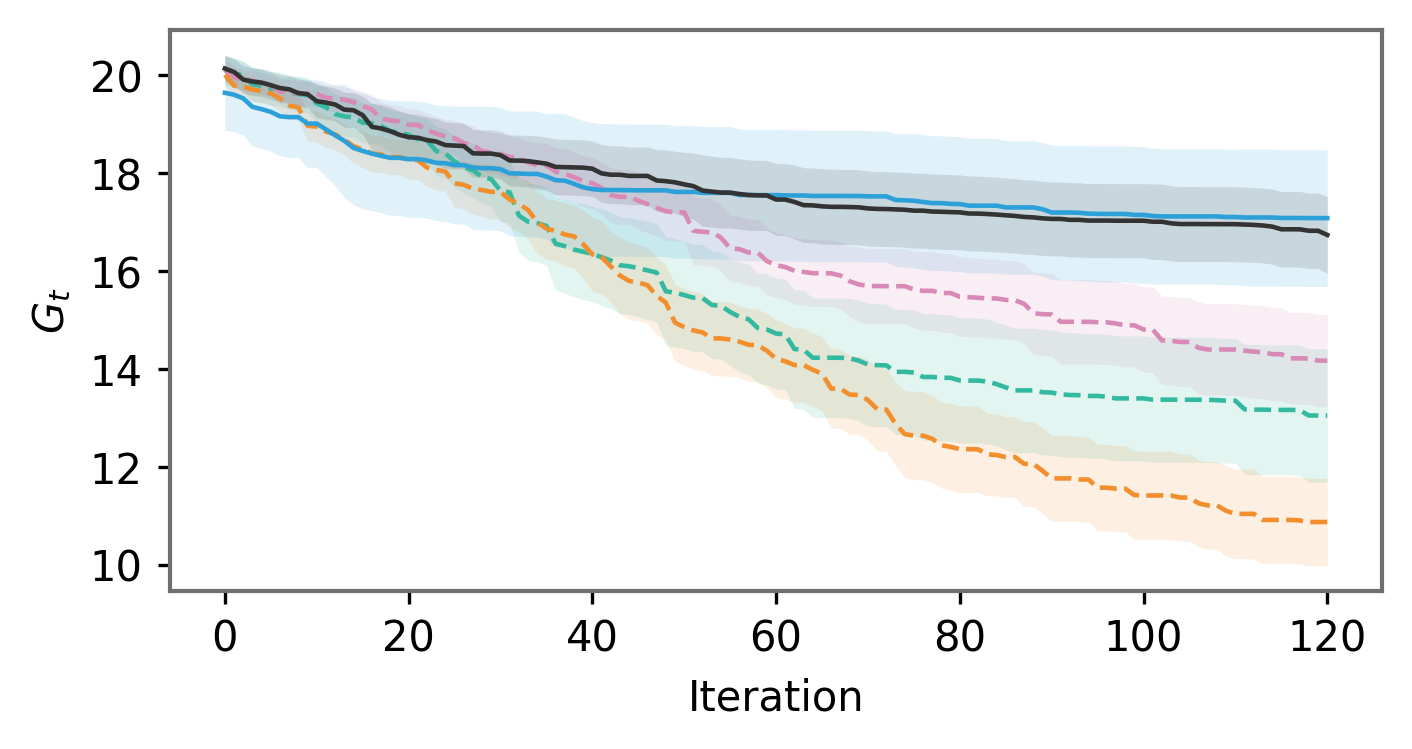}
        \caption{Ackley-6$D$}
    \end{subfigure}
    \hfill
    \begin{subfigure}[t]{0.32\textwidth}
        \includegraphics[width=\linewidth]{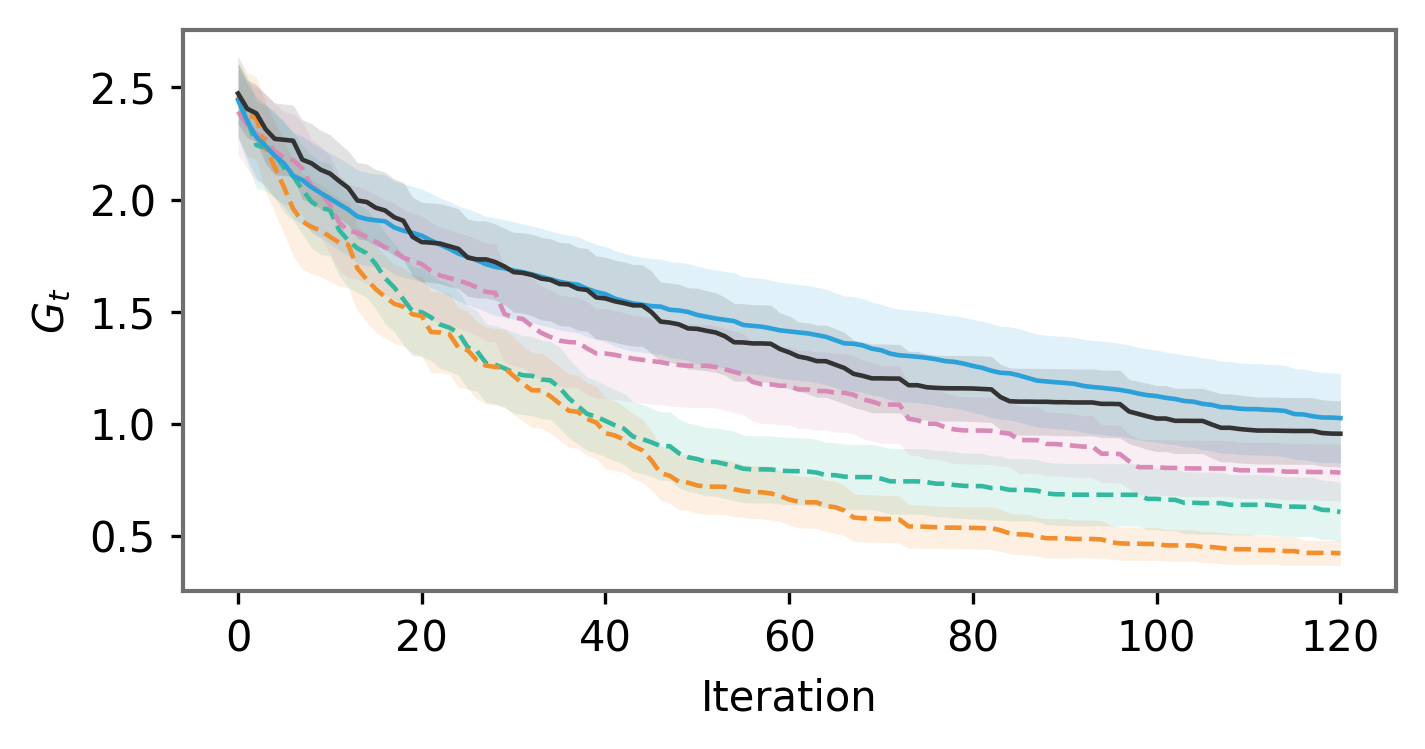}
        \caption{Hartmann-6$D$}
    \end{subfigure}
    \hfill
    \begin{subfigure}[t]{0.32\textwidth}
        \includegraphics[width=\linewidth]{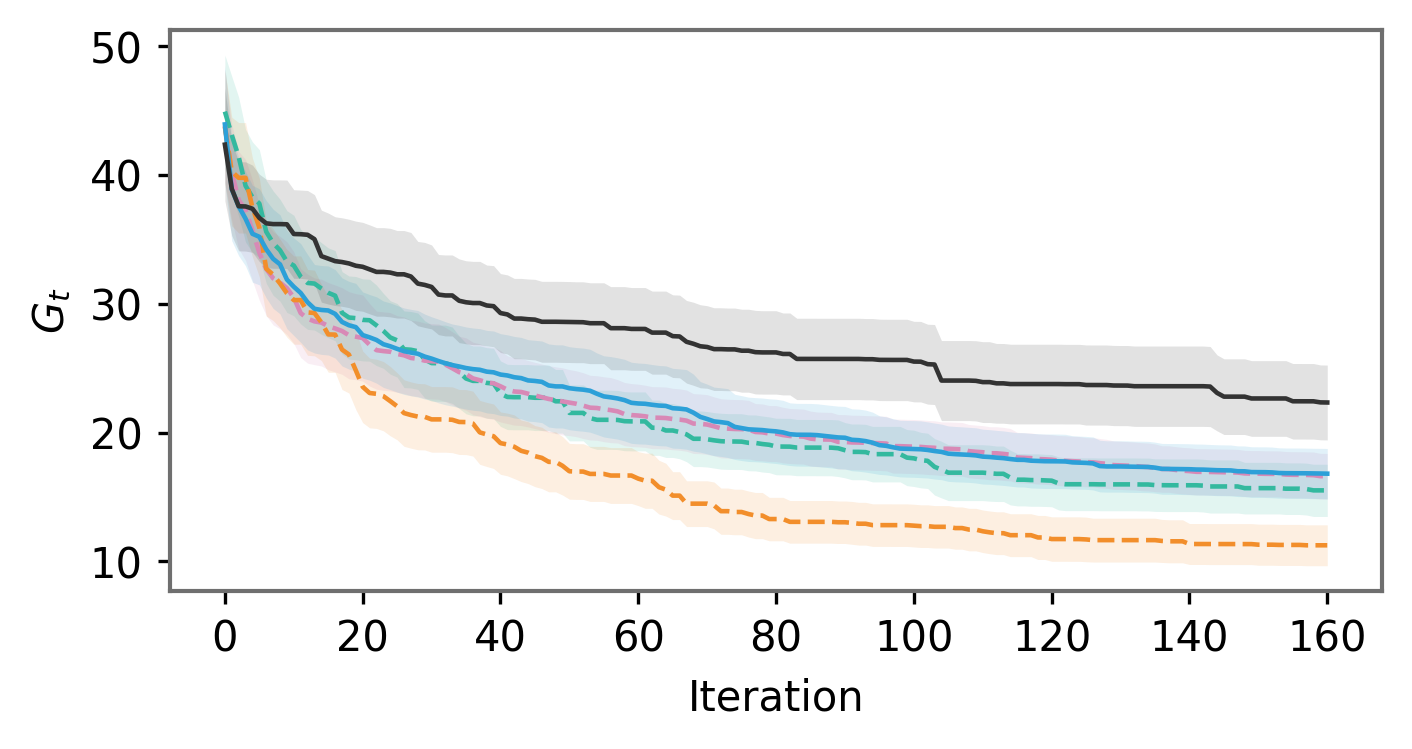}
        \caption{Levy-8$D$}
    \end{subfigure}

\caption{
$G_t$ comparison for BBO using newer LLM model. Each line shows the mean regret, shaded with 95\% confidence intervals.
\textbf{Proposed methods}: \dashedline{darkgreen}{1.5em} \texttt{LLINBO-Transient}, 
\dashedline{magenta}{1.5em} \texttt{LLINBO-Justify}, 
\dashedline{brickred}{1.5em} \texttt{LLINBO-Constrained}. \textbf{Baselines}:
\solidline{cobaltblue}{1.5em} \texttt{LLAMBO-light}, 
\solidline{black}{1.5em} BO.
}
\label{fig:newmodel}
\end{figure*}

\subsection{Comparison with LLM-initialized BO}\label{app:llmibo}
In Sec.~\ref{sec:num}, LLM-based methods, including \texttt{LLINBO}, generate the initial dataset $\mathcal{D}_0$ by warmstarting with the LLM, whereas BO uses random sampling. Part of the early advantage of \texttt{LLINBO} could therefore come from a better $\mathcal{D}_0$ rather than from its collaboration mechanisms. To isolate this effect, we add a stronger baseline, LLM-initialized BO (\texttt{LLMiBO}), which generates $\mathcal{D}_0$ with the same warmstarting prompt as the LLM-based methods and then runs GP-UCB for the remaining budget. \texttt{LLMiBO} and BO therefore differ only in how $\mathcal{D}_0$ is generated, and \texttt{LLMiBO} and \texttt{LLINBO} share the same $\mathcal{D}_0$ and differ only in how the LLM is used after initialization.

All other settings are the same as in Appendix~\ref{app:newmodel}: GPT-5 mini is the LLM agent, and we consider Ackley-6$D$, Hartmann-6$D$, and Levy-8$D$ with a budget of $T=20D$ and $50$ independent replications. All \texttt{LLINBO} hyperparameters are the same as in Sec.~\ref{sec:num}.

\begin{figure*}[ht]
\centering
    \begin{subfigure}[t]{0.32\textwidth}
        \includegraphics[width=\linewidth]{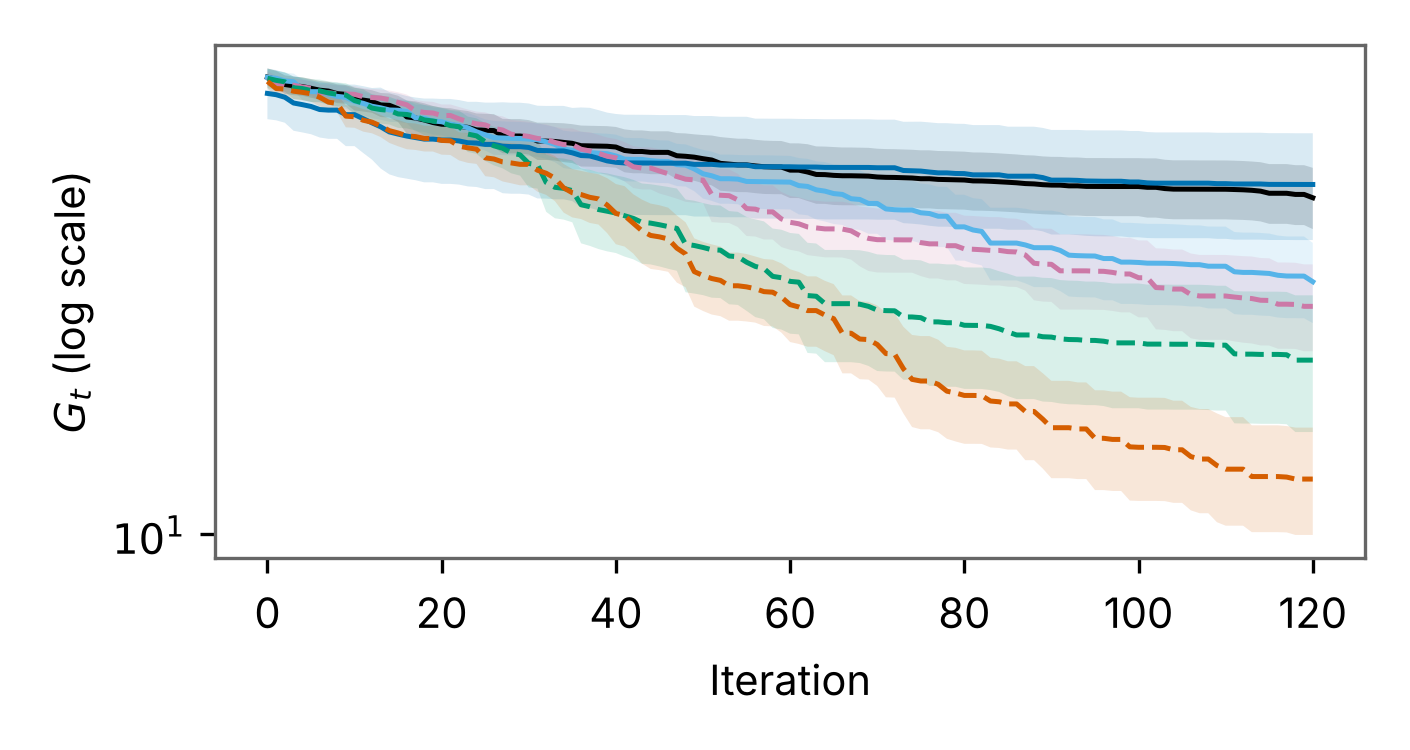}
        \caption{Ackley-6$D$}
    \end{subfigure}
    \hfill
    \begin{subfigure}[t]{0.32\textwidth}
        \includegraphics[width=\linewidth]{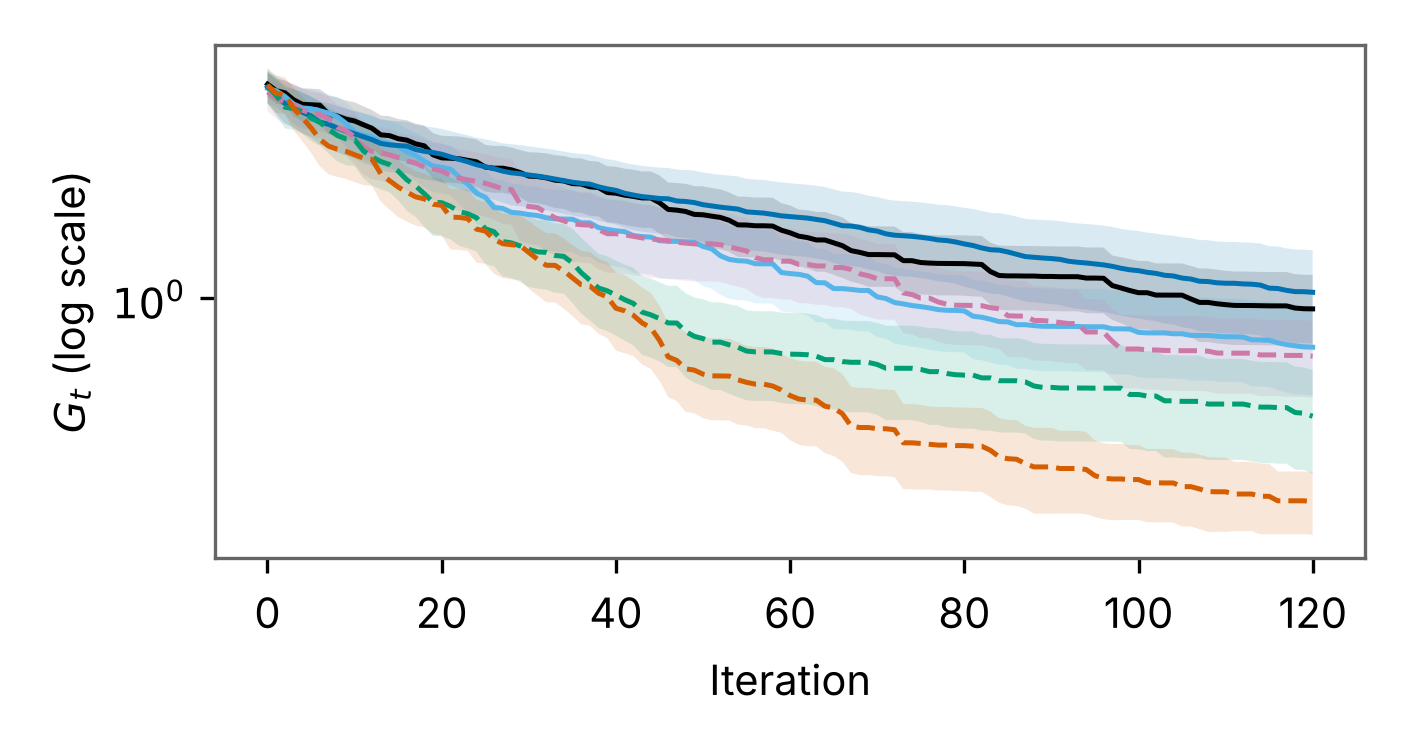}
        \caption{Hartmann-6$D$}
    \end{subfigure}
    \hfill
    \begin{subfigure}[t]{0.32\textwidth}
        \includegraphics[width=\linewidth]{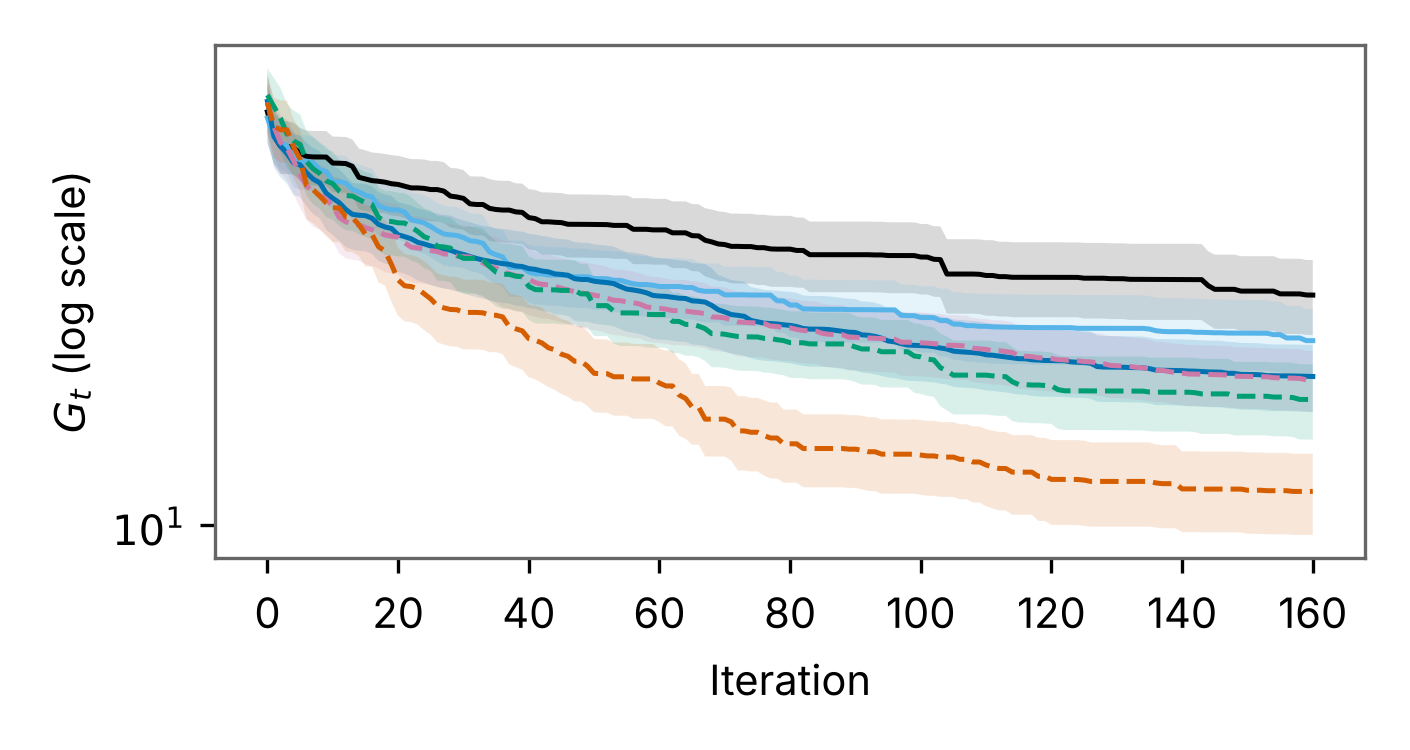}
        \caption{Levy-8$D$}
    \end{subfigure}

\caption{
$G_t$ comparison for BBO with GPT-5 mini, including the LLM-initialized BO baseline. Each line shows the mean regret over 50 independent replications; shaded regions denote 95\% confidence intervals of the mean. The vertical axis is on a log scale.
\textbf{Proposed methods}: \dashedline{darkgreen}{1.5em} \texttt{LLINBO-Transient},
\dashedline{magenta}{1.5em} \texttt{LLINBO-Justify},
\dashedline{brickred}{1.5em} \texttt{LLINBO-Constrained}. \textbf{Baselines}:
\solidline{skyblue}{1.5em} \texttt{LLMiBO},
\solidline{cobaltblue}{1.5em} \texttt{LLAMBO-light},
\solidline{black}{1.5em} BO.
}
\label{fig:llmibo}
\end{figure*}
Figure~\ref{fig:llmibo} shows that \texttt{LLMiBO} improves on BO on all three functions, most clearly on Levy-8$D$, confirming that an LLM-generated $\mathcal{D}_0$ is useful. However, \texttt{LLINBO-Transient} and \texttt{LLINBO-Constrained} achieve lower regret than \texttt{LLMiBO} throughout the optimization on all three functions, and \texttt{LLINBO-Justify} performs at least as well as \texttt{LLMiBO}. Since \texttt{LLMiBO} and \texttt{LLINBO} start from the same $\mathcal{D}_0$, the improvement of \texttt{LLINBO} cannot be explained by warmstarting alone; it comes from continuing to use the LLM's suggestions, validated by the $\mathcal{GP}$, after initialization. \texttt{LLMiBO} is also the cheapest LLM-based method, using about 1K--2K tokens per run since it queries the LLM only once. This highlights the cost-performance trade-off: minimal LLM use is cheap but leaves performance on the table, while \texttt{LLINBO-Transient} achieves much lower regret with only 13K--24K tokens.

\subsection{Comparison with HAIC BO}
\label{app:pibo}
While the scale of external information considered in previous HAIC works is not directly comparable to the setting of either \textit{LLM-based BO} or the proposed methods, in this section, we compare our approach with $\pi$BO \cite{hvarfnerpi} on the Branin-2$D$ and Levy-2$D$ functions by replacing the human effort of specifying $\pi(x)$ (the preference function) with an LLM. Specifically, we provide the problem context and the initial dataset as input to the LLM. For each function (defined on $[0,1]^2$), we then randomly select 100 points $\{z_i\}_{i=1}^{100}$ and query the LLM for the probability of each point being the optimum, denoted $p_i,\ i \in [100]$. To approximate a continuous prior $\pi(x)$, we normalize the probabilities to sum to one and apply Kernel Density Estimation. The hyperparameter $\beta$ is set to $T/100$, following the settings in~\cite{hvarfnerpi}, and all other configurations remain the same as in Sec.~\ref{sec:num}. The acquisition function in $\pi$BO is given by
\[
\alpha_{\pi}(x,F_{t-1}) = \alpha(x,F_{t-1}) \, \pi(x)^{\beta/t},
\]
where $\alpha$ is the acquisition function, which we set to UCB in this experiment. We also evaluate a dynamic variant in which $\pi(x)$ is updated at each iteration by re-querying the LLM with both the problem context and the historical observations, where we call it $\pi$BO-dynamic.

Fig.~\ref{fig:hporesa} presents the regret trajectories for all methods. We observe that our proposed approaches consistently outperform $\pi$BO in both experimental settings. Notably, even though $\pi$BO updates the preference function $\pi(x)$ at every iteration using the LLM, its performance remains unstable and unreliable. We acknowledge that extracting richer information from LLMs—beyond a single design point per iteration—remains an open question and represents an exciting direction for future research.

\begin{figure*}[ht]
    \centering
    \begin{subfigure}[t]{0.48\textwidth}
\includegraphics[width=\linewidth]{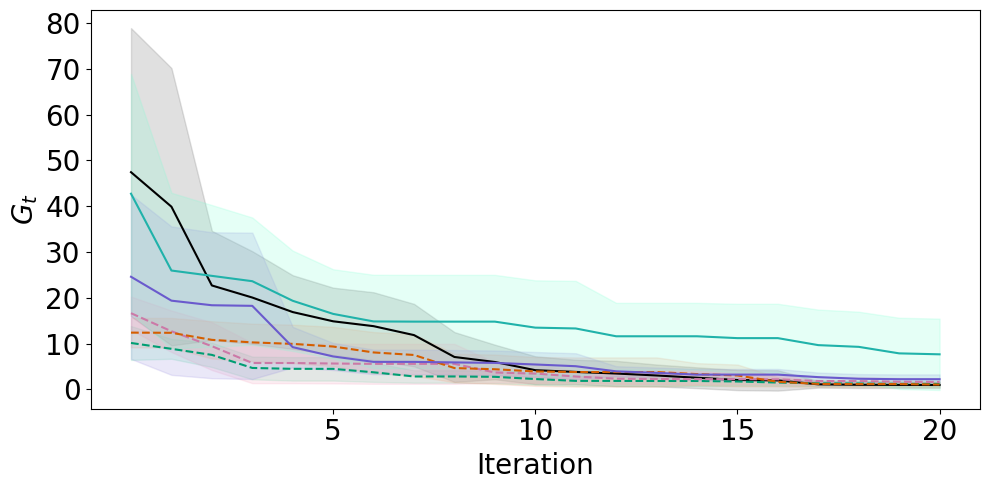}
        \caption{Branin-2$D$}
    \end{subfigure}
    \hfill
    \begin{subfigure}[t]{0.48\textwidth}
\includegraphics[width=\linewidth]{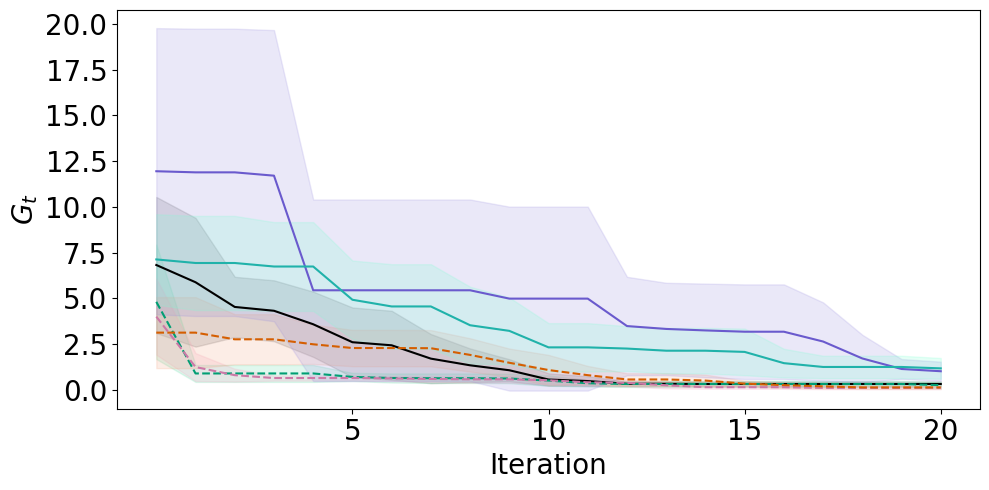}
        \caption{Levy-2$D$}
    \end{subfigure}

\caption{
Regret comparison between proposed methods and $\pi$BO. Each line shows the regret $G_t$, shaded with 95\% confidence intervals.
\textbf{Proposed methods}:\dashedline{darkgreen}{1.5em} \texttt{LLINBO-Transient}, 
\dashedline{magenta}{1.5em} \texttt{LLINBO-Justify}, 
\dashedline{brickred}{1.5em} \texttt{LLINBO-Constrained}. \textbf{Baselines}:
\solidline{piBO}{1.5em} $\pi$BO,   
\solidline{piBOi}{1.5em} $\pi$BO-dynamic,
\solidline{black}{1.5em} BO.
}
\label{fig:hporesa}
\end{figure*}

\subsection{Experiments on High Dimensional Settings}
\label{app:highdim}

We evaluate the proposed methods on two high-dimensional BBO tasks using the Levy-10$D$ and Ackley-15$D$ benchmark functions. We set the evaluation budget to T=100 to balance optimization performance against token cost and computational tractability.
As shown in Fig.~\ref{fig:bbforesh}, the proposed \texttt{LLINBO} variants reduce regret more rapidly and consistently than the baselines, achieving both lower mean regret and narrower confidence intervals. In contrast, \texttt{LLAMBO} exhibits substantially higher and flatter regret curves, indicating slower improvement and greater variability throughout the optimization process. We hypothesize that this degradation is partly caused by the longer and more complex prompts required in higher-dimensional settings, which may make it more difficult for the LLM to consistently identify and focus on promising regions of the search space.

\begin{figure*}[ht]
\centering
    \begin{subfigure}[t]{0.48\textwidth}
        \includegraphics[width=\linewidth]{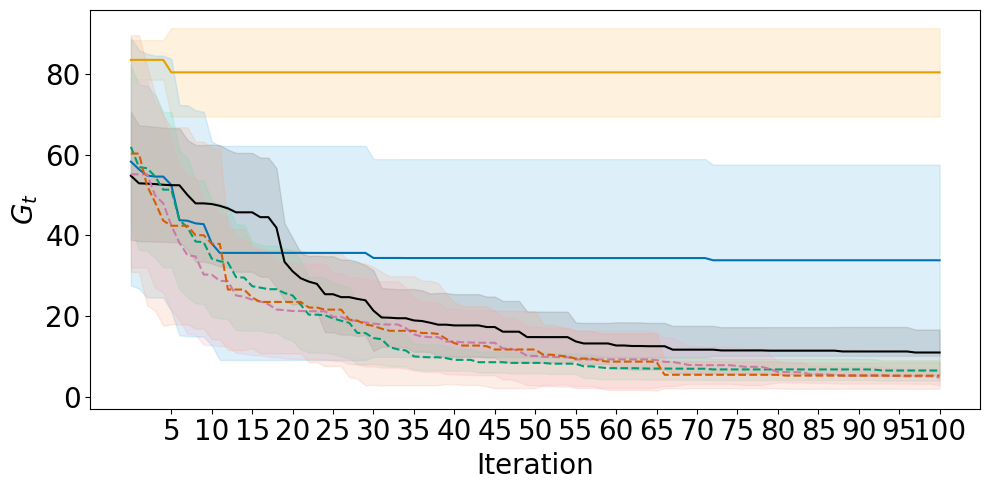}
        \caption{Levy-10$D$}
    \end{subfigure}
    \hfill
    \begin{subfigure}[t]{0.48\textwidth}
        \includegraphics[width=\linewidth]{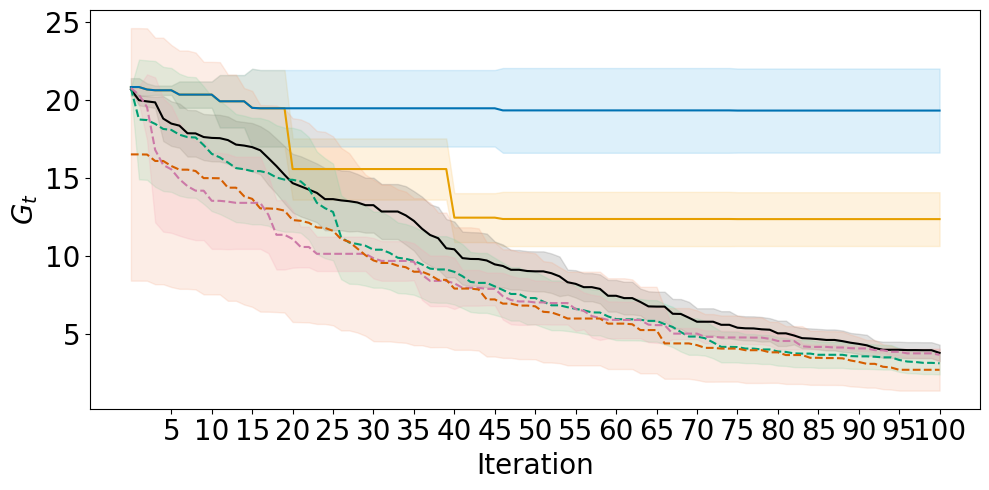}
        \caption{Ackley-15$D$}
    \end{subfigure}

\caption{
$G_t$ comparison for BBO in high dimensional settings. Each line shows the mean regret, shaded with 95\% confidence intervals.
\textbf{Proposed methods}: \dashedline{darkgreen}{1.5em} \texttt{LLINBO-Transient}, 
\dashedline{magenta}{1.5em} \texttt{LLINBO-Justify}, 
\dashedline{brickred}{1.5em} \texttt{LLINBO-Constrained}. \textbf{Baselines}:
\solidline{burntorange}{1.5em} \texttt{LLAMBO}, 
\solidline{cobaltblue}{1.5em} \texttt{LLAMBO-light}, 
\solidline{black}{1.5em} BO.
}
\label{fig:bbforesh}
\end{figure*}

\subsection{Ablation study}
\label{sec:ablation}
We investigate how the amount of contextual information provided to the LLM agent affects the performance of the proposed \texttt{LLINBO} variants. While our primary focus is on integrating LLM agents into the BO framework, we also conduct prompt ablation experiments to evaluate the sensitivity of our methods to varying levels of LLM informativeness.

\begin{table*}[h!]
    \centering
    \begin{tabular}{lccc}
    \hline
    \textbf{Method} & \textbf{Fully informed} & \textbf{Partially informed} & \textbf{Minimally informed} \\
    \hline
    \texttt{LLINBO-Transient} & $0.055 \pm 0.011$ & $0.052 \pm 0.011$ & $0.040 \pm 0.012$ \\
    \texttt{LLINBO-Justify} & $0.065 \pm 0.031$ & $0.059 \pm 0.029$ & $0.062 \pm 0.029$ \\
    \texttt{LLINBO-Constrained} & $0.061 \pm 0.021$ & $0.060 \pm 0.017$ & $0.062 \pm 0.022$ \\
    \hline
    \end{tabular}
    \caption{Performance comparison using $\Delta$ (mean $\pm$ std over 10 runs.) under different levels of LLM contextual information.}
    \label{tab:llm_knowledge}
\end{table*}

To analyze the influence of LLM knowledge, we conducted HPT on Robot with RF-4$D$. We modified the prompt content used in \texttt{LLAMBO-light}, which serves as the LLM agent in our framework, under three settings: (i) \textbf{fully informed}, where the prompt includes historical data, extracted data patterns, and Random Forest model patterns; (ii) \textbf{partially informed}, where it includes historical data and model patterns; and (iii) \textbf{minimally informed}, where it includes only the historical dataset. Performance is summarized using the proportion of improvement, which is defined as $\Delta = (G_0-G_T)/G_0$.

The results are shown in Table~\ref{tab:llm_knowledge}. The performance of \texttt{LLINBO-Transient} is noticeably affected by how informative the LLM inputs are. When the LLM receives limited contextual information, optimization performance declines. However, as the design process gradually shifts from the LLM to the $\mathcal{GP}$ via the diminishing schedule of $p_t$, the final regret becomes similar across settings, underscoring the role of diminishing-$p_t$ in improving robustness. In contrast, \texttt{LLINBO-Justify} and \texttt{LLINBO-Constrained} robustly mitigate unreliable LLM recommendations by validating each proposal using surrogate uncertainty and reverting to classical BO updates when the LLM provides weak guidance.

\section{Numerical Experiments Details}
\label{sec:appb}
We utilize \texttt{GPT-3.5-turbo} as the LLM agent, selected for its demonstrated capability to generate high-quality responses. The temperature parameter is set to its default value of 1.0. Prompt structures for \texttt{LLAMBO} are primarily adapted from the methodology proposed by \cite{liu2024large}. For each task, we define a task-specific system prompt. Specifically, the system prompt for BBO is:
\textit{"You are an AI assistant that helps people find the maximum of a black-box function."}
and for hyperparameter tuning tasks:
\textit{"You are an AI assistant that helps me reduce the mean square error by tuning the hyperparameters in the machine learning model."}
System prompts and the prompts in Tables~\ref{tab:llmbpo}--\ref{tab:llm3dl} are reproduced with minor grammatical edits for readability; the exact strings are in the released code. If an LLM response cannot be parsed as a $D$-dimensional numerical vector, the LLM is queried again, and suggestions outside $\mathcal{X}$ are projected onto $\mathcal{X}$.

We use \texttt{SingleTaskGP} in Python's BoTorch package \cite{balandat2020botorch} as the surrogate model when a statistical model is involved. Namely, its prior mean is set to be constant, where the constant is learned while training, and the kernel function is set to be Mat\'ern-5/2 with automatic relevance determination, with $\text{Gamma}(3,6)$ and $\text{Gamma}(2,0.15)$ priors on the lengthscales and the output scale, respectively; responses are standardized before fitting. For HPT, hyperparameters are rescaled to $[0,1]^D$ before fitting the $\mathcal{GP}$.

\subsection{Experimental Details for BBO}
\label{sec:appb1}

For the BBO task, we employ the following simulation functions: Levy-2$D$, Rastrigin-2$D$, Branin-2$D$, Bukin-2$D$, Hartmann-4$D$, and Ackley-6$D$, as implemented in the Virtual Library of Simulation Experiments \cite{surjanovic2013virtual}. Each function is rescaled to the unit hypercube $[0,1]^D$ , and a negative sign is applied to the response to convert the problem into a maximization task. A summary of these simulation functions is provided below.
\begin{itemize}
    \item  Levy-2$D$
    \[
x_i' = 20x_i - 10, \quad w_i = 1 + \frac{x_i' - 1}{4}, \quad i = 1, 2
\]
\[
f(x) = -\sin^2(\pi w_1) - (w_1 - 1)^2 \left[1 + 10 \sin^2(\pi w_1 + 1)\right] - (w_2 - 1)^2 \left[1 + \sin^2(2\pi w_2)\right],
\]
with $f(x^*)=0$.
\item Rastrigin-2$D$
\[
x'=10.24x-5
\]
\[
f(x) = -20 - \sum_{i=1}^{2} \left[{x'}_i^2 - 10 \cos(2\pi {x'}_i)\right],
\]
with $f(x^*)=0$.
\item Branin-2$D$
\[
x_1' = 15x_1 - 5, \quad x_2' = 15x_2
\]
\[
f(x) = -\left(x_2' - \frac{5.1}{4\pi^2}x_1'^2 + \frac{5}{\pi}x_1' - 6\right)^2 - 10\left(1 - \frac{1}{8\pi}\right)\cos(x_1') - 10,
\]
with $f(x^*)=-0.397887$.
\item Bukin-2$D$
\[
x_1' = 10x_1 - 15, \quad x_2' = 6x_2 - 3
\]
\[
f(x) = -100 \sqrt{\left|x_2' - 0.01 x_1'^2\right|} - 0.01 \left|x_1' + 10\right|,
\]
with $f(x^*)=0$.
\item Hartmann-4$D$
\[
f(x) = \sum_{i=1}^{4} a_i \exp\left(-\sum_{j=1}^{4} A_{ij} (x_j - P_{ij})^2\right)
\]
With constants:
\[
\begin{aligned}
a &= [1.0, 1.2, 3.0, 3.2] \\
A &= \begin{bmatrix}
10 & 3 & 17 & 3.5 \\
0.05 & 10 & 17 & 0.1 \\
3 & 3.5 & 1.7 & 10 \\
17 & 8 & 0.05 & 10
\end{bmatrix} \\
P &= 10^{-4} \times
\begin{bmatrix}
1312 & 1696 & 5569 & 124 \\
2329 & 4135 & 8307 & 3736 \\
2348 & 1451 & 3522 & 2883 \\
4047 & 8828 & 8732 & 5743
\end{bmatrix}
\end{aligned}
\]
with $f(x^*)=3.729841$.
\item Ackley-6$D$
\[
x' = 65.536x - 32.768
\]
\[
f(x) = 20 \exp\left(-0.2 \sqrt{\frac{1}{6} \sum_{i=1}^{6} {x'}_i^2}\right) + \exp\left(\frac{1}{6} \sum_{i=1}^{6} \cos(2\pi {x'}_i)\right) - 20 - e,
\]
with $f(x^*)=0$.
\end{itemize}
The regret $G_t=f(x^*)-\max_{i\le t}y_i$ uses the global maximum $f(x^*)$ stated for each function.

\paragraph{Prompts design for BBO task. }To facilitate effective reasoning by the LLM, each function is accompanied by a \colorbox{lightgreen}{\textbf{Description Card}}, which provides essential contextual information. The \colorbox{lightgreen}{\textbf{Description Card}} includes the following components:

\begin{itemize}
\item \texttt{Function Patterns}: A high-level summary of the function's characteristics, offering partial information to guide the LLM's reasoning. For example:

\textit{"Non-convex and multi-modal. The function exhibits a nearly flat outer region with a prominent central depression, resulting in multiple local optima surrounding a single global optimum. It is highly symmetric and separable, yet optimization remains challenging due to the abundance of local maxima."}

\item \texttt{Dimensionality}: Specifies the number of input dimensions. Given that the input space is normalized to the unit hypercube, this field simply indicates the dimensionality of the design space.
\end{itemize}
The \texttt{Function Patterns} included in each \colorbox{lightgreen}{\textbf{Description Card}} are derived from the benchmark function descriptions provided by \cite{surjanovic2013virtual}, and a summary of these patterns is presented in Table~\ref{tab:llmbbfofp}.

\begin{table}[htbp]
\centering
\renewcommand{\arraystretch}{1.2}
\begin{tabular}{>{\centering\arraybackslash}p{0.16\textwidth} p{0.83\textwidth}}
\toprule
\textbf{Simulation functions} & \colorbox{lightgreen}{\textbf{Description Card}}[\texttt{Function Patterns}]\\
\midrule
Levy-2$D$&highly multimodal but with a unique global maximum.\\
Rastrigin-2$D$& which is highly multimodal, non-convex function with a large number of regularly spaced local minima.\\
Branin-2$D$&smooth, multimodal benchmark with three global maxima\\
Bukin-2$D$&steep, narrow, and highly non-convex landscape with a sharp valley and a unique global maximum\\
Hartmann-4$D$&4-dimensional, non-convex, multi-modal and is composed of weighted, anisotropic Gaussian-like bumps centered at different points, making it highly non-separable and challenging to optimize.\\
Ackley-6$D$&6-dimensional, non-convex, and multi-modal. The function exhibits a nearly flat outer region and a large hole at the center, resulting in many local optima surrounding a single global optimum. It is highly symmetric and separable in nature, but optimization is still challenging due to the numerous local maxima.\\
\bottomrule
\end{tabular}
\caption{Function patterns used in the \colorbox{lightgreen}{\textbf{Description Card}} for each simulation function.}
\label{tab:llmbbfofp}
\end{table}

Next, we introduce \colorbox{lightpurple}{\textbf{Data Card}}, which collects the information of previously observed designs and the responses. For example, at iteration 4, the \colorbox{lightpurple}{\textbf{Data Card}} would be \textit{x: (0.2334, 0.12), f(x): 1.2311; x: (0.1217, 0.433), f(x): 1.091; x: (0.9, 0.5), f(x): 4.502; x: (0.108, 0.203), f(x): 3.22}.  

In the \texttt{LLAMBO} framework, candidate sampling is facilitated by a structured prompt designed to elicit a diverse set of potential query points. This mechanism is illustrated in the Candidate sampling phase of Table~\ref{tab:llmbpo}. At each iteration, we prompt LLM 10 times to generate a total of 10 candidate points. To enhance the diversity of these candidates, we follow the strategy outlined in \cite{liu2024large}, where the content of the \colorbox{lightpurple}{\textbf{Data Card}} is permuted across prompts.

The \texttt{LLAMBO} framework~\cite{liu2024large} introduces a hyperparameter \( \alpha = 0.1 \) to balance exploration and exploitation during the candidate sampling phase. At iteration \( t \), we compute the \colorbox{lightorange}{\textbf{Target Score}} based on the current observed values \( \{ y_i \} \) as follows:
\[
\colorbox{lightorange}{\textbf{Target Score}} =
\begin{cases}
    \min_i y_i + \alpha \cdot (\max_i y_i - \min_i y_i), & \text{for minimization}, \\
    \max_i y_i - \alpha \cdot (\max_i y_i - \min_i y_i), & \text{for maximization}.
\end{cases}
\]
This value serves as a dynamic threshold to guide the LLM in proposing candidates that are both competitive with current best observations and diverse enough to enable exploration.

In the \texttt{LLAMBO} framework, a surrogate prompt is used to estimate the predictive mean and variance at each candidate point generated by the candidate sampling prompt. This process corresponds to the Surrogate modeling phase illustrated in Table~\ref{tab:llmbpo}. To promote variability in the surrogate responses, we similarly permute the \colorbox{lightpurple}{\textbf{Data Card}} across prompts. Finally, an AF is applied to select the next query point. We adopt the Expected Improvement (EI) criterion~\cite{jones1998efficient}, consistent with the acquisition strategy employed in~\cite{liu2024large}.

In contrast, the \texttt{LLAMBO-light} variant bypasses explicit surrogate querying by prompting LLM directly with the problem formulation and historical observations to generate the next evaluation point. This streamlined design process corresponds to the Candidate generation phase shown in Table~\ref{tab:llmbpo}.

\begin{table}[htbp]
\centering
\renewcommand{\arraystretch}{1.2}
\begin{tabular}{>{\centering\arraybackslash}p{0.21\textwidth} p{0.78\textwidth}}
\toprule
\textbf{Phases} & \textbf{Prompts} \\
\midrule
Warmstarting \texttt{LLAMBO LLAMBO-light}& 
You are assisting me with maximizing a black-box function. The function is \colorbox{lightgreen}{\textbf{Description Card}}[\texttt{Function Patterns}]. Suggest \colorbox{lightgreen}{\textbf{Description Card}}[\texttt{Dimensionality}] promising starting points in the range \([0, 1]\textasciicircum{}{\colorbox{lightgreen}{\textbf{Description Card}}[\texttt{Dimensionality}]}\). Return the points strictly in JSON format as a list of \colorbox{lightgreen}{\textbf{Description Card}}[\texttt{Dimensionality}]-dimensional vectors. Do not include any explanations, labels, formatting, or extra text. The response must be strictly valid JSON. \\
Candidate sampling \texttt{LLAMBO}& 
The following are past evaluations of a black-box function. The function is \colorbox{lightgreen}{\textbf{Description Card}}[\texttt{Function Patterns}]. \colorbox{lightpurple}{\textbf{Data Card}} The allowable ranges for x is [0, 1]\textasciicircum{}\colorbox{lightgreen}{\textbf{Description Card}}[\texttt{Dimensionality}]. Recommend a new x that can achieve the function value of \colorbox{lightorange}{\textbf{Target Score}}. Return only a single \colorbox{lightgreen}{\textbf{Description Card}}[\texttt{Dimensionality}]-dimensional numerical vector with the highest possible precision. Do not include any explanations, labels, formatting, or extra text. The response must be strictly valid JSON. \\
Surrogate modeling \texttt{LLAMBO}& 
The following are past evaluations of a black-box function, which is \colorbox{lightgreen}{\textbf{Description Card}}[\texttt{Function Patterns}]. \colorbox{lightpurple}{\textbf{Data Card}} The allowable ranges for x is [0, 1]\textasciicircum{}\colorbox{lightgreen}{\textbf{Description Card}}[\texttt{Dimensionality}]. Predict the function value at x = $x$. Return only a single numerical value. Do not include any explanations, labels, formatting, or extra text. The response must be strictly a valid floating-point number.\\

Candidate generation \texttt{LLAMBO-light}& The following are past evaluations of a black-box function, which is \colorbox{lightgreen}{\textbf{Description Card}}[\texttt{Function Patterns}]. \colorbox{lightpurple}{\textbf{Data Card}} The allowable ranges for x is [0, 1]\textasciicircum{} \colorbox{lightgreen}{\textbf{Description Card}}[\texttt{Dimensionality}]. Based on the past data, recommend the next point to evaluate that balances exploration and exploitation: - Exploration means selecting a point in an unexplored or less-sampled region that is far from the previously evaluated points. - Exploitation means selecting a point close to the previously high-performing evaluations. The goal is to eventually find the global maximum. Return only a single \colorbox{lightgreen}{\textbf{Description Card}}[\texttt{Dimensionality}]-dimensional numerical vector with high precision. The response must be valid JSON with no explanations, labels, or extra formatting. Do not include any explanations, labels, formatting, or extra text. 
\\

\bottomrule
\end{tabular}
\caption{Prompts used across different stages of \texttt{LLAMBO} and \texttt{LLAMBO-light} in the BBO task.}
\label{tab:llmbpo}
\end{table}
\clearpage
\subsection{Experiment Details for HPT}
\label{sec:appb2}
For each simulation function, we draw 1{,}000 inputs by Latin hypercube sampling over its input domain and add Gaussian noise with variance $0.1\cdot\mathrm{Var}(y)$ to the responses. Robot uses 8 inputs (four angles in $[0, 2\pi]$ and four lengths in $[0, 1]$), and Piston uses its 7 standard inputs. The tuning objective for all models is to minimize the MSE. The number of trees in RF and XGBoost is fixed at 500. The search spaces for the hyperparameters are specified as follows.
\subsection*{RF-4$D$}
\begin{itemize}
\item \texttt{max\_depth} (Maximum depth of a tree): $[-1, 50]$ (integer; values $\le 0$, e.g.\ $-1$, indicate no depth limit)
\item \texttt{min\_samples\_split} (Minimum samples to split an internal node): $[2, 20]$ (integer)
\item \texttt{min\_samples\_leaf} (Minimum samples required in a leaf node): $[1, 20]$ (integer)
\item \texttt{max\_features} (Fraction of features to consider for best split): $[0.1, 1.0]$
\end{itemize}

\subsection*{SVR-3$D$}
\begin{itemize}
\item \texttt{C} (Regularization parameter): $C \in [0.01, 1000.0]$
\item \texttt{epsilon} (Epsilon in the $\epsilon$-insensitive loss): $\epsilon \in [0.0001, 1.0]$
\item \texttt{gamma} (Kernel coefficient for RBF kernel): $\gamma \in [0.0001, 1.0]$
\end{itemize}

\subsection*{XGB-4$D$}
\begin{itemize}
\item \texttt{max\_depth} (Maximum depth of a tree): $[1, 10]$ (integer)
\item \texttt{learning\_rate} (Step size shrinkage): $[0.01, 0.3]$
\item \texttt{subsample} (Subsample ratio of the training set): $[0.5, 1.0]$
\item \texttt{colsample\_bytree} (Subsample ratio of columns per tree): $[0.5, 1.0]$
\end{itemize}

\paragraph{Prompts design for hyperparameter tuning task.} The prompt settings for both \texttt{LLAMBO} and \texttt{LLAMBO-light} in the hyperparameter tuning task follow the same configuration as in the BBO task ($\alpha$ and AF), with the exception of the prompt structure. In particular, the hyperparameter tuning prompts also require both the \colorbox{lightgreen}{\textbf{Description Card}} and the \colorbox{lightpurple}{\textbf{Data Card}} to capture the relevant model specifications and historical evaluations.  

Each \colorbox{lightgreen}{\textbf{Description Card}} specifies four key components:

   \begin{itemize}
        \item \texttt{Data Patterns}: Summarize key dataset features that help the LLM understand the task.
        \begin{enumerate}
            \item Piston simulation function: \textit{"The dataset models the cycle time of a piston moving within a cylinder, based on seven physical input variables including mass, surface area, pressure, and temperature."}
            \item Robot simulation function: \textit{"The dataset models the position of a planar robotic arm consisting of four rotating joints and link lengths, computing the Euclidean distance of the arm’s endpoint from the origin."}
        \end{enumerate}

        \item \texttt{Model Patterns}: Describe the predictive model being used and any fixed configurations.

        \item \texttt{Controllable Hyperparameters}: List the tunable hyperparameters along with their types and ranges, and this matches the controllable parameters described previously.
        \item \texttt{Dimensionality}: The dimensions of controllable hyperparameters.
    \end{itemize}
The prompts also state the number of samples $N$ and the number of features $F$ of the dataset.

The \colorbox{lightpurple}{\textbf{Data Card}} for the hyperparameter tuning task may, for instance, take the form:
\textit{(C, gamma): (0.21, 0.12), mean squared error: 0.899; (C, gamma): (0.98, 0.422), mean squared error: 1.00},
where each entry reflects a past evaluation consisting of a specific hyperparameter configuration and its corresponding performance metric (i.e., MSE).

Together with the \colorbox{lightgreen}{\textbf{Description Card}}, which outlines the model and search space, the complete prompt structure used in both \texttt{LLAMBO} and \texttt{LLAMBO-light} is illustrated in Table~\ref{tab:llmhpo} and ~\ref{tab:llmhpol}.

\begin{table}[htbp]
\centering
\renewcommand{\arraystretch}{1.2}
\begin{tabular}{>{\centering\arraybackslash}p{0.21\textwidth} p{0.78\textwidth}}
\toprule
\textbf{Phases} & \textbf{Prompts} \\
\midrule
Warmstarting \texttt{LLAMBO LLAMBO-light}& 
You are assisting with automated hyperparameter tuning using \colorbox{lightgreen}{\textbf{Description Card}}[\texttt{Model Patterns}] for a regression task. \colorbox{lightgreen}{\textbf{Description Card}}[\texttt{Data Patterns}] Model performance is evaluated using mean square error (MSE). The dataset contains $N$ samples and $F$ continuous features. The following hyperparameters are tunable: \colorbox{lightgreen}{\textbf{Description Card}}[\texttt{Controllable Hyperparameters}]. Please suggest \colorbox{lightgreen}{\textbf{Description Card}}[\texttt{Dimensionality}] diverse yet effective configurations to initiate a Bayesian Optimization process for hyperparameter tuning. Format your response strictly as a JSON array of \colorbox{lightgreen}{\textbf{Description Card}}[\texttt{Dimensionality}]-dimensional numerical vectors (lists). Do not include explanations, comments, or any extra text outside the JSON.\\

Candidate sampling \texttt{LLAMBO}& 
The following are examples of the performance of a \colorbox{lightgreen}{\textbf{Description Card}}[\texttt{Model Patterns}] measured in mean square error and the corresponding model hyperparameter configurations. \colorbox{lightpurple}{\textbf{Data Card}} The model is evaluated on a regression task. \colorbox{lightgreen}{\textbf{Description Card}}[\texttt{Data Patterns}] The allowable ranges for the hyperparameters are: \colorbox{lightgreen}{\textbf{Description Card}}[\texttt{Controllable Hyperparameters}]. Recommend a configuration that can achieve the target mean square error of \colorbox{lightorange}{\textbf{Target Score}}, and each dimension must be strictly within the allowable range specified above. Return only a single \colorbox{lightgreen}{\textbf{Description Card}}[\texttt{Dimensionality}]-dimensional numerical vector with the highest possible precision. The response needs to be a list and must be strictly valid JSON. Do not include any explanations, labels, formatting, or extra text like \texttt{json}.\\

Surrogate modeling \texttt{LLAMBO}& The following are examples of the performance of a \colorbox{lightgreen}{\textbf{Description Card}}[\texttt{Model Patterns}] measured in mean square error and the corresponding model hyperparameter configurations. \colorbox{lightpurple}{\textbf{Data Card}} The model is evaluated on a regression task. \colorbox{lightgreen}{\textbf{Description Card}}[\texttt{Data Patterns}] The dataset contains $N$ samples and $F$ features and all of the features are continuous. Predict the mean square error when the model hyperparameter configuration is set to be $x$. Do not include any explanations, labels, formatting, or extra text. The response must be strictly a valid floating-point number.\\

\bottomrule
\end{tabular}
\caption{Prompts used across different stages of \texttt{LLAMBO} in the hyperparameter tuning task.}
\label{tab:llmhpo}
\end{table}

\begin{table}[htbp]
\centering
\renewcommand{\arraystretch}{1.2}
\begin{tabular}{>{\centering\arraybackslash}p{0.21\textwidth} p{0.78\textwidth}}
\toprule
\textbf{Phases} & \textbf{Prompts} \\
\midrule
Candidate generation \texttt{LLAMBO-light}& The following are examples of the performance of a \colorbox{lightgreen}{\textbf{Description Card}}[\texttt{Model Patterns}] measured in mean square error and the corresponding model hyperparameter configurations. \colorbox{lightpurple}{\textbf{Data Card}} The model is evaluated on a regression task. \colorbox{lightgreen}{\textbf{Description Card}}[\texttt{Data Patterns}] The dataset contains $N$ samples and $F$ features and all of the features are continuous. The allowable ranges for the hyperparameters are: \colorbox{lightgreen}{\textbf{Description Card}}[\texttt{Controllable Hyperparameters}]. Your goal is to recommend the next setting to evaluate that balances exploration and exploitation: - Exploration favors regions that are less-sampled or farther from existing evaluations. - Exploitation favors regions near previously low mean square error. To encourage exploration, avoid suggesting values too close to past evaluations. You are on iteration $t$ out of $T$. The ultimate objective is to find the global minimum prediction mean square error. The ideal prediction mean square error is 0. Return only a single \colorbox{lightgreen}{\textbf{Description Card}}[\texttt{Dimensionality}]-dimensional numerical vector with the highest possible precision. Do not include any explanations, labels, formatting, or extra text like \texttt{json}. The response must be strictly valid JSON.
\\

\bottomrule
\end{tabular}
\caption{Prompts used across different stages of \texttt{LLAMBO-light} in the hyperparameter tuning task.}
\label{tab:llmhpol}
\end{table}

\clearpage
\section{3D Printing Details}
\label{sec:appc}
We define the controllable design parameters of the printer via a comprehensive correlation analysis, and the selected variables of interest are summarized below.
\begin{itemize}
    \item Nozzle Temperature: Temperature of the hot-end nozzle in °C.
    \item Z Hop Height: The vertical lift of the nozzle during travel (non-printing) moves.
    \item Coasting Volume: Volume of filament not extruded at the end of a line.
    \item Retraction Distance: Distance (mm) the filament is pulled back before a travel move.
    \item Outer Wall Wipe Distance: Distance (mm) the nozzle continues moving after the outer wall ends.
\end{itemize}

\subsection{Quantifying the Stringing Percentage}
\label{sec:appc1}
An image-based metric is used to quantify the stringing percentage. Printed parts were photographed under consistent lighting conditions against a black background. Each image was converted to grayscale to simplify processing, and a fixed region of interest (ROI) was cropped to capture the space between the two vertical columns (see the left panel of Fig.~\ref{fig:measure}). This region should appear empty when no stringing is present.

To differentiate potential stringing from the background, a pixel intensity threshold was selected through trial-and-error. Pixels with intensity below the threshold were set to black, while those above were set to white (see the right panel of Fig.~\ref{fig:measure}). The stringing percentage was then calculated as the ratio of white pixels to the total number of pixels within the ROI. This approach offers a fast and consistent approximation of stringing severity across multiple prints.

\begin{figure}[!htbp]
\includegraphics[height = 0.35\textwidth]{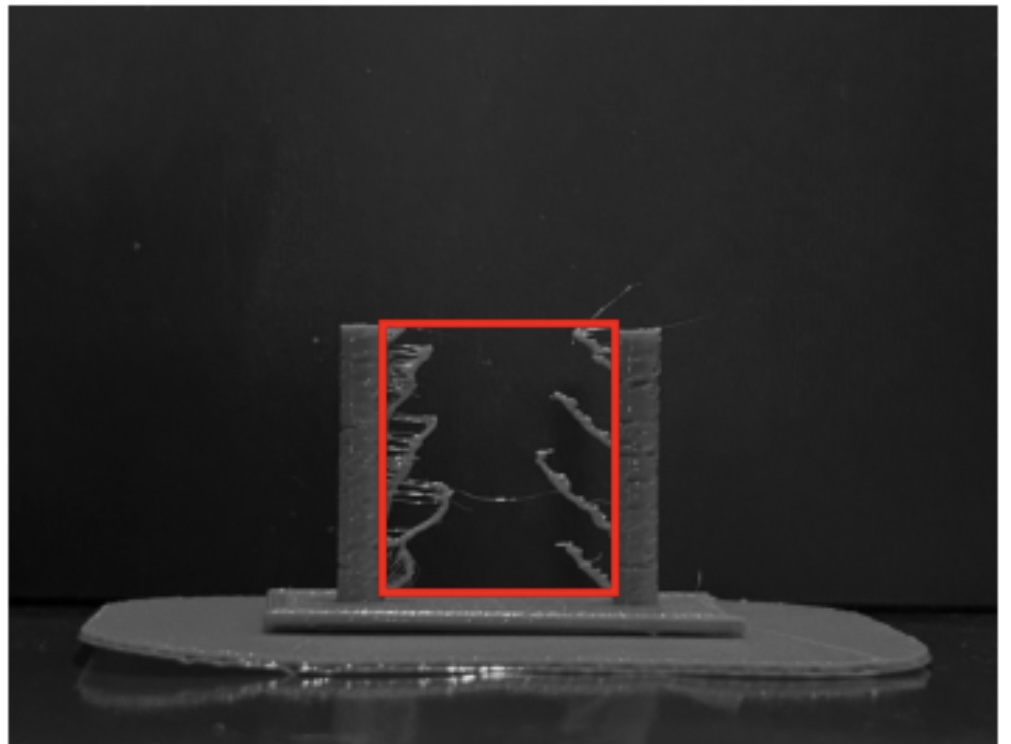}\hfill\includegraphics[height = 0.35\textwidth]{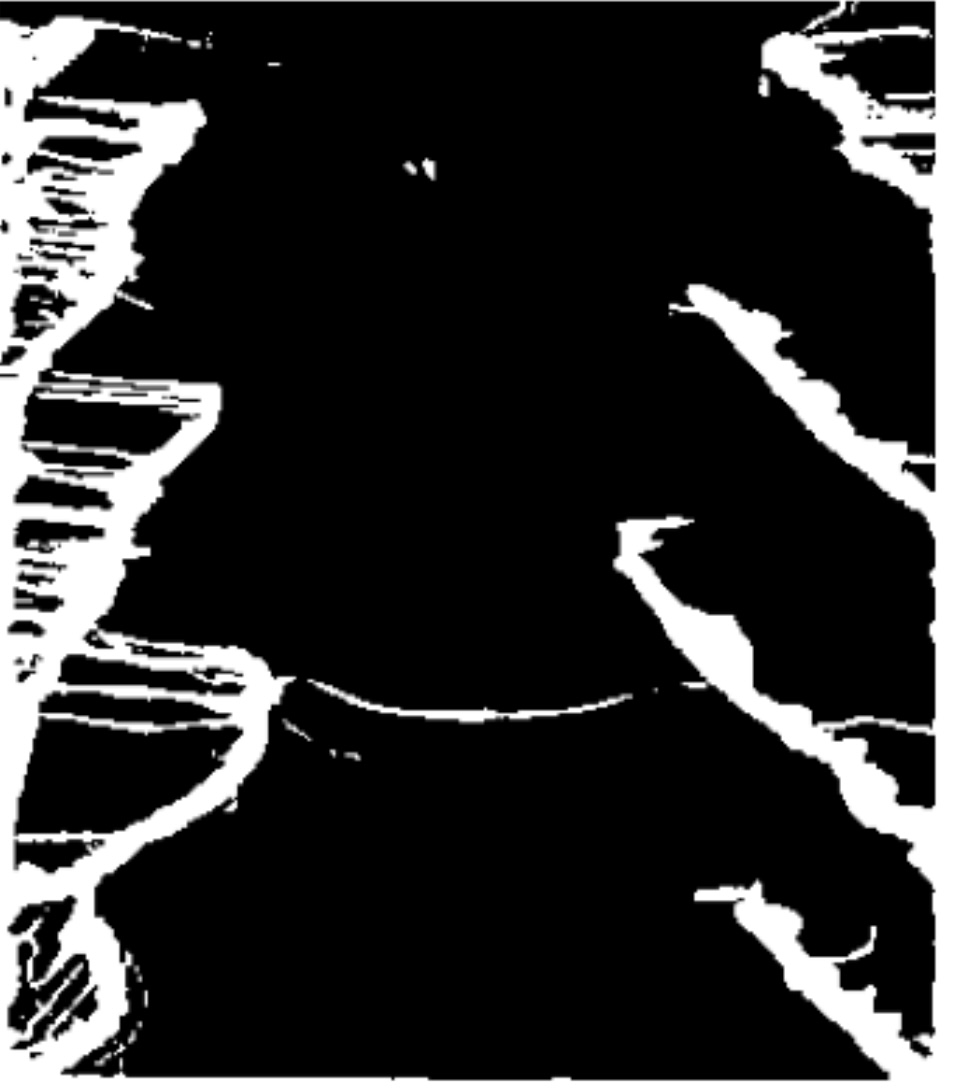}\hfill
\caption{Grayscale image (BO, iteration 2) of the printed part with the region of interest (left panel), and white pixels approximating the stringing amount (15.9\%) over the region of interest (right panel).} 
\label{fig:measure}
\end{figure}

\subsection{Prompts Design}
\label{sec:appc2}
The settings of LLMs are the same as in Appendix~\ref{sec:appb}. The system prompt is \textit{You are an AI assistant that helps me optimize the 3D manufacturing process by controlling parameters}. An example of the \colorbox{lightpurple}{\textbf{Data Card}} is \textit{"(Nozzle Temperature, Z Hop Height, Coasting Volume, Retraction Distance, Outer Wall Wipe Distance): (235, 0.3, 0.06, 4, 0.3), Stringing percentage: 12\%}. We also need a \colorbox{lightblue}{\textbf{Parameter Description Card}} to describe the controllable and fixed variables, which is \\
\textit{You are allowed to adjust only five slicing parameters: \textbf{Nozzle Temperature}: Range 220--260\textdegree{}C (step: 1\textdegree{}C), \textbf{Z Hop Height}: Range 0.1--1.0 mm (step: 0.1 mm), \textbf{Coasting Volume}: 0.02--0.1 mm\textsuperscript{3} (step: 0.01 mm\textsuperscript{3}), \textbf{Retraction Distance}: 1.0--10.0 mm (step: 1 mm), and \textbf{Outer Wall Wipe Distance}: 0.0--1.0 mm (step: 0.1 mm) Slicing settings below are fixed: Retraction Speed = 60 mm/s, Travel Speed = 178 mm/s, Fan Speed = 60\%. Other slicing settings are set to be the software's default values}. In the surrogate modeling and candidate generation prompts, the \colorbox{lightblue}{\textbf{Parameter Description Card}} ends after the fixed settings, i.e., without its last sentence on the software's default values.

The warmstarting prompt (for \texttt{LLAMBO-light} and \texttt{LLAMBO}), candidate sampling prompt (for \texttt{LLAMBO}), surrogate modeling prompt (for \texttt{LLAMBO}), and candidate generation prompt(for \texttt{LLAMBO-light}) are shown in Table~\ref{tab:llm3d} and \ref{tab:llm3dl}. 

\begin{table}[htbp]
\centering
\renewcommand{\arraystretch}{1.2}
\begin{tabular}{>{\centering\arraybackslash}p{0.21\textwidth} p{0.78\textwidth}}
\toprule
\textbf{Phases} & \textbf{Prompts} \\
\midrule
Warmstarting \texttt{LLAMBO LLAMBO-light}& 
You are assisting with process planning for 3D printing a simple part using Overture PETG filament on an Ender 3 Pro in a room-temperature environment (around 22°C). The objective is to reduce stringing as much as possible, using knowledge of PETG printing behavior. \colorbox{lightblue}{\textbf{Parameter Description Card}} After each print, stringing is measured via an image-based algorithm, returning a percentage between 0 and 100\%. You must now propose 2 promising combinations of Nozzle Temperature and Z Hop that are likely to minimize stringing, based on your understanding of PETG behavior (e.g., lower nozzle temps reduce oozing, moderate Z hop avoids drag). Format your answer strictly as a valid JSON list of 5-dimensional vectors. Each vector should be: [Nozzle Temperature (°C), Z Hop Height (mm), Coasting Volume (mm³), Retraction Distance (mm), Outer Wall Wipe Distance (mm)]. Do not include any explanations, labels, formatting, or extra text.\\

Candidate sampling \texttt{LLAMBO}&
The following are past evaluations of the stringing percentage and their corresponding Nozzle Temperature and Z Hop values: \colorbox{lightpurple}{\textbf{Data Card}} \colorbox{lightblue}{\textbf{Parameter Description Card}} Recommend a new [Nozzle Temperature (°C), Z Hop Height (mm), Coasting Volume (mm³), Retraction Distance (mm), Outer Wall Wipe Distance (mm)] that can achieve the stringing percentage of \colorbox{lightorange}{\textbf{Target Score}}. Instructions: Return only one 5D vector: `[Nozzle Temperature (°C), Z Hop Height (mm), Coasting Volume (mm³), Retraction Distance (mm), Outer Wall Wipe Distance (mm)]`. Ensure the values respect the allowed ranges and increments. Respond with strictly valid JSON format. Do not include any explanations, comments, or extra text.\\

Surrogate modeling \texttt{LLAMBO}& 
The following are past evaluations of the stringing percentage and the corresponding Nozzle Temperature and Z Hop. \colorbox{lightpurple}{\textbf{Data Card}} \colorbox{lightblue}{\textbf{Parameter Description Card}} Predict the stringing percentage at [Nozzle Temperature, Z Hop Height, Coasting Volume, Retraction Distance, Outer Wall Wipe Distance] = $x$. The stringing percentage needs to be a single value between 0 and 100. Return only a single numerical value. Do not include any explanations, labels, formatting, percentage symbol, or extra text.\\

\bottomrule
\end{tabular}
\caption{Prompts used across different stages of \texttt{LLAMBO} in the 3D printing experiment.}
\label{tab:llm3d}
\end{table}

\begin{table}[htbp]
\centering
\renewcommand{\arraystretch}{1.2}
\begin{tabular}{>{\centering\arraybackslash}p{0.21\textwidth} p{0.78\textwidth}}
\toprule
\textbf{Phases} & \textbf{Prompts} \\
\midrule

Candidate generation \texttt{LLAMBO-light} & The following are past evaluations of the stringing percentage and their corresponding Nozzle Temperature and Z Hop values: \colorbox{lightpurple}{\textbf{Data Card}} \colorbox{lightblue}{\textbf{Parameter Description Card}} Your goal is to recommend the next setting to evaluate that balances exploration and exploitation: Exploration favors regions that are less-sampled or farther from existing evaluations. Exploitation favors regions near previously low stringing percentages. You are on iteration $t$ out of 10. The ultimate objective is to find the global minimum stringing percentage. The ideal stringing percentage is 0\%. Instructions: Return only one  5-dimensional vector: [Nozzle Temperature (°C), Z Hop Height (mm), Coasting Volume (mm³), Retraction Distance (mm), Outer Wall Wipe Distance (mm)]. Ensure the values respect the allowed ranges and increments. Respond with strictly valid JSON format. Do not include any explanations and comments.
\\

\bottomrule
\end{tabular}
\caption{Prompts used across different stages of \texttt{LLAMBO-light} in the 3D printing experiment.}
\label{tab:llm3dl}
\end{table}

\end{document}